\documentclass[12pt]{article}

\usepackage{hyperref}
\hypersetup{
    unicode = false,
    pdftoolbar = true,
    pdfmenubar = true,
    pdffitwindow = true,
    pdftitle = {metadistance},
    pdfauthor = {Hall\'e-Hannan et al.},
    pdfsubject = {metadistance},
    pdfkeywords = {metadistance},
    pdfnewwindow = true,
    colorlinks = true,
    linkcolor = blue,
    citecolor = blue,
    filecolor = black,
    urlcolor = blue,
    breaklinks = true
}

\usepackage[margin=1in]{geometry}
\usepackage{booktabs}
\usepackage{comment}
\usepackage[english]{babel}
\usepackage[T1]{fontenc}
\usepackage{lmodern}
\usepackage[export]{adjustbox}
\usepackage{booktabs}
\usepackage{multirow}
\usepackage{diagbox}

\usepackage{graphicx}
\usepackage{epstopdf}

\usepackage[toc,page]{appendix}

\usepackage{amsmath}

\usepackage{stackengine}
\usepackage{mathtools}
\usepackage{stackrel}
\usepackage{dsfont}

\usepackage{pifont}

\usepackage{mathrsfs}   % to write cursive letters
\usepackage{changepage}   % to prevent overfull box
\usepackage{enumitem}

\usepackage{amsfonts}
\usepackage{makecell}
\usepackage{amssymb}
\usepackage{lipsum}
\usepackage{multicol}
\usepackage[font=small]{caption}
\usepackage{subcaption}
\usepackage{float}
\usepackage[separate-uncertainty=true]{siunitx}
\usepackage{amsmath}
\usepackage{graphicx}
\usepackage{tikz}

\usepackage{forest}
\usepackage{stanli}
\usetikzlibrary{positioning,shapes,shadows,arrows ,automata,arrows.meta,decorations, decorations.text,calligraphy}
\usetikzlibrary{matrix}
\usepackage{pgfplots}
\pgfplotsset{compat=1.18}

\usepackage[normalem]{ulem} % to crossout text using \sout{}

\def\getangle(#1)(#2)#3{%
  \begingroup%
    \pgftransformreset%
    \pgfmathanglebetweenpoints{\pgfpointanchor{#1}{center}}{\pgfpointanchor{#2}{center}}%
    \expandafter\xdef\csname angle#3\endcsname{\pgfmathresult}%
  \endgroup%
}
\usepackage{xfrac}
\usepackage{adjustbox}
\usepackage{standalone}
\usepackage{changepage}
\usepackage{yhmath}

\usepackage[textwidth=21mm,backgroundcolor=yellow,linecolor=orange,textsize=tiny]{todonotes}

\usepackage{arydshln}

\usepackage{xcolor}
\definecolor{Red}{rgb}{1,0,0}
\definecolor{Green}{rgb}{0,.6,0}
\definecolor{Blue}{rgb}{0,0,1}

\newcommand{\blr}{\color{black}}
\newcommand{\meta}{\mathrm{m}}        % meta
\newcommand{\metadec}{\mathrm{md}}    % meta-decreed  
\newcommand{\neutral}{ \mathrm{neu} }
\newcommand{\decreed}{\mathrm{dec}}

\newcommand{\integer}{\mathrm{int}} 
\newcommand{\continuous}{\mathrm{con}} 
\newcommand{\cat}{\mathrm{cat}}

\newcommand{\dist}{\mathrm{dist}}
\newcommand{\edist}{\overline{\mathrm{dist}}}

\newcommand{\restrictedset}{\overline{\mathcal{X}}_i^r / \phantom{}_{\overline{\parents}_i^r}}

\newcommand{\ancestors}{\mathrm{anc}}
\newcommand{\parents}{\mathrm{par}}

\newcommand{\nomad}{{\sf NOMAD}}
\newcommand{\pytorch}{{\sf PyTorch}}

\usepackage{fancyhdr}
\usepackage{csquotes}
\usepackage{xcolor}
\usepackage{amsthm}
\theoremstyle{plain}
\newtheorem{mydef}{Definition}
\newtheorem{theorem}{Theorem}
\newtheorem{corollary}{Corollary}
\newtheorem{assumption}{Assumption}
\title{{A distance for mixed-variable and hierarchical domains with meta variables}
\thanks{GERAD and Department of Mathematics and Industrial Engineering, Polytechnique  Montr\'eal.}
\thanks{IRIT, UT Capitole, and Universit\'e de Toulouse, France.
}}

\author{ 
\href{https://www.gerad.ca/en/people/edward-halle-hannan}{Edward Hall\'e-Hannan}
\thanks{\href{mailto:edward.halle-hannan@polymtl.ca}{\url{edward.halle-hannan@polymtl.ca}}}
\and
\href{https://www.gerad.ca/Charles.Audet/}{Charles Audet}
\thanks{\url{https://www.gerad.ca/Charles.Audet}}
\and
\href{https://www.gerad.ca/en/people/youssef-diouane}{Youssef Diouane}
\thanks{\url{https://www.polymtl.ca/expertises/diouane-youssef}}
\and
\href{https://www.gerad.ca/Sebastien.Le.Digabel/}{S\'ebastien Le~Digabel}
\thanks{\url{https://www.gerad.ca/Sebastien.Le.Digabel}}
\and
\href{https://scholar.google.com/citations?user=yVohWjcAAAAJ}{Paul Saves}
\thanks{\href{mailto:paul.saves@irit.fr}{\url{paul.saves@irit.fr}}}
}

\begin{document}
%---------------------------------------------------%

\maketitle

\begin{center}
    \text{\Large 
    \textbf{Abstract}}
\end{center}

\begin{adjustwidth}{30pt}{30pt}

% 160 words
Heterogeneous datasets emerge in various machine learning and optimization applications that feature different input sources, types or formats.
Most models or methods do not natively tackle heterogeneity.
Hence, such datasets are often partitioned into smaller and simpler ones, which may limit the generalizability or performance, especially when data is limited.
The first main contribution of this work is a modeling framework that generalizes hierarchical, tree-structured, variable-size or conditional search frameworks.
The framework models mixed-variable and {\blr hierarchical} domains in which variables may be continuous, integer, or categorical, with some identified as meta when they influence the structure of the problem.
The second main contribution is a novel distance that compares any pair of mixed-variable points that do not share the same variables, allowing to use whole heterogeneous datasets that reside in mixed-variable and {\blr hierarchical domains with meta variables}.
{\blr The contributions are illustrated through regression and classification experiments using simple distance-based models applied to datasets of hyperparameters with corresponding performance scores.} \\


\end{adjustwidth}

\noindent
{\small
\textbf{Funding:} This research is funded by a Natural Sciences and Engineering Research Council of Canada (NSERC) PhD Excellence Scholarship (PGS D), a Fonds de Recherche du Qu\'ebec (FRQNT) PhD Excellence Scholarship and an Institut de l'{\'E}nergie Trottier (IET) PhD Excellence Scholarship, 
as well as by the NSERC discovery grants  
    RGPIN-2020-04448 (Audet),
    RGPIN-2024-05093 (Diouane)
    and RGPIN-2018-05286 (Le~Digabel).
The work of Saves is part of the activities of ONERA - ISAE - ENAC joint research group. His research, presented in this paper, has been performed in the framework of the COLOSSUS project (Collaborative System of Systems Exploration of Aviation Products, Services and Business Models) and has received funding from the European Union Horizon Europe program under grant agreement n${^\circ}$ 101097120 and in the MIMICO research project funded by the Agence Nationale de la Recherche (ANR) n$^o$ ANR-24-CE23-0380. 
}

{\blr 

%--------------------------------------------------%
%\section{Introduction: context and motivations} 
\section{Introduction} 
\label{sec:intro}
%--------------------------------------------------%

Machine learning tasks or optimization problems with heterogeneous datasets face inherent challenges that arise from several reasons, such as the generation of data from different sources and the presence of various types of variables.
%and data {\rd with different variables}.
%
Although recent models, such as large language models, can natively process heterogeneous objects~\cite{ReHaHoBu2024}, these models require massive
data and computational resources.
%to perform adequately.
%
For this reason, simpler and less costly models, that do not inherently handle heterogeneity, are often employed in practice.
%
%In heterogeneous datasets, data may not necessarily share the same variables, and these variables may be of various types, \textit{i.e.}, \textit{continuous}, \textit{integer} and \textit{categorical}.
%
Heterogeneous datasets are often partitioned into homogeneous data subsets of similar types that are easier to tackle: this approach is undesirable when data is limited or \textit{expensive-to-generate}~\cite{AlNeTr19, JiLi05, MScGhazaleh}, since the performance of methods and generalizability of models can be negatively impacted.

This work has two main contributions: a comprehensive modeling framework and a novel distance function between any pair of points in a heterogeneous dataset.
These contributions allow to use whole mixed-variable datasets, in which variables vary from one data point to another.
The motivation is to allow affordable models and data-driven optimization methods to handle such data heterogeneity, especially when data is limited or expensive.
The work opens up the possibility for research on enhancing models like $K$-nearest neighbors (KNN) or Gaussian processes, and optimization methods for expensive-to-evaluate functions.%, such as Bayesian optimization.

%This work proposes a generalized framework for applications with heterogeneous dataset.
%This research is motivated by real-life machine learning and optimization applications involving heterogeneous data.
%
%Examples from the literature are listed next.
%
%In~\cite{BaDiMoLeSa2023}, the regression of the performance of a neural network based on its hyperparameters involves a variable dictating the number of hidden layers, which implies that the data has different numbers of variables associated to the units of the architecture.
%
%In~\cite{Abra04, KoAuDe01a}, the optimization of thermal insulation design contains a variable that determines the number of heat intercepts, and each additional heat intercept involves new design variables.
%
%In~\cite{BuCiDeNaLa2021}, an architecture design of an aircraft engine is optimized from a surrogate model constructed from a heterogeneous dataset, in which part of the data includes a fan and another part does not.
%
%This works also covers many more applications in various fields, including software architecture design~\cite{AlBuGrKoMe2013}, difficult operational research problems, \textit{e.g.}, multiple vehicle routing problem~\cite{HoIsKu2015}, and many more. 

%--------------------------------------------------%
\subsection{Scope of the work}
\label{sec:scope}
%-------------------------------------------------

%
% Paragraph on dataset
%In the present work, data heterogeneity is intrinsic to the problems addressed.
%
%More precisely, a dataset is generated from points of a domain $\mathcal{X}$ that has two characteristics that implies heterogeneity:
%
%1) it is mixed-variable, \textit{i.e.}, a point $x \in \mathcal{X}$ is composed of finitely many variables from any type amongst categorical ($\cat$), integer ($\integer$) or continuous ($\continuous$),
%
%and 2) two points $x,y\in \mathcal{X}$ do not necessarily share the same variables and/or are not necessarily subject to same bounds.
%{\blr and 2) it is \textit{hierarchical}, meaning that the inclusion or bounds of some variables are influenced by other variables of higher hierarchy, referred as \textit{meta}.

%
The present work focuses on modeling a heterogeneous domain $\mathcal{X}$ with two fundamental characteristics.
First, $\mathcal{X}$  is \textit{mixed-variable}, meaning that a point is composed of finitely many variables from any type amongst categorical ($\cat$), integer ($\integer$) or continuous ($\continuous$).
Second, $\mathcal{X}$ is \textit{hierarchical}, signifying that variables have different degrees of importance and influence between each other:
this work studies
%machine learning or optimization
domains and datasets
in which two mixed-variable points $x,y\in \mathcal{X}$ do not necessarily share the same variables and/or are not necessarily subject to same bounds.
The following example illustrates a domain defining admissible hyperparameters of a Multi-Layer Perceptron (MLP).
\begin{figure}[htb!]
\centering
    \scalebox{0.85}{\begin{tikzpicture}
\color{black}
     %\centering
    \node [shape=rectangle, align=center](table1) at (0,0) {
        \begin{tabular}{lcc} 
        \toprule
            HP  & Bounds & Type  \\ \midrule
            Learning rate $(r)$ & $]0,1[$ & cont \\ 
            Optimizer $(o)$ & $\{ \texttt{ASGD}, \texttt{ADAM} \}$ & cat  \\
            \bottomrule
        \end{tabular}};

    % ---------Split with optimizers --------------%
    \node [shape=rectangle, align=center, xshift=-5cm, yshift=-3cm, at=(table1)](table2){
        \begin{tabular}{lcc} 
        \toprule
            HP  & Bounds  & Type \\ \midrule
            Update $(\alpha)$  & $]0,1[$ & cont \\
            \#  layers $(l)$  & $\{0,1\}$ &  int \\ 
            \# units  $(u_i)$  & $U_{\texttt{ASGD}}$ & int \\
            \bottomrule
        \end{tabular}};

     \node [shape=rectangle, align=center, xshift=5cm, yshift=-3cm, at=(table1)] (table3)  {
    \begin{tabular}{lcc} 
    \toprule
        HP  & Bounds  & Type \\ \midrule
        Average $(\beta)$   & $]0,1[$ & cont \\
        \#  layers $(l)$    & $\{0,1,2\}$ & int \\ 
         \# units  $(u_i)$  & 
         $U_{\texttt{ADAM}}$ & int \\
        \bottomrule
    \end{tabular}};
   
    \draw[->] (table1)--(table2) node[midway, anchor=west, above, xshift=-2cm, yshift=-0.25cm] {if $o=\texttt{ASGD}$};
    \draw[->] (table1)--(table3) node[midway, anchor=east, above, xshift=2cm, yshift=-0.25cm] {if $o=\texttt{ADAM}$};
\end{tikzpicture}}
%\caption{Roles of variables for the choice of hyperparameters (HPs) in a MLP model.}
\caption{\blr Hyperparameters of a MLP involving a hierarchical and mixed-variable domain.}
\label{fig:intro_ex}
\end{figure}

\noindent 
%In Figure~\ref{fig:intro_ex}, the domain is mixed-variable since all types of variables are present.
%the optimizer $o$ is categorical, the number of hidden layers $l$ is an integer and the learning rate $r$ is continuous.
%
The mixed-variable domain from Figure~\ref{fig:intro_ex} is hierarchical for the following reasons.
First, the optimizer variable $o$ controls the bounds on the number of layers and units, and both optimizers {\tt ASGD} and {\tt ADAM} possess their own hyperparameter,
$\alpha$ and $\beta$, respectively.
%the update $\alpha$ or the average $\beta$. 
%
Second, the number of layers $l$ controls how many units $u_i$ are present in each layer, since $0 < i \leq l$.
%in the architecture.
%
%In this example, there are two levels of hierarchy: the bounds of the number of layers $l$ is subjected to the choice of optimizer $o$, and $l$ itself controls the number of units.
%
%This works treat an arbitrary finite number of hierarchy. 
%

%For clarity, a {\em point} $x \in \mathcal{X}$ is arbitrary, whereas a {\em data point} $x_{(i)} \in \mathcal{X}$ is known and part of a dataset. 
The following terminology is used throughout the document.
A {\em point} $x$ refers to a generic element of the domain $\mathcal{X}$, whereas a {\em data point} $x_{(i)}$, with $i \in \{1,2, \ldots, N \}$, refers to a specific element of $\mathcal{X}$ from a dataset of $N$ points.
For example, $x \in \mathcal{X}$ is a mixed-variable vector containing hyperparameters, and  $x_{(1)}$ is a known point where $r = 0.5, o = \texttt{ADAM}, \beta = 0.6$ and $l = 0$.
The datasets are always related to a domain $\mathcal{X}$.
In supervised learning or data-driven optimization, a target function $f:\mathcal{X} \to \mathbb{R}$ is modeled or optimized with a dataset $\{(x_{(i)}, f(x_{(i)}))\}_{i=1}^{N}$ of $N \in \mathbb{N}$ data pairs.
%where $(x_{(i)}, f(x_{(i)})) \in \mathcal{X} \times \mathbb{R}$ is the $i$-th data couple.
%
In unsupervised learning, the dataset $\{ x_{(i)}\}_{i=1}^{M}$ consists of $M \in \mathbb{N}$ data points. 
%
%The datasets are seen as byproducts of the domains from which they are generated, hence the emphasis is on domains.

%
In this work, a hierarchical domain is modeled with a graph structure that is specific to the domain.
However, the proposed distance, called the \textit{meta distance}, is not defined for graphs.
The following table distinguishes the meta distance from distances defined for graphs.
% ---------- Table 1 --------------- %
 \begin{table}[htb!]
    \centering
    \footnotesize % Slightly smaller text
    \renewcommand{\arraystretch}{1.4} % Increase row height for better spacing
    \begin{tabular}{p{0.45\textwidth} | p{0.45\textwidth}} % Wider columns using text width
        \hline
        \multicolumn{1}{c|}{Meta distance (present work)} & \multicolumn{1}{c}{Graph distances~\cite{BeFaPeSkSu2005, RaLiShAmTa2018, SaFu1983, Zh2020}} \\
        \hline
        \hline
        % ------------------------ %
        Uses a single graph structure to treat a hierarchical domain. & Multiple graphs, each with its own structure and architecture. \\
        \hline
        % ------------------------ %
        Compare points. & Compare graphs with nodes and edges. \\
        \hline
        % ------------------------ %
        Uses standard distances, such as Euclidean or Hamming. & Uses graph operations, such as node or edge insertions, structure comparisons or paths. \\
        \hline
        % ----------------------- %
        Useful for ML or optimization involving mixed-variable points with different variables. & Useful for graph-based ML, such as graph matching, graph similarity or computer vision. \\
        \hline
    \end{tabular}
    \caption{\blr Comparison between the meta distance and graph distances.}
    \label{tab:distance_comparison}
\end{table}
% ---------- Table 1 --------------- %
%
Table~\ref{tab:distance_comparison} outlines the main differences between the meta and graph distances.
% tailored for hierarchical domains
%Graph distances have a different purpose, as established in Table~\ref{tab:distance_comparison}, but they can be useful for tackling hierarchical domains. 
%
The meta distance is specifically tailored for mixed-variable and hierarchical domains. 
Graph distances can also be useful for such domains.
%are not specialized for such domains, but could be useful.
%
For instance, \cite{Zh2020} proposes a graph distance comparing model architectures via embeddings and representations, which are learned from graph neural networks.
%
%Similarly, in~\cite{RaLiShAmTa2018},\ca{mettre [36] dans la table ?} a Bayesian optimization method solves deep learning structure optimization problems using a graph-induced kernel based on the geodesic distance over nodes. 
%
Architectures are represented as nodes, and they are connected by edges if they are similar.

The present paper does not only compare architectures or structures representing hierarchies, but also compares variables of different types taking values or categories, as in Figure~\ref{fig:intro_ex}.
%
%An approach focusing on variables is more suitable and graph distances are considered an alternative approach that is out-of-scope.
%
This study focuses on variables, and graph distances are outside the scope.

Data fusion techniques integrating multiple data subsets into a single heterogeneous dataset~\cite{ChBoNaWhPa2012, ZhXiXiZh21} is related, but also considered outside the scope.
The modeling framework developed in this manuscript implicitly performs data fusion on data subsets that do not share the same variables.
%
%The fusion is done by constructing the \textit{extended domain} that considers all variables across all possible subsets.
%
%The contributions of this work can be applied to heterogeneous datasets resulting from data fusion.

%{\rd Hierarchical problems are typically high-dimensional. 
%
%TODO.}

%
This research is motivated by real-life applications.
%
%Examples from the literature are listed next.
In deep learning, hyperparameter optimization seeks an optimal configuration of a neural network with respect to its hyperparameters~\cite{FeHu19, hypernomad_paper, WuChZhXiDe19, YaSh20}.
This problem is typically mixed-variable and hierarchical. %, as in Figure~\ref{fig:intro_ex}.
%As mentioned previously, these mixed-variable problems includes the meta variable for 
%
%In~\cite{Abra04, KoAuDe01a}, the optimization of thermal insulation design contains a variable that determines the number of heat intercepts, and each additional heat intercept involves new design variables.
In~\cite{LuPi04a}, the optimization of a magnetic resonance device includes a variable that determines the number of magnets, and each additional magnet involves new design variables.
%
% In~\cite{BuCiDeNaLa2021}, an architecture design of an aircraft engine is optimized from a surrogate model constructed from a heterogeneous dataset, in which part of the data includes a fan and another part does not.
%
In~\cite{BuCiDeNaLa2021}, an aircraft engine is designed through a heterogeneous dataset, in which a part of the data includes a fan and another part does not.
More applications from various fields are also covered, including software architecture design~\cite{AlBuGrKoMe2013},
statistical medical research~\cite{HuHuKeTs16}, drug discovery in heterogeneous datasets~\cite{Da09} and, multiple vehicle routing problems~\cite{HoIsKu2015}.

%--------------------------------------------------%
\subsection{Objectives, contributions and organization of the work }
\label{sec:objectives_organization}
%--------------------------------------------------%

The overall objective is to develop an interpretable and constant-time distance function for mixed-variable and hierarchical domains in which two points do not necessarily share the same dimension, bounds or variables.
Similar distances exist in the literature, but have at least one of the following issues: reliance on randomness (non-interpretable), lacking properties of a metric, such as the identity of indiscernibles, or complexity depending on datasets or solving an optimization subproblem.
%complexity of computing $d(x,y)$ that depends on datasets or solving an optimization subproblem.
%

The research gap addressed is the lack of mixed-variable distance functions defined on hierarchical domains without these issues.
The motivation is to use whole datasets with mixed-variable and hierarchical domains, in order to improve generalization of models and to provide data-efficient optimization in the context of poor data availability.
%
%of distance-based ML models, improved performance of data-driven optimization methods and enhance interpretability.
%in context where the amount of data is limited. 

%
To achieve the overall objective, two steps, representing the main contributions, are distinguished.
The first formalizes a modeling framework that thoroughly models a hierarchical domain with a graph structure specific to the problem.
This graph encompasses the information related to the hierarchy of variables and the hierarchical interrelationships between variables. 
%HERE.
%This graph encompasses all hierarchical information hierarchies between variables.
%a role to each variable, in addition to their type.
%
%The role of a variable represents its influence on the structure and dimension of a problem.
%
For example, in Figure~\ref{fig:intro_ex}, the optimizer has the highest hierarchy since it influences bounds and inclusion of variables, but it is not influenced by other variables.
%
%The framework tackles a hierarchical domain with a problem-specific graph structure that encompasses all information regarding the roles of variables. 
%
%{\blr This graph structure is problem-specific and it is shared across all points of a domain.}
%on the problem and it tackles the hierarchal part of a domain}
%In~\cite{G-2022-11}, meta-decreed variables are prohibited.
%
%The present work generalizes the roles of variables.
%
In the literature, variants of hierarchical domains are referred to as tree-structured~\cite{bergstra2011algorithms}, variable-size~\cite{PeBrBaTaGu2021} as well as conditional search space or neural architecture search~\cite{JiXuZh2022}.
The modeling framework developed in this study generalizes and unifies all these variants, as well as the framework in~\cite{HuOs2013}, which introduced the term hierarchical in the context of this work.

The second step constructs a distance function based on the proposed modeling framework.
The distance is defined on the {\em extended domain} $\overline{\mathcal{X}}$, rather than directly on its corresponding domain $\mathcal{X}$.
The extended domain is an extension of the domain that involves all included and excluded variables.
Excluded variables are those that are excluded for a given point $x \in \mathcal{X}$, but present in another point $y \in \mathcal{X}$.
For example, in Figure~\ref{fig:intro_ex}, the number of units in the second layer is excluded, when there is only one layer.
Excluded variables are considered by the distance, since they simultaneously provide valuable information and facilitate the comparison of two points that do not share the same variables.
%excluded variables are particularly useful for comparing heterogeneous data points.
%

The motivation and the objectives are illustrated in Figure~\ref{fig:intro_big_picture}.
The colors represent datasets.
Each group of four green blocks corresponds to a dataset, where each data point has four variables. 
The blue and red blocks correspond to datasets with three and two variables, respectively.
\begin{figure}[htb!]
\begin{subfigure}[t]{0.55\textwidth}
\centering
  \scalebox{1}{\begin{tikzpicture}

 \color{black}
    % Define the colors
    % Yellow is green for simplicity
    %\definecolor{myyellow}{RGB}{255,255,0}
    \definecolor{myyellow}{RGB}{0,170,0}
    \definecolor{mygreen}{RGB}{0,170,0}
    \definecolor{myblue}{RGB}{0,160,255}
    \definecolor{myred}{RGB}{255,0,0}

    %\node[label={[align=center]\small {Hierarchical domain} \\ \small {with} \\ \small heterogeneous data}, ellipse, draw, minimum width=1.75cm, minimum height=2.75cm] (Domain) at (5,0) {};
    \node[label={[align=center]\small {Hierarchical} \\ \small {domain} $\mathcal{X}$}, ellipse, draw, minimum width=1.75cm, minimum height=2.75cm] (Domain) at (5,0) {};
    % Domain with heteregeneous data

    % ADDING SCATTERED BLOCKS FIGURE
    \begin{scope}[shift={(4.45,-0.35)}, scale=0.275] 
        % Yellow row (rotated)
        \begin{scope}[rotate=10, shift={(0.25,2.9)}]
            \fill[myyellow] (0,0) rectangle (4,1);
            \draw[thick] (0,0) rectangle (4,1);
            \foreach \x in {1,2,3} \draw[thick] (\x,0) -- (\x,1);
        \end{scope}

        % Green row (rotated)
        \begin{scope}[rotate=8, shift={(1,1.6)}]
            \fill[mygreen] (0,0) rectangle (4,1);
            \draw[thick] (0,0) rectangle (4,1);
            \foreach \x in {1,2,3} \draw[thick] (\x,0) -- (\x,1);
        \end{scope}

        % Blue row (not rotated)
        \begin{scope}[shift={(0,-0.2)}]
            \fill[myblue] (0,0) rectangle (3,1);
            \draw[thick] (0,0) rectangle (3,1);
            \foreach \x in {1,2} \draw[thick] (\x,0) -- (\x,1);
        \end{scope}

        % Red row (rotated)
        \begin{scope}[rotate=-10, shift={(2,-1.5)}]
            \fill[myred] (0,0) rectangle (2,1);
            \draw[thick] (0,0) rectangle (2,1);
            \draw[thick] (1,0) -- (1,1);
        \end{scope}
    \end{scope}
    
   %------------------- RIGHT-SIDE ----------------- %

   %\node[label={[align=center]\small {Partitioned domain} \\ \small {with} \\ \small homogeneous data}, ellipse,  minimum width=1.75cm, minimum height=2.75cm] (Extended) at (9,0) {};
    %
    \node[label={[align=center]\small {Partitioned} \\ \small {domain}}, ellipse,  minimum width=1.75cm, minimum height=2.75cm] (Extended) at (9,0.1) {};

    \begin{scope}[shift={(-1, 0)}]

        % Extended domain ellipse
        %\node[label={$\left(\overline{\mathcal{X}}, \overline{\text{dist}} \right)$}, ellipse, draw, minimum width=1.75cm, minimum height=2.75cm] (Extended) at (10,0) {};
        %\node[label={Extended domain}, ellipse, draw, minimum width=1.75cm, minimum height=2.75cm] (Extended) at (10,0) {};

        % Scattered data
        \begin{scope}[shift={(9.475, -0.55+0.2)}, scale=0.275]
            % Yellow block (on top, bonded with green)
            \fill[myyellow] (0,5) rectangle (4,6);
            \draw[thick] (0,5) rectangle (4,6);
            \foreach \x in {1,2,3} \draw[thick] (\x,5) -- (\x,6);
    
            % Green block (directly below yellow)
            \fill[mygreen] (0,4) rectangle (4,5);
            \draw[thick] (0,4) rectangle (4,5);
            \foreach \x in {1,2,3} \draw[thick] (\x,4) -- (\x,5);
    
            % Blue block (spaced down, 3 squares)
            \begin{scope}[shift={(0.5,0)}]
                \fill[myblue] (0,1) rectangle (3,2);
                \draw[thick] (0,1) rectangle (3,2);
                \foreach \x in {1,2} \draw[thick] (\x,1) -- (\x,2);
            \end{scope}
            
            % Added spacing between blue and red
            \begin{scope}[shift={(1,-2)}]
                % Red block (bottom, 2 squares)
                \fill[myred] (0,0) rectangle (2,1);
                \draw[thick] (0,0) rectangle (2,1);
                \draw[thick] (1,0) -- (1,1);
            \end{scope}

        \end{scope}
          
    \end{scope}

    % --------------------------------- %
    % Arrows
    \begin{scope}[shift={(4.2, -0.25)}, scale=0.45]
        % Arrows Connecting Left to Right (Aligned as Requested)
         \draw [thick, -{Latex[length=3mm]}] (4.5, 0.75) -- (8, 2.5);  % top 
        \draw [thick, -{Latex[length=3mm]}] (4.5, 0.75) -- (8, 0.75);  % Middle horizontal (straight)
         \draw [thick, -{Latex[length=3mm]}] (4.5, 0.75) -- (8, -1.5);  % bottom
    \end{scope}
    % --------------------------------- %

    % Ellipse around the green blocks (top)
    \begin{scope}[shift={(0.5, 0)}]
        \draw (8.5, 0.825+0.2) ellipse [x radius=0.9cm, y radius=0.45cm];
    
        \draw (8.5, -0.15+0.2) ellipse [x radius=0.7cm, y radius=0.25cm];

        \draw (8.5, -0.965+0.2) ellipse [x radius=0.5cm, y radius=0.25cm];
    \end{scope}
 
    % NO symbol
    %\begin{scope}[shift={(8.2, 0.85)}, scale=1.25] 
    %    \draw[red, thick] (0,0) circle (1); % Red circle
    %    \draw[red, thick] (-0.7,-0.7) -- (0.7,0.7); %       Diagonal line
    %\end{scope}

\end{tikzpicture}}
  \subcaption{Motivation: avoid partitioning a hierarchical \\domain and splitting heteregeneous data.}
  \label{subfig:intro_big_picture_motivation}
\end{subfigure}
\hspace{0.25cm}
\begin{subfigure}[t]{0.4\textwidth}
\centering
  \scalebox{1}{\begin{tikzpicture}
 \color{black}
    % Yellow is green for simplicity
    %\definecolor{myyellow}{RGB}{255,255,0}
    \definecolor{myyellow}{RGB}{0,170,0}
    \definecolor{mygreen}{RGB}{0,170,0}
    \definecolor{myblue}{RGB}{0,160,255}
    \definecolor{myred}{RGB}{255,0,0}
    %\definecolor{mygray}{RGB}{160,160,160}
    \definecolor{mygray}{RGB}{255,255,255}
    \definecolor{mywhite}{RGB}{255,255,255}

    % Domain ellipse
    %\node[label={$(\mathcal{X}, \operatorname{dist})$}, ellipse, draw, minimum width=1.75cm, minimum height=2.75cm] (Domain) at (5,0) {};
    %\node[label={[align=center]\small {Hierarchical domain} \\ \small {with} \\ \small heterogeneous data}, ellipse, draw, minimum width=1.75cm, minimum height=2.75cm] (Domain) at (5,0) {};

    \node[label={[align=center]\small {Hierarchical} \\ \small {domain} $\mathcal{X}$}, ellipse, draw, minimum width=1.75cm, minimum height=2.75cm] (Domain) at (5,0) {};

    % ADDING SCATTERED BLOCKS FIGURE
    \begin{scope}[shift={(4.45,-0.35)}, scale=0.275] 
        % Yellow row (rotated)
        \begin{scope}[rotate=10, shift={(0.25,2.9)}]
            \fill[myyellow] (0,0) rectangle (4,1);
            \draw[thick] (0,0) rectangle (4,1);
            \foreach \x in {1,2,3} \draw[thick] (\x,0) -- (\x,1);
        \end{scope}

        % Green row (rotated)
        \begin{scope}[rotate=8, shift={(1,1.6)}]
            \fill[mygreen] (0,0) rectangle (4,1);
            \draw[thick] (0,0) rectangle (4,1);
            \foreach \x in {1,2,3} \draw[thick] (\x,0) -- (\x,1);
        \end{scope}

        % Blue row (not rotated)
        \begin{scope}[shift={(0,-0.2)}]
            \fill[myblue] (0,0) rectangle (3,1);
            \draw[thick] (0,0) rectangle (3,1);
            \foreach \x in {1,2} \draw[thick] (\x,0) -- (\x,1);
        \end{scope}

        % Red row (rotated)
        \begin{scope}[rotate=-10, shift={(2,-1.5)}]
            \fill[myred] (0,0) rectangle (2,1);
            \draw[thick] (0,0) rectangle (2,1);
            \draw[thick] (1,0) -- (1,1);
        \end{scope}
    \end{scope}

    % ---------------- RIGHT SIDE -------------------- %

    \begin{scope}[shift={(-1, 0)}]

        % Extended domain ellipse
        %\node[label={$\left(\overline{\mathcal{X}}, \overline{\text{dist}} \right)$}, ellipse, draw, minimum width=1.75cm, minimum height=2.75cm] (Extended) at (10,0) {};
        %\node[label={[align=center]\small {Extended domain} \\ \small {with} \\ \small aggregated data}, ellipse, draw, minimum width=1.75cm, minimum height=2.75cm] (Extended) at (10,0) {};

        \node[label={[align=center]\small {Extended} \\ \small {domain} $\overline{\mathcal{X}}$}, ellipse, draw, minimum width=1.75cm, minimum height=2.75cm] (Extended) at (10,0) {};

        % Grid inserted where the image was
        \begin{scope}[shift={(9.475, -0.55)}, scale=0.275]
            % Fill rectangles
            \fill[myyellow] (0,3) rectangle (4,4);
            \fill[mygreen] (0,2) rectangle (4,3);
            \fill[myblue] (0,1) rectangle (3,2);
            \fill[mywhite] (3,1) rectangle (4,2);
            \fill[myred] (0,0) rectangle (2,1);
            \fill[mygray] (2,0) rectangle (4,1);
            %\fill[myred] (3,0) rectangle (4,1);
    
            % Outline the entire rectangle
            \draw[thick] (0,0) rectangle (4,4);
            
            % Draw horizontal lines
            \draw[thick] (0,3) -- (4,3);
            \draw[thick] (0,2) -- (4,2);
            \draw[thick] (0,1) -- (4,1);
    
            % Draw equidistant vertical lines
            \draw[thick] (1,0) -- (1,4);
            \draw[thick] (2,0) -- (2,4);
            \draw[thick] (3,0) -- (3,4);
        \end{scope}
          
    \end{scope}

    % ------------------------ %
    % Arrows connecting points
    %\draw [-{Latex[length=3mm]}] (Point.north east) to [out=65, in=120] (Extended_pt.north west);
    %
    %\begin{scope}[yscale=-1, xscale=-1]
    %    \draw [{Latex[length=3mm]}-] (Point.north east) to [out=130, in=50] (Extended_pt.north west);
    %\end{scope}

    %\draw [{Latex[length=3mm]}-{Latex[length=3mm]}]  (5.5, 0.15) to [out=0, in=180] node[midway, above] {$T_G$} (8, 0.15);

    \draw [{Latex[length=3mm]}-{Latex[length=3mm]}] 
    (5.9, 0.155) to [out=0, in=180] 
    node[midway, above] {\small Graph} % Replace \(\star\) with the desired symbol
    (7.6+0.5, 0.15);

    % ---------------------- %

\end{tikzpicture}}
  \subcaption{Objectives: define a distance on the extended domain, in which data is aggregated.}
  \label{subfig:intro_big_picture_approach}
\end{subfigure}
\caption{\blr Global overview of the work.}
\label{fig:intro_big_picture}
\end{figure}

\noindent Figure~\ref{subfig:intro_big_picture_motivation} illustrates a commonly used approach in which the heteregenous dataset is split into simpler homogeneous datasets, each containing data points with the same variables.
Figure~\ref{subfig:intro_big_picture_approach} schematizes the approach employed in this work.
The extended domain $\overline{\mathcal{X}}$ aggregates datasets, even if they have different variables. 
The relations between the variables and the meta distance are established in the extended domain.
%
%A distance is defined on the extended domain.
%
Finally, a graph structure establishes a direct correspondence between the two domains ${\mathcal{X}}$ and $\overline{{\mathcal{X}}}$.
The rest of the document is organized as follows. 
%
%First, the remainder of this section discusses related work in Section~\ref{sec:intro_literature_review}.
%
A literature review is detailed in Section~\ref{sec:literature_review}.
%Then, the MLP example is further developed and detailed in Section~\ref{sec:working_example}.
%
Next, the extended point $\overline{x}$ and the extended domain $\overline{{\mathcal{X}}}$ are thoroughly
%defined via the roles of variables and graph theory 
defined in Section~\ref{sec:notation}.
Afterwards, in Section~\ref{sec:distance}, the meta distance is defined on the extended domain $\overline{{\mathcal{X}}}$, which induces a distance on the domain $\mathcal{X}$.
%as shown in Figure~\ref{subfig:intro_big_picture_approach}.
%
Finally, computational experiments are carried out in Section~\ref{sec:numerical_exp}.
%to benchmark the performance of the novel distance on classification and regression problems.
%two approaches on simple regression models.
%
%The first approach separates the regression problem into subproblems, each with homogeneous data, and the second one is based on the induced distance on the entire heterogeneous dataset.  

}
% end blr section intro

{\blr 
%--------------------------------------------------%
\section{Literature review}
\label{sec:literature_review}
%--------------------------------------------------%

This section presents the relevant literature in three parts.
Section~\ref{sec:background} details background work that this study relies on. 
Afterward, Section~\ref{sec:related} discusses related work within the scope of this work.  
Finally, background and related work are compared with this study in Section~\ref{sec:literature_comparison}.

%--------------------------------------------------%
\subsection{Background work}
\label{sec:background}
%--------------------------------------------------%

%Paragraph on variables
In this work, the particularity that two points in a hierarchical domain $\mathcal{X}$ do not share the same variables, dimension or bounds is a consequence of the so-called {\em meta} variables introduced in the previous work~\cite{G-2022-11}.
These meta variables determine if other variable(s), called {\em decreed}, are excluded or included in a point of the domain $\mathcal{X}$, and control their bounds.
In addition to their type, each variable is assigned a role, such as meta or decreed, that reflects either how it influences the dimension, structure or bounds of the domain $\mathcal{X}$, or how it is subject to the influence of other variables. 
% Neutral variables
Variables that neither influence nor are influenced by other variables are assigned the {\em neutral} role, 
and they are always included in a point. 
%
%The notion of roles of variables is taken from previous framework~\cite{G-2022-11} with a single layer of hierarchy.
%which develops a modeling framework for blackbox optimization on mixed-variable domains with only a single level of hierarchy.
% Meta-decreed develop in this work
%
The framework in~\cite{G-2022-11}, tackling only a single level of hierarchy, is generalized in the present paper.
Notably, roles are generalized and a graph structure is introduced to tackle an arbitrary (finite) level of hierarchy.

An important reference for this work is the technical report~\cite{HuOs2013}, which proposes a mixed-variable kernel function for hierarchical spaces, each paired with a directed acyclic graph, where the nodes are the variables.
Variables with child nodes are required to be categorical.  
The kernel is constructed from one-dimensional kernels for which pseudodistances take into account whether the variables are included or excluded.
The inclusion of a variable is determined by a Kronecker delta function that takes the values of its ancestors as arguments.
%

%--------------------------------------------------%
\subsection{Related work}
\label{sec:related}
%--------------------------------------------------%

% Classification and regression for mixed-variable
Most of the literature on distance or similarity measures for heterogeneous datasets treats the simpler case where heterogeneity originates solely from the variety of variable types.
In classification, variants of the $K$-nearest neighbors, based on distances (or similarity measures) built with combinations of continuous, integer or categorical distances, are commonly studied~\cite{AhDe07, AlNeTr19_2, PeCaRe10}.
Decision trees or random forests are also used for classification~\cite{SoLu15}, and even regression~\cite{KiHo17}.
% Regression
In regression, many kernel functions (similarity measures) have been recently developed for constructing Gaussian processes (GPs)~\cite{RaWi06} over heterogeneous datasets with mixed-variables.
%
%In regression, GPs are characterize by a kernel that computes correlations between inputs. 
%
Kernel methods are well adapted for mixed-variable problems, since mixed kernels can be directly constructed with products or additions of well-documented continuous~\cite{RaWi06}, integer~\cite{GaHe2020} or categorical kernels~\cite{PeBrBaTaGu2019, QiWuJe2008, BaDiMoLeSa2023, ZhTaChAp2020}.

In~\cite{GaDuKaSe2018}, a novel similarity measure, called the Earth mover's intersection (EMI), compares sets of different sizes with an Earth mover's distance (EMD).
EMD measures the minimum cost to transform one distribution into another by solving an optimal transport problem. 
Such distances are also known as \textit{Wasserstein} distances~\cite{SoLePeZaKe2023}.

In~\cite{PeBrBaTaGu2021}, GPs are constructed on said variable-design spaces, which contain dimensional variables~\cite{LuPiSc05a} that are essentially discrete meta variables controlling the inclusion or exclusion of other variables.
The variable-wise decomposition kernel handles the hierarchical domain by comparing only the shared variables between two points based on their respective dimensional variables.  

% Graph distance is now out-of-scope: I will use a shorten version of this passage in the scope, as an example of graph distance.
%{\rd
%In~\cite{RaLiShAmTa2018}, a Bayesian optimization method is used to solve deep learning structure optimization problems via graph-induced kernel functions that compute similarities between possible network architectures.
%
%The search space is modeled as a graph topology where the nodes are the architectures and the edges connect similar architectures.
%
%The kernel is based on the geodesic distance defined over nodes. 
%
%}

In~\cite{BeSeRu10}, feature models manage and capture heterogeneity across data points through tree-structured models consisting of features nodes and relationships arcs, that represent parent-child dependencies or integrity constraints~\cite{AsGrMoGa16}.
Recent advances in features models addressed complex dependencies and constraints in large-scale heterogeneous datasets~\cite{Ba05} with semantic logic.

% Transition?
%Based on the literature works aforementioned, this paper propose a novel distance function tailored for heterogeneous datasets with meta variables to enhance the analysis and understanding of complex data landscapes described in the following sections.

%--------------------------------------------------%
\subsection{Comparison of background and related work}
\label{sec:literature_comparison}
%--------------------------------------------------%

The following table compares approaches and frameworks, from background and related work, with similar objectives to the current work.
The first column indicates whether the approach or framework tackles mixed-variable domains.
The second one shows if variables that are not shared between points are considered in distance computations.
The third column verifies if the distance is deterministic, \textit{i.e.}, interpretable.
The last two columns show how the number of parameters and the cost of the distance scales with the number of variables $n$ or the number of data points $N$. 
The other columns are self-explanatory. 
\begin{table}[htb!]
    \centering
    \footnotesize
    \renewcommand{\arraystretch}{1.2} % Slightly reduce row height
    \setlength{\tabcolsep}{2pt} % Colum spacing
    \begin{tabular}{l c c c c c c c}
        \toprule
        \multirow{2}{*}{\makecell{Approach/ \\ framework }} 
        & \multirow{2}{*}{Mixed-var.} 
        & \multirow{2}{*}{\makecell{Comparisons \\ different var.}} 
        & \multirow{2}{*}{\makecell{Deterministic/ \\ interpretable}} 
        & \multirow{2}{*}{\makecell{Layers of \\ hierarchy}} 
        & \multirow{2}{*}{\makecell{Type of \\ proximity}}
        & \multirow{2}{*}{\makecell{\# of \\ param.}} 
        & \multirow{2}{*}{\makecell{Cost of \\ $d(x,y)$}} \\
        \\[-0.1em] \midrule 
        % ---------- Methods ----------- %
        Subproblems & \checkmark & \ding{53}  & \checkmark  & 0 & Distance & 0 & $\mathcal{O}(1)$ \\
        % --------------------- %
        Meta~\cite{G-2022-11} & \checkmark  & \ding{53}   &  \ding{53}   & 1 & None & \ding{53}   & \ding{53}   \\
        % --------------------- %
        %MMD~\cite{GrBoRaScSm2012} & \checkmark  & ?  & \checkmark  & $?$ & Kernel & $\mathcal{O}(n)$ & $\mathcal{O}(N^2)$ \\
        % --------------------- %
        Variable-size~\cite{PeBrBaTaGu2021} & \checkmark & \ding{53}    & \checkmark  & 1 & Kernel & $\mathcal{O}(n)$ & $\mathcal{O}(1)$ \\
        % --------------------- %
        Hierarchical~\cite{HuOs2013} & $\sim$  & \checkmark & $\sim$  & $m$ & Kernel & $\mathcal{O}(n)$ & $\mathcal{O}(1)$ \\
        % --------------------- %
        EMI~\cite{GaDuKaSe2018} & \checkmark  & \checkmark  & \checkmark  & $m$ & Distance & $\mathcal{O}(N^2)$ & $\mathcal{O}(N^3 \log N)$ \\
        % --------------------- %
        \hdashline
        Current work & \checkmark & \checkmark & \checkmark & $m$ & Distance & $\mathcal{O}(n)$ & $\mathcal{O}(1)$ \\
        \bottomrule
    \end{tabular}
    \caption{\blr Comparison of frameworks based on key features, where \checkmark signifies true, $\sim$ means partially true, \ding{53} represents false, $n$ is the number of variables and $N$ is the number of points in a dataset.
    }
    \label{tab:comparison_methods}
\end{table}

% Naive
In the table, ``Subproblems'' refers to the approach that simply partitions hierarchical domains, as in Figure~\ref{subfig:intro_big_picture_motivation}.
This approach addresses mixed-variable domains, but does not handle hierarchical domains.
%since its partitioning eliminates all hierarchical interactions.
%

% MMD
%MMD stands for \textit{Maximum mean discrepancy} distance that computes average distances weighted by kernels~\cite{GrBoRaScSm2012}.
%
%This approach is general and depends on the kernel employed, explaining the ? in some columns.
%
%The main drawback is its quadratic complexity.
%
%This can be problematic in applications where the distance must be computed several times.

% Previous
The meta framework from previous work~\cite{G-2022-11} does not provides a distance and only tackles a single layer of hierarchy.
This framework is proposed specifically for unifying specific optimization approaches in context where derivatives are not available.
%for problems with meta variables.
%
%It is included in Table~\ref{tab:comparison_methods} to highlight the advancements of the current work.

% Mixed-variable
The variable-size framework proposes a mixed-variable kernel that considers only shared variables between points and supports only a single layer of hierarchy.
%
%The number of hyperparameters of the kernel scales with the number of variables $n$.
%

% Earth's mover
The EMI is defined over sets or distributions of varying sizes, inherently tackling arbitrary level of hierarchy noted as $m \in \mathbb{N}$~\cite{GaDuKaSe2018}. 
%
%Its applicability to mixed-variable domains depends on the distance employed in the optimal transport formulations.
%
%This explains the ? in the mixed-variable and number of parameters columns.
%
The drawback is the computational cost of $d(x,y)$, which involves solving an optimal transport problem with a worst-case complexity of $\mathcal{O}(N^3 \log N)$.
This is impractical for machine learning or optimization with distance-based models.

% Hierarhical
The hierarchical framework models multiple layers of hierarchy.
The variables of higher hierarchy influencing other variables are restricted to the categorical type.
%
%The number of parameters in the kernel scales with the number of variables.
%as for the variable-size framework.
%
The kernel function relies on pseudo-distances without the identity of indiscernibles, \textit{i.e.}, $d(x,y) = 0 \not \Rightarrow x=y$.
This is problematic, particularly in optimization, as distinct points with a distance of zero can lead to convergence issues or unreliable exploration.
In supervised learning, this may cause different inputs to be treated as identical, resulting in poorer generalization and potential numerical instability.

% Current
The framework of the current work fully supports mixed-variable domains, considers all variables between two points, even if not shared, and handles an arbitrary level of hierarchies. 
It generalizes the previous work~\cite{G-2022-11} in multiple aspects, as presented in Table~\ref{tab:comparison_methods}, and satisfies the overall objective introduced in Section~\ref{sec:objectives_organization}.

}

%--------------------------------------------------%
%\section{Graph-structured domains}
\section{Modeling framework for hierarchical domains based on meta variables}
\label{sec:notation}
%--------------------------------------------------%

In this section, hierarchical domains that generate heterogeneous datasets are formalized.  
% Roles 
In Section~\ref{sec:notation_roles}, the roles of variables are explicitly introduced.
% Excluded variables and extended point
Then, excluded variables and extended point, containing all variables whether they are excluded or not, are defined in Section~\ref{sec:excluded_variables_and_extended_point}.
% Graph 
In Section~\ref{sec:role_graph}, notions of graph theory are adapted for this work.
% Components and sets
%Afterwards, the components by roles and their role sets are detailed in Section~\ref{sec:components}.
Afterwards, the restricted sets, in which variables of the extended point belong, are detailed in Section~\ref{sec:admissible_set}.
% Graph-structured domain
%Subsequently, the extended domain $\overline{\mathcal{X}}$, in which the graph-structured distance $\edist_p:\overline{\mathcal{X}} \times \overline{\mathcal{X}} \to \overline{\mathbb{R}}^+$ is defined later on, is introduced in Section~\ref{sec:extended_domain}.
Subsequently, the extended domain $\overline{\mathcal{X}}$ is introduced in Section~\ref{sec:extended_domain}.
% MLP example 
%Finally, Section~\ref{sec:notation_working_example} models the MLP example with the content introduced in Sections~\ref{sec:notation_roles}~to~\ref{sec:extended_domain}.

%--------------------------------------------------%
\subsection{Roles of variables}
\label{sec:notation_roles}
%--------------------------------------------------%

The roles of variables are established from the the decree property that is generalized from~\cite{G-2022-11}.
% part 1: Meta and decreed
The following definition allows a variable to simultaneously have the decree property, and have its inclusion or admissible values determined by a decree dependency (by other variables with the decree property):~this is not allowed in~\cite{G-2022-11}.

%--- Def decree property ----%
\begin{mydef}[Decree property and decree dependency]
The decree property is attributed to variables whose values determine if other variables are included or excluded from a point $x \in \mathcal{X}$, or whose values determine the admissible values (or bounds) of other variables.
\medskip

\noindent A decree dependency refers to the inclusion or admissible values dependency of a variable with respect to an another variable with the decree property.
%the two variables are said to be linked by a decree dependency.
%
%Variables can be linked by multiple decree dependencies with different variables.
Variables can have multiple decree dependencies with different variables.

\label{def:decree_property_meta_variables}
\end{mydef}
%-----------%

In Section~\ref{sec:role_graph}, decree dependencies are viewed as parent-children dependencies, where the values of a parent variable determines the inclusion or admissible values of its children variables.
% --------- EXAMPLE shorten ---------- %
% part 0: MLP example
{\blr
In Figure~\ref{fig:intro_ex}, the optimizer $o$ has the decree property, since it determines the inclusion of, among others, the update $\alpha$, which the update $\alpha$ is included when $o=\texttt{ASGD}$ and excluded otherwise, hence it has has a decree dependency with the optimizer $o$ (parent).
}
%
%In Section~\ref{sec:admissible_set}, the admissible values of a variable are determined by respecting its decree dependencies for given values of its parent variables.
% Two decree properties
%Moreover, in Figure~\ref{fig:intro_ex}, the number of unit $u_i$ has two decree dependencies, one with the optimizer that influences is bounds and one with the number of hidden layers $l$ that controls its inclusion.
% --------- OLD EXAMPLE ---------- %

% --------- OLD EXAMPLE ---------- %
% part 0: MLP example
%In Figure~\ref{fig:working_example_table}, the optimizer $o$ has the decree property, since it determines the inclusion of, among others, the decay $\alpha_1$.
%
%The decay $\alpha_1$ has a decree dependency with the optimizer $o$ (parent), that is, the decay $\alpha_1$ is included when $o=\texttt{ASGD}$, and excluded otherwise.
%
%In Section~\ref{sec:admissible_set}, the admissible values of a variable are determined by respecting its decree dependencies for given values of its parent variables.
% Units influence the domain of the dropout
%The number of units in the hidden layers also have the decree property, since they influence the bounds of the dropout $\rho$ as presented in~\eqref{eq:dropout_bounds}.
% Dropout
%The dropout $\rho$ is always included, but its bounds are influenced by the values taken by the numbers of units and the number of hidden layers $l$.
% Decree dependencies between units and dropout
%The dropout $\rho$ has multiple decree dependencies, one with each number of units $u_i$ and one with the number of hidden layers $l$.
% --------- OLD EXAMPLE ---------- %

% 
Definition~\ref{def:decree_property_meta_variables} on the decree property and decree dependency establishes four possible cases, which are formalized as the roles of variables in the following definition.
%--- Def decree property ----%
\begin{mydef}[Roles of variables]
The role of a variable represents its relation to the decree property.
A variable is assigned one of the following roles:
\begin{enumerate}[leftmargin=*,labelindent=16pt]

 %
    %\item meta $(\meta)$, if it has the decree property, and neither its inclusion nor its admissible values are determined by another variable with the decree property;
    \item meta $(\meta)$, if it has the decree property, and has no decree dependency;
    %neither its inclusion nor its admissible values are determined by another variable with the decree property;
    
%
    %\item meta-decreed $(\metadec)$, if it has the decree property, and its inclusion or its admissible values are determined by the values of other variables with the decree property;
    \item meta-decreed $(\metadec)$, if it has the decree property, and has at least one decree dependency;
%
    %\item decreed $(\decreed)$, if it do not have the decree property, and its inclusion or its admissible values are determined by the values of other variables with the decree property;
    \item decreed $(\decreed)$, if it does not have the decree property, but has at least one decree dependency;
%
    %\item neutral $(\neutral)$, if it do not have the decree property, and neither its inclusion nor its admissible values are determined by another variable with the decree property.
    \item neutral $(\neutral)$, if it does not have the decree property nor decree dependency.  
    
\end{enumerate}
\label{def:roles_of_variables}
\end{mydef}
%-----------%

%}
% Meta variables encompasses strictly meta and meta-decreed, and decreed encompasses meta-decreed and strictly decreed
% In the following, referring to a meta variable means that it can either be strictly meta or meta-decreed, and referring to a decreed variable means that it can either be meta-decreed or strictly decreed.

% part 3: Type vs Role
Recall that the role of a variable must not be confused with its variable type. 
Each variable has its own variable type and is assigned its own role.
{\blr
Figure~\ref{fig:framework_MLP} further details the MLP example by adding roles and schematizing decree dependencies.
}
%
% -------------------------------- %
%\begin{figure}[htb!]
%    \centering
%    \scalebox{0.75}{\input{new_figs/framework_table_ex_reprise}}
%    \caption{MLP example: types and roles of the hyperparameters.}
%    \label{fig:framework_ex_reprise_for_roles}
%\end{figure}
%
\begin{figure}[htb!]
\begin{subfigure}[t]{0.6\textwidth}
    \centering
    \scalebox{0.725}{\begin{tikzpicture}[shift={(0, -2)}]
    % Table 1
    \node [shape=rectangle, align=center](table1) at (0,0) {
        \begin{tabular}{lccc} 
        \toprule
            HP  & Bounds & Type  & Role \\ \midrule
            Learning rate $(r)$ & $]0,1[$ & Cont & Neutral \\ 
            Optimizer $(o)$ & $\{ \texttt{ASGD}, \texttt{ADAM} \}$ & Cat & Meta \\
            \bottomrule
        \end{tabular}};

    % Table 2
    \node [shape=rectangle, align=center, xshift=-4.4cm, yshift=-3.25cm, at=(table1)](table2){
        \begin{tabular}{lccc} 
        \toprule
            HP  & Bounds  & Type & Role \\ \midrule
            Update $(\alpha)$  & $]0,1[$ & Cont & Decreed \\
            \#  layers $(l)$  & $\{0,1\}$ & Int & Meta-dec. \\ 
              \# units  $(u_i)$  & $U_{\texttt{ASGD}}$ & Int & Decreed \\
            \bottomrule
        \end{tabular}};

    % Table 3
    \node [shape=rectangle, align=center, xshift=0.5cm, yshift=-6.25cm, at=(table1)] (table3)  {
        \begin{tabular}{lccc} 
        \toprule
            HP  & Bounds  & Type & Role \\ \midrule
            Average $(\beta)$   & $]0,1[$ & Cont & Decreed \\
            \#  layers $(l)$    & $\{0,1,2\}$ & Int & Meta-dec. \\ 
             \# units  $(u_i)$  & $U_{\texttt{ADAM}}$ & Int & Decreed \\
            \bottomrule
        \end{tabular}};
   
    \draw[->] (table1)--(table2) node[midway, anchor=west, above, xshift=-2cm, yshift=-0.15cm] {if $o=\texttt{ASGD}$};
    \draw[->] ($(table1.south) + (5mm,0mm)$)--(table3) node[midway, anchor=east, above, xshift=2cm, yshift=-0.15cm] {if $o=\texttt{ADAM}$};
\end{tikzpicture}}
    \caption{Types and roles of the hyperparameters.}
    \label{subfig:framework_ex_reprise_for_roles}
\end{subfigure}
\hspace{0.1cm}
\begin{subfigure}[t]{0.4\textwidth}
    \centering
    \raisebox{2cm}{
        \scalebox{0.75}{ {\begin{tikzpicture}
% Minimum size for ellipses
\newcommand{\ellipseWidth}{16mm}  % Width of the ellipses
\newcommand{\ellipseHeight}{14mm} % Height of the ellipses

\begin{scope}[shift={(0, -2)}]

    % --- meta ---%
    \node[draw, shape=ellipse, minimum width=\ellipseWidth, minimum height=\ellipseHeight] at (4,10) (o) {\large $o$}; 
    
    % ------- %

    % --- meta-dec ---%
    \node[draw, shape=ellipse, minimum width=\ellipseWidth, minimum height=\ellipseHeight, xshift=0cm,  yshift=-2cm, at=(o)] (l) {\large $l$}; 

    % Dashed dependency arrow
    \draw[dash pattern=on 2pt off 2pt, line width=1pt, ->] (o)--(l);
    
    % ------- %

    % --- decreed ---% 
    \node[draw, shape=ellipse, minimum width=\ellipseWidth, minimum height=\ellipseHeight, xshift=-3.25cm,  yshift=-2.5cm+1cm, at=(l)] (a) {\large $\alpha$};   
    \node[draw, shape=ellipse, minimum width=\ellipseWidth, minimum height=\ellipseHeight, xshift=-1.35cm,  yshift=-2.5cm+1cm, at=(l)] (u1) {\large $u_1$};   
    \node[draw, shape=ellipse, minimum width=\ellipseWidth, minimum height=\ellipseHeight, xshift=1.35cm,  yshift=-2.5cm+1cm, at=(l)] (u2) {\large $u_2$};   
    \node[draw, shape=ellipse, minimum width=\ellipseWidth, minimum height=\ellipseHeight, xshift=3.25cm,  yshift=-2.5cm+1cm, at=(l)] (b) {\large $\beta$};
    
    % After nodes have appeared
    \draw[->, line width=1pt] (o)--(a);
    \draw[->, line width=1pt] (l)--(u1);
    \draw[->, line width=1pt] (l)--(u2);
    \draw[->, line width=1pt] (o)--(b);

    % Dashed dependency arrows
    \draw[dash pattern=on 2pt off 2pt, line width=1pt, ->] (o)--(u1.north);
    \draw[dash pattern=on 2pt off 2pt, line width=1pt, ->] (o)--(u2.north);

    % ------- %

    % --- neutral ---%
    \node[draw, shape=ellipse, minimum width=\ellipseWidth, minimum height=\ellipseHeight, xshift=-3cm,  yshift=0cm, at=(o)] (r) {\large $r$};  
    % ------- %
    
\end{scope}

\end{tikzpicture}}
        }
    }
    %\caption{MLP example: decree dependencies as arcs, solid for inclusion dependencies and and dotted for bounds dependencies.}
    \caption{Decree dependencies as arcs, solid for inclusion and dotted for values.}
    \label{subfig:framework_decree_dependencies}
\end{subfigure}
\caption{\blr Roles of the variables and decree dependencies in the MLP example.}
\label{fig:framework_MLP}
\end{figure}
% -------------------------------- %

% Comment on the figure
{\blr
%Figure~\ref{subfig:framework_ex_reprise_for_roles} adds a role column to the tables in Figure~\ref{fig:intro_ex}.
%
Figure~\ref{subfig:framework_decree_dependencies} illustrates the decree dependencies between hyperparameters of the MLP example, using solid lines for inclusion dependencies and dotted lines for admissible values dependencies.}
%
% Example meta: optimizer
In the MLP example, the optimizer $o \in \{ \texttt{ADAM}, \texttt{ASGD} \}$ is a categorical variable that is assigned the meta role, since it determines the inclusion and the admissible values of other variables (decree propriety), and neither its inclusion nor its admissible values are determined by other variables (no decree dependency).
% Simplicity
For convenience, a variable and its role are referred similarly as a variable and its type, \textit{e.g.}, the optimizer is referred as a meta categorical variable. 
% Exemple meta-decreed: the number of hidden layers
%
{\blr The number of layers $l \in L_o$ is a meta-decreed integer variable, since its admissible values are determined by the optimizer $o$ (decree dependency), and it determines the inclusion of the number of units $u_1, u_2 \in U_o$, where
%
%\begin{align}
%\begin{split}
%L_o \coloneq
%\begin{cases}
%\{0,1\}     &\text{ if } o=\texttt{ASGD}, \\
%\{0,1, 2\}  &\text{ if } o=\texttt{ADAM}. \\
%\end{cases}
%\end{split}
%\label{eq:layers_set}
%\end{align}
%

\begin{center}
\begin{minipage}{.4\linewidth}
    \begin{align*}
        L_o \coloneq
        \begin{cases}
        \{0,1\}     &\text{ if } o=\texttt{ASGD}, \\
        \{0,1, 2\}  &\text{ if } o=\texttt{ADAM}, \\
        \end{cases}
        %\label{eq:layers_set}
    \end{align*}
\end{minipage}%
\quad \text{ and } \quad
\begin{minipage}{.4\linewidth}
     \begin{align*}
        U_o \coloneq
        \begin{cases}
        U_{\texttt{ASGD}}     &\text{ if } o=\texttt{ASGD}, \\
        U_{\texttt{ADAM}}  &\text{ if } o=\texttt{ADAM}. \\
        \end{cases}
        %\label{eq:units_set}
    \end{align*}
\end{minipage}\\~
\end{center}%
}
%$u_1, u_2, \ldots, u_{l}$ (decree property).
% Exemple decreed:

\noindent The update $\alpha$ is a decreed continuous variable, since it does not have the decree property, and its inclusion is determined by the optimizer $o$.
%
%Finally, the activation function $a \in \{ \texttt{ReLU}, \texttt{Sig}, \texttt{Tanh} \}$ is a neutral categorical variable, as it does not have the decree property, and it has no decree dependency.
Finally, the learning rate $r \in \, ]0,1[$ is a neutral categorical variable, as it does not have the decree property, and it has no decree dependency.

%--------------------------------------------------%
\subsection{Excluded variables and extended point}
\label{sec:excluded_variables_and_extended_point}
%--------------------------------------------------%

% Old modeling paper works only with included variables
In previous work~\cite{G-2022-11}, a point contains only variables that are included for the given values of variables with the decree property.
% Now, we also contain excluded variables
%In this work, distance functions also consider the variables that are excluded, since it provides useful information for computing distances between two points of the domain that do not share the same variables.
In this work, variables that are excluded are also considered, since 1) it provides useful information for computing distances between two points of the domain $\mathcal{X}$ that do not share the same variables, and 2) it facilitates the computations themselves.
This last remark leads to the following definition. 

%--- Def excluded variable ----%
\begin{mydef}[Excluded variable]
    %An excluded variable is a meta-decreed or decreed variable, that, for given variables with the decree property, is not included in the given point $x \in \mathcal{X}$, but is included in at least one other point $y \in \mathcal{X}$.
    An excluded variable is a meta-decreed or decreed variable, that, for the given values of the variables associated to its decree dependencies, is not included in the given point $x \in \mathcal{X}$, but is included in at least one other point $y \in \mathcal{X}$.
%\noindent 
An excluded variable is assigned the special value {\normalfont $\texttt{EXC}$} and its variable type is conserved.
\label{def:excluded_variable}
\end{mydef}
%-----------%

%
In the MLP example, the update $\alpha$ is a decreed continuous variable, which is excluded when the optimizer $o=\texttt{ADAM}$, whereas $\alpha \in \ ]0,1[$ (included) when $o=\texttt{ASGD}$.
%
%Hence, the decay $\alpha_1$ is a decreed continuous variable.
%whether it is included or excluded. 
% All variables of a domain
Definition~\ref{def:excluded_variable} is introduced to allow the meta distance to consider every variable that is included in at least one point of the domain, collectively referred as all the included and excluded variables.
%
%In the MLP example, all the variables of the domain $\mathcal{X}$ are: the learning rate $r$; the activation function $a$; the optimizer $o$; the decay $\alpha_1$; the power update $\alpha_2$; the average start $\alpha_3$; the number of hidden layers $l$; the number of units $u_1,u_2, \ldots, u_{l_{\text{max}}}$; the running average 1 $\beta_1$; the running average 2 $\beta_2$; the numerical stability $\beta_3$ and the dropout $\rho$, with $l_{\text{max}} = \max \left(L_{\texttt{ASGD}} \cup L_{\texttt{ADAM}}\right)$ as the highest possible value for the number of hidden layers $l$.
%
The definition of an extended point formalizes the previous sentence.

%--- Def point ----%
\begin{mydef}[An extended point]
An extended point $\overline{x}$ contains all the included and excluded variables of a corresponding point $x \in \mathcal{X}$. For $r \in R := \{\meta, \metadec, \decreed, \neutral \}$ and $i \in I^r := \{1,2, \ldots, n^r \}$, the $i$-th variable assigned to the role $r$ is noted $\overline{x}_i^r$, where $n^r \in \mathbb{N}$ is the number of variables assigned to the role $r$.

%\noindent For $r \in \{\meta, \metadec, \decreed, \neutral \}$, the role $r$ component $\overline{x}^r$ is a $n^{r}$-tuplet that contained all the included and excluded variables, which are assigned to the role $r$, such that 
%
%\begin{equation*}
%    \overline{x}^r := \left( \overline{x}^r_1, \overline{x}^r_2, \ldots, \overline{x}^r_{n^r} \right),
    %\label{eq:component}
%\end{equation*}
%
%where $\overline{x}_i^r$ is the $i$-th variable of the role $r$ component in the extended point $\overline{x}$ and $n^r \in \mathbb{N}$ is the number of variables assigned to the r.  

\label{def:extended_point_and_components}
\end{mydef}
%-----------%

%Few details of Definition~\ref{def:extended_point_and_components} are discussed next.}
% 1) Overline notation 
{\blr
The bar notation of an extended point $\overline{x}$ outlines that a corresponding point $x$ is conceptually extended to incorporate excluded variables.
This notation is inspired by the extended set of real numbers $\overline{\mathbb{R}}= \mathbb{R} \cup \pm \{\infty \}$.
%referred to as the \textit{extended real-valued} set.
}
%
%The bar notation is not used for the variables, since they are not technically extended themselves, but are rather assigned the special value $\texttt{EXC}$ when excluded.
%
A point $x \in \mathcal{X}$ contains only included variables, but its variables can be attributed roles, similarly as an extended point $\overline{x}$.
%One-to-one correspondence
%Note that there is a bijection between a point $x \in \mathcal{X}$ and an extended point $\overline{x} \in \overline{\mathcal{X}}$: more details are given in Section~\ref{sec:extended_domain}.
%

%2) Meta/neutral no difference for extended or normal pt
Second, meta and neutral variables are always included, hence there is no distinction between these variables whether they are part of an extended point $\overline{x}$ or of a point $x$, \textit{i.e.}, $\overline{x}_i^r = x_i^r$ for $r \in \{\meta, \neutral\}$ and $i \in I^r$.
The admissible values of meta and neutral variables are also fixed.
In contrast, meta-decreed $\overline{x}_i^{\metadec}$ and decreed $\overline{x}_j^{\decreed}$ variables of an extended point $\overline{x}$ may be excluded from a point $x$, and/or their admissible values may differ between points.
%
%For the sake of the presentation, the term extended is not textually mentioned for these components and their sets, especially since the bar notation is always used.

% 3) n-tuplet
%Third, a $r$-role component $\overline{x}^r$ of an extended point $\overline{x}$ is defined as a $n^r$-tuplet instead of a vector, since meta-decreed $\overline{x}^{\metadec}$ and decreed $\overline{x}^{\decreed}$ components may contain excluded variables that take the special value $\texttt{EXC}$.
%
%The notion of additivity and scalar multiplication is not well established for this special value, hence the terminology related to vector spaces is avoided.

% 4)
Third, the framework proposed in~\cite{G-2022-11} is a special case of the one developed in this work, as it includes meta variables but lacks meta-decreed variables.
The special case can be recovered easily by removing the meta-decreed variables: this is convenient as the special case provides a specialized framework for problems of great interests, such as most hyperparameter optimization problems.

Fourth, meta-decreed and decreed variables can be seen as bounds-dependent variables and, in that case, the bounds of an excluded variable are restricted to the empty set.

% 6) Variables types and justification for point by roles: distance
Fifth, both the roles and types of a variables carry important information for properly computing distances between variables in Section~\ref{sec:distance}.
% Each component can be repartitioned into the variables types.
%Fortunately, each $r$-role component $\overline{x}^r$ can itself be partitioned into components by types, such as
% x^r repartitioned with types
%\begin{equation}
% \overline{x}^r = \left(\overline{x}^r_{\cat}, \overline{x}^r_{\integer}, \overline{x}^r_{\continuous} \right),
% \label{eq:role_type_component}
%\end{equation}
%
%where $r \in \{\meta, \metadec, \decreed, \neutral \}$ and, for $t \in \{ \cat, \integer, \continuous \}$, $\overline{x}^r_{t}$ is the $r$-role $t$-type component of the extended point $\overline{x}$. 
%
%
To avoid a cumbersome notation, variable types are not explicit, but they are implicitly considered in the computation of distances in Section~\ref{sec:distance}.

{\blr 
The following figures schematize two extended points of the MLP example.
Each figure builds upon Figure~\ref{subfig:framework_decree_dependencies} by assigning values to the variables.
%while maintaining the representation of decree dependencies as arcs.
}
% -------------------------------- %
\begin{figure}[htb!]
\begin{subfigure}[t]{0.5\textwidth}
    \centering
    \scalebox{0.675}{\begin{tikzpicture}
% Minimum size for ellipses
\newcommand{\ellipseWidth}{16mm}  % Width of the ellipses
\newcommand{\ellipseHeight}{14mm} % Height of the ellipses

    % --- meta ---%
    \node[draw, shape=ellipse, minimum width=\ellipseWidth, minimum height=\ellipseHeight] at (4,10) (o) {\large $\overline{o} = \texttt{ADAM}$}; 
    
    % ------- %

    \begin{scope}[shift={(0, 0)}]
        % --- meta-dec ---%
        \node[draw, shape=ellipse, minimum width=\ellipseWidth, minimum height=\ellipseHeight, xshift=0cm,  yshift=-2cm, at=(o)] (l) {\large $\overline{l} = 1$}; 
    
        % Dashed dependency arrow
        \draw[dash pattern=on 2pt off 2pt, line width=1pt, ->] (o)--(l);
        
        % ------- %
    
        % --- decreed ---% 
        \node[draw, shape=ellipse, minimum width=\ellipseWidth, minimum height=\ellipseHeight, xshift=-4.75cm,  yshift=-2.5cm+1cm, at=(l)] (a) {\large $\overline{\alpha} = \texttt{EXC}$};   
        \node[draw, shape=ellipse, minimum width=\ellipseWidth, minimum height=\ellipseHeight, xshift=-1.66cm,  yshift=-2.5cm+1cm, at=(l)] (u1) {\large $\overline{u}_1 = 100$};   
        \node[draw, shape=ellipse, minimum width=\ellipseWidth, minimum height=\ellipseHeight, xshift=1.66cm,  yshift=-2.5cm+1cm, at=(l)] (u2) {\large $\overline{u}_2 = \texttt{EXC}$};   
        \node[draw, shape=ellipse, minimum width=\ellipseWidth, minimum height=\ellipseHeight, xshift=4.75cm,  yshift=-2.5cm+1cm, at=(l)] (b) {\large $\overline{\beta} = 0.5$};

        % Dashed dependency arrows
        \draw[dash pattern=on 2pt off 2pt, line width=1pt, ->] (o)--(u1.north);
        \draw[dash pattern=on 2pt off 2pt, line width=1pt, ->] (o)--(u2.north);
    
        % ------- %
    
        % --- neutral ---%
        \node[draw, shape=ellipse, minimum width=\ellipseWidth, minimum height=\ellipseHeight, xshift=-4cm,  yshift=0cm, at=(o)] (r) {\large $\overline{r} = 0.1$};  
        % ------- %

    \end{scope}

    % After nodes have appeared
    \draw[->, line width=1pt] (o)--(a);
    \draw[->, line width=1pt] (l)--(u1);
    \draw[->, line width=1pt] (l)--(u2);
    \draw[->, line width=1pt] (o)--(b);

\end{tikzpicture}}
    \caption{Extended point $\overline{x}$ with \text{ADAM} and one layer.}
    \label{subfig:extended_point_ADAM}
\end{subfigure}
\hspace{0.1cm}
\begin{subfigure}[t]{0.5\textwidth}
    \centering
    \scalebox{0.675}{\begin{tikzpicture}
% Minimum size for ellipses
\newcommand{\ellipseWidth}{16mm}  % Width of the ellipses
\newcommand{\ellipseHeight}{14mm} % Height of the ellipses

    % --- meta ---%
    \node[draw, shape=ellipse, minimum width=\ellipseWidth, minimum height=\ellipseHeight] at (4,10) (o) {\large $\overline{o}' = \texttt{ASGD}$}; 
    
    % ------- %

    % --- meta-dec ---%
    \node[draw, shape=ellipse, minimum width=\ellipseWidth, minimum height=\ellipseHeight, xshift=0cm,  yshift=-2cm, at=(o)] (l) {\large $\overline{l}' = 0$}; 

    % Dashed dependency arrow
    \draw[dash pattern=on 2pt off 2pt, line width=1pt, ->] (o)--(l);
    
    % ------- %

    % --- decreed ---% 
    \node[draw, shape=ellipse, minimum width=\ellipseWidth, minimum height=\ellipseHeight, xshift=-4.75cm,  yshift=-2.5cm+1cm, at=(l)] (a) {\large $\overline{\alpha}' = 0.3$};   
    \node[draw, shape=ellipse, minimum width=\ellipseWidth, minimum height=\ellipseHeight, xshift=-1.66cm,  yshift=-2.5cm+1cm, at=(l)] (u1) {\large $\overline{u}_1' = \texttt{EXC}$};   
    \node[draw, shape=ellipse, minimum width=\ellipseWidth, minimum height=\ellipseHeight, xshift=1.66cm,  yshift=-2.5cm+1cm, at=(l)] (u2) {\large $\overline{u}_2' = \texttt{EXC}$};   
    \node[draw, shape=ellipse, minimum width=\ellipseWidth, minimum height=\ellipseHeight, xshift=4.75cm,  yshift=-2.5cm+1cm, at=(l)] (b) {\large $\overline{\beta}' = \texttt{EXC}$};
    
    % After nodes have appeared
    \draw[->, line width=1pt] (o)--(a);
    \draw[->, line width=1pt] (l)--(u1);
    \draw[->, line width=1pt] (l)--(u2);
    \draw[->, line width=1pt] (o)--(b);

    % Dashed dependency arrows
    \draw[dash pattern=on 2pt off 2pt, line width=1pt, ->] (o)--(u1.north);
    \draw[dash pattern=on 2pt off 2pt, line width=1pt, ->] (o)--(u2.north);

    % ------- %

    % --- neutral ---%
    \node[draw, shape=ellipse, minimum width=\ellipseWidth, minimum height=\ellipseHeight, xshift=-4cm,  yshift=0cm, at=(o)] (r) {\large $\overline{r} = 0.01$};  
    % ------- %

\end{tikzpicture}}
    \caption{Extended point $\overline{x}'$ with \text{ASGD} and no layer.}
    \label{subfig:extended_point_ASGD}
\end{subfigure}
\caption{\blr Two different extended points in the MLP example.}
\label{fig:extended_points}
\end{figure}
% -------------------------------- %

{\blr
In Figure~\ref{fig:extended_points}, the two extended points contain all included and excluded variables of the domain, facilitating variable comparison.
For instance, the update $\overline{\alpha}$ remains comparable, while being excluded in the extended point $\overline{x}$ on the left.
Both extended points share the same arc representations.
In the next section, elements of graph theory formalize a common graph structure shared by all extended points of a problem.}

%Elements of graph theory are introduced in Section~\ref{sec:role_graph}.
%Then, further developing the variables of an extended point and their sets in Section~\ref{sec:admissible_set}. %
%The extended domain $\overline{\mathcal{X}}$ is defined in Section~\ref{sec:extended_domain}, which allows the presentation of the graph-structured distance $\edist_p:\overline{\mathcal{X}} \times \overline{\mathcal{X}} \to \overline{\mathbb{R}}^+$ in Section~\ref{sec:distance}. 

% -------- Role graph --------- %
\subsection{Notions from graph theory}
\label{sec:role_graph}
% ------------------------------- %
 
At this stage of the work, Definition~\ref{def:decree_property_meta_variables} allows meta-decreed variables to impact each others through decree dependencies.
Consequently, decree dependencies between meta-decreed variables can lead to circular reasoning or contradiction.
Indeed, this can be shown with a simple example with only two meta-decreed binary variables: 
\begin{center}
\begin{minipage}{.4\linewidth}
\begin{equation*}
  x_1^{\metadec} =
  \begin{cases}
      0 \quad \text{ if } x_2^{\metadec}=1, \\
      1 \quad \text{ if } x_2^{\metadec}=0, \\ 
  \end{cases}
\end{equation*}
\end{minipage}%
\begin{minipage}{.4\linewidth}
\begin{equation*}
  x_2^{\metadec} =
  \begin{cases}
      0 \quad \text{ if } x_1^{\metadec}=0, \\
      1 \quad \text{ if } x_1^{\metadec}=1, \\ 
  \end{cases}
\end{equation*}
\end{minipage}\\~
\end{center}%

\noindent where $x_1^{\metadec}=0 \Rightarrow x_2^{\metadec}=0 \Rightarrow x_1^{\metadec}=1 \neq 0$ (contradiction).
%
% This is accomplished with a principled-approach based on graph theory that takes into account the decree property between the meta and decreed variables.
To avoid these problematic, and uncommon situations, an assumption about the decree dependencies must be introduced. 
Beforehand, the role graph, which among other things allows to formulate the assumption, is defined below.

%--- Def role graph ----%
\begin{mydef}[Role graph]
The role graph $G = \left(V, A \right)$ is a graph structure, where
%that contains the decree dependencies between all the variables of a domain $\mathcal{X}$, where
\begin{itemize}[leftmargin=*,labelindent=16pt]

    \item $V$ is the set of variables that contains all the included and excluded variables, represented as nodes,
    
    \item $A$ is the set of decree dependencies that contains references for all inclusion-exclusion and admissible values dependencies
    between all the included and excluded variables, represented as arcs. 
    
\end{itemize}

\noindent An arc $a \in A$, which refers to a decree dependency, connects a parent (variable) to a child (variable), whose inclusion or admissible values are influenced by the parent.
A parent is either a meta or meta-decreed $\overline{x}^{r}_i$, and a child is either meta-decreed or decreed, such that $a=(\overline{x}^{r}_i, \overline{x}^{r'}_j)$, where $r \in \{\meta, \metadec \}$ and $r' \in \{ \metadec, \decreed \}$, with $i \neq j$ when $r=r'=\metadec$.
A child can have multiple parents, and vice versa.

% \noindent A decree dependency is an arc $a \in A$ that connects a meta or meta-decreed variable $\overline{x}^{r}_i$ (parent) to a corresponding meta-decreed or decreed variable $\overline{x}^{r'}_j$ (child), whose inclusion or admissible values are influenced by $\overline{x}^{r}_i$, such that $a=(\overline{x}^{r}_i, \overline{x}^{r'}_j)$, where $r \in \{\meta, \metadec \}$, $r' \in \{ \metadec, \decreed \}$, and $i \neq j$ when $r=\metadec=r'$.

\label{def:role_graph}
\end{mydef}
% ---------------- %

{\blr 
In the MLP example, the set of nodes correspond to all included and excluded variables, \textit{i.e.}, 
$V= \left\{ \overline{o}, \overline{l}, \overline{\alpha}, \overline{\beta}, \overline{u}_1, \overline{u}_2, \overline{r} \right\}$.
%$V= \left\{o, l, \alpha, \beta, u_1, u_2, r \right\}$.
%
The arcs of the graph represent pairs of parent-child variables, in which a child has a decree dependency with its parent.
From Figure~\ref{fig:extended_points}, the set of decree dependencies in the example is 
$A= \left\{ \left(\overline{o}, \overline{l} \right), \left(\overline{o}, \overline{\alpha} \right), \left(\overline{o}, \overline{\beta} \right), \left(\overline{o}, \overline{u}_1 \right), \left(\overline{o}, \overline{u}_2 \right), \left(\overline{l}, \overline{u}_1 \right), \left(\overline{l}, \overline{u}_2 \right)  \right\}$.
%$A= \left\{ \left(o, l \right), \left(o, \alpha \right), \left(o, \beta\right), \left(o, u_1 \right), \left(o, u_2 \right), \left(l, u_1 \right), \left(l, u_2 \right)  \right\}$.
%
The influence of parents on the inclusion and/or admissible values of their children is formalized with these arcs in the next section.
}

Now that the role graph $G$ is properly defined, the assumption discarding situations with circular reasoning or contradiction is introduced.
 
%-- Assumption DAG ---%
\begin{assumption}
The role graph $G$ is a Directed Acyclic Graph (DAG).

\label{hyp:DAG}
\end{assumption}
% ----------------- %

%Comment on the assumption
Assumption~\ref{hyp:DAG} ensures that the nodes in the role graph $G$ are (partially) ordered.
Circular decree dependencies, as presented in the example with two meta-decreed variables, are forbidden. 
To determine such a partial order, it suffices to apply a topological sort on $G$.
%the role graph $G$. 
%

Under Assumption~\ref{hyp:DAG}, the role graph $G$ is a data structure that: 1) contains all the included and excluded variables in the set of variables $V$, 2) contains references for all inclusion-exclusion or admissible values dependencies in the set of decree dependencies $A$, and 3) establishes the roles of variables according to the positions of nodes in the DAG:
\begin{itemize}
\itemsep0em 

    \item a meta variable $\overline{x}_i^{\meta}$ is a root node;

    \item a meta-decreed variable $\overline{x}_i^{\metadec}$ is an internal node with at least one parent and one child;

    \item a decreed variable $\overline{x}_i^{\decreed}$ is a leaf node with at least one parent and no child;

    \item a neutral variable $\overline{x}_i^{\neutral}$ is an isolated node.
\end{itemize}

{\blr Figure~\ref{fig:graph_ex} schematizes the positions of variables based on their role.}
\begin{figure}[htb!]
\centering
    \scalebox{0.7}{\begin{tikzpicture}[
    node distance=2.5cm and 2.8cm, % Adjusted spacing
    every node/.style={draw, circle, minimum size=12mm, font=\large, align=center} % No global arrow
]

    % --------------------------- %
    % Meta variables (top layer)
    \node[draw, rectangle, minimum width=2.5cm, minimum height=1cm, anchor=east] (meta) at (-2,6) {Meta variables};
    \begin{scope}[shift={(-2, 0)}]
        %\node (a) at (0,6) {$x_1^{\meta}$};
        \node (b) at (3,6) {$\overline{x}_1^{\meta}$};
        \node (c) at (6,6) {$\overline{x}_2^{\meta}$};
        \node (d) at (9,6) {$\overline{x}_3^{\meta}$};
    \end{scope}
    % ------------------------ %

    % ------------------------ %
    % Meta-decreed variables (middle layer)
    \begin{scope}[shift={(-2, -0.4)}]
        %\node (e) at (1.5,4.5) {$x_1^{\metadec}$};  % Positioned between a and b
        \node (f) at (4.5,4.5) {$\overline{x}_1^{\metadec}$};  % Perfectly centered between b and c

        %\node (g) at (0,3) {$x_3^{\metadec}$};
        \node (h) at (3,3) {$\overline{x}_2^{\metadec}$};
        \node (i) at (6,3) {$\overline{x}_3^{\metadec}$};
      
    \end{scope}

    \begin{scope}[shift={(0, -0.4)}]
     % Meta-decreed variables box (increased height)
        \draw[dashed, thick, rounded corners] (-0.75, 5.2) rectangle (9.75, 2.3);
        %
        % Proper bracket for meta-decreed variables (covering e and g)
        \draw [decorate,decoration={brace,amplitude=7pt,mirror}, line width=0.8pt] (-1.25, 5) -- (-1.25, 2.5);

        \node[draw, rectangle, minimum width=2.5cm, minimum height=1cm, anchor=east] (meta_dec) at (-2, 3.75) {Meta-decreed variables};

    \end{scope}

    % ---------------------- %

    % ------------------------------ %
    % Decreed variables (aligned with "Decreed variables" label)
    \begin{scope}[shift={(-2, -1)}]
        %
        %\node (j) at (0,1.5) {$x_1^{\decreed}$};
        \node (k) at (4.5,1.5) {$\overline{x}_1^{\decreed}$};
        \node (m) at (9,1.5) {$\overline{x}_2^{\decreed}$}; % Lowered to align with j and k
    \end{scope}

    \node[draw, rectangle, minimum width=2.5cm, minimum height=1cm, anchor=east] (dec) at (-2, 0.5) {Decreed variables};
    % ----------------------------- %

    % -------------------------------- %
    % Neutral variables (aligned with "Neutral variables" label)
    \begin{scope}[shift={(0.5, -1.25)}]
        \node[draw, rectangle, minimum width=2.5cm, minimum height=1cm, anchor=east] (neutral) at (-2.4,0) {Neutral variables};
        \node (n) at (2-1.5,0) {$\overline{x}_1^{\neutral}$};
        \node (l) at (6.5,0) {$\overline{x}_3^{\neutral}$};
        \node (new) at (4.25-0.5,0) {$\overline{x}_2^{\neutral}$};
    \end{scope}
    % ------------------------------- %

    % Arrows (dependencies) - Now explicitly defined
    %\draw[->] (a) -- (e);
    %\draw[->] (a) -- (g);
    \draw[->, dash pattern=on 2pt off 2pt, line width=1pt] (b) -- (h);
    
    \draw[->, line width=1pt] (b) -- (f);

    \draw[->, line width=1pt] (c) -- (f);
    \draw[->, line width=1pt] (d) -- (m);
    
    %\draw[->] (e) -- (g);
    %\draw[->] (e) -- (j);
    \draw[->, line width=1pt] (f) -- (h);
    \draw[->, line width=1pt] (f) -- (i);
    
    %\draw[->] (g) -- (j);
    \draw[->, dash pattern=on 2pt off 2pt, line width=1pt] (h) -- (k);
    \draw[->, dash pattern=on 2pt off 2pt, line width=1pt] (i) -- (k);

    % ----------------------------- %
    % Dotted Rectangles (Encapsulation)
    % Meta variables box (increased height)
    \draw[dashed, thick, rounded corners] (-0.75, 6.7) rectangle (9.75, 5.3);

    % Decreed variables box (lowered and adjusted height)
    \begin{scope}[shift={(0, -0.8)}]
        \draw[dashed, thick, rounded corners] (-0.75, 2) rectangle (9.75, 0.6);
    \end{scope}

    % Neutral variables box (lowered further and adjusted height)
    \begin{scope}[shift={(0, -1.05)}]
        \draw[dashed, thick, rounded corners] (-0.75, 0.5) rectangle (9.75, -0.9);
    \end{scope}
    % ----------------------------- %

\end{tikzpicture}}
\caption{\blr Example of a role graph showing positions of variables based on their role. 
%where solid and dotted arcs respectively represent inclusion and bounds decree dependencies
}
\label{fig:graph_ex}
\end{figure}
{\blr
In this simple example, there are three levels of hierarchy, as the longest path contains three edges.
The graph is acyclic, ensuring no circular dependencies are present.
The example contains two levels of meta-decreed variables.
In general, there can be an arbitrary number of levels.
Disconnected trees can exist, meaning not all variables must be part of a single connected structure.
Finally, decree dependencies (arcs) do not need to follow a strict layer-by-layer structure.
For instance, $x_1^{\meta}$ has arcs connecting to both $x_1^{\metadec}$ and $x_2^{\metadec}$, which are on different levels.}

%The decree dependencies in the role graph $G$ implies that the inclusion and/or admissible values of some variables are determined by the values taken by their parents, which themselves can be determined by their own parents.
%
%Hence, to determine the inclusion and/or the admissible of values of these variables, it is necessary to consider the values taken by their ancestors, that is, their parents, grandparents, great-grandparents, and so on. 

%
Recall that meta-decreed and decreed variables are child variables whose inclusions or admissible values are determined by the values of {\blr their parents composing their decree dependencies.}
The role graph $G$ outlines that the inclusion and/or admissible values of a variable can be determined by multiple different decree dependencies, that is, from multiple parents.
Hence, to determine the inclusion and/or the admissible values of a variable, it is necessary to consider simultaneously all the values of its parents.
In the MLP example, the number of units $\overline{u}_1$ must consider simultaneously the values of its parents the optimizer $\overline{o}$, for its admissible values, and the number of layers $\overline{l}$, for its inclusion.
The following definition introduces formally the notion of the parents in our context.
%then the ancestors are defined next from the parents.}

\begin{mydef}[Parents]
    For $r \in R$ and $i \in I^r$, the parents $\overline{\parents}_i^r$ of the variable $\overline{x}_i^r$ is the subset of variables for which there exists an arc from those variables to $\overline{x}_i^r$, \textit{i.e.},
    \begin{equation}
        \overline{\parents}_i^r := \left\{ \overline{v} \in V \ : \ \left(\overline{v}, \overline{x}_i^r  \right) \in A \right\}. 
    \end{equation}    
    
    \label{def:parents} 
\end{mydef}

The parents are used to handle the inclusion-exclusion or the admissible values of meta-decreed and decreed variables in the
next section.
% On dit in the next section
% ou bien in Section~\ref{sec:admissible_set}
% mais pas in the next Section~\ref{sec:admissible_set}.
%
Note that although meta and neutral variables have no parents, they are still defined (as empty sets) for these roles.
This is allows to propose a concise expression for the extended domain $\mathcal{X}$ in Section~\ref{sec:extended_domain}.

% Justification for ancestors
For a given variable, the inclusion or admissible values of its parents can be determined by their own parents (\textit{i.e.}, grandparents).
%
%In Section~\ref{sec:admissible_set},
In next section, the ancestors of a variable, representing possible multiple generations of parents and grandparents, will be used to determine all the values that such variable can take across all possible extended points.
%
% ---------- NEW --------- %
{\blr
The next example removes some hyperparameters and adds the dropout $\rho$ to the MLP example.
This variant justifies the need to consider ancestors when determining all the possible values.% .
}
\begin{figure}[htb!]
\begin{subfigure}[t]{0.65\textwidth}
    \centering
    \scalebox{0.7}{\begin{tikzpicture}
    % Table 1
    \node [shape=rectangle, align=center](table1) at (0,0) {
        \begin{tabular}{lccc} 
        \toprule
            HP  & Bounds  & Role \\ \midrule
            Optimizer $(o)$ & $\{ \texttt{ASGD}, \texttt{ADAM} \}$  & Meta \\
            \bottomrule
        \end{tabular}};

    % Table 2
    \node [shape=rectangle, align=center, xshift=-3.75cm, yshift=-2.75cm, at=(table1)](table2){
        \begin{tabular}{llccc} 
        \toprule
            \multicolumn{2}{l}{HP}  & Bounds   & Role \\ \midrule
            \multicolumn{2}{l}{\#  layers $(l)$}  & $\{0,1\}$  & Meta-dec. \\ 
             &  \# units  $(u_i)$  & $U_{\texttt{ASGD}}$  & Decreed \\
            \bottomrule
        \end{tabular}};

    % Table 3
    \node [shape=rectangle, align=center, xshift=3.75cm, yshift=-2.75cm, at=(table1)] (table3)  {
        \begin{tabular}{llccc} 
        \toprule
            \multicolumn{2}{l}{HP}  & Bounds   & Role \\ \midrule
            \multicolumn{2}{l}{\#  layers $(l)$}    & $\{0,1,2\}$  & Meta-dec. \\ 
             &  \# units  $(u_i)$  & $U_{\texttt{ADAM}}$  & Decreed \\
            \bottomrule
        \end{tabular}};
   
    \draw[->] (table1)--(table2) node[midway, anchor=west, above, xshift=-2cm, yshift=-0.15cm] {if $o=\texttt{ASGD}$};
    %
    
    %\draw[->] ($(table1.south) + (5mm,0mm)$)--(table3) node[midway, anchor=east, above, xshift=2cm, yshift=-0.15cm] {if $o=\texttt{ADAM}$};
    \draw[->] (table1)--(table3) node[midway, anchor=east, above, xshift=2cm, yshift=-0.15cm] {if $o=\texttt{ADAM}$};

      % ----------- center leaf ---------%
      %\node [shape=rectangle, align=center, xshift=5.15cm, yshift=-3.75cm, at=(table2)](table5){
      %  \begin{tabular}{lcc} 
      %  \toprule
      %      HP & Variable & Bounds   \\ \midrule
      %      Dropout & $\rho$ & $\Bigg[ \frac{\sum \limits_{i=1}^{l} u_i - \tau_{\text{min}}}{\tau_{\text{max}}} , \frac{\sum \limits_{i=1}^{l} u_i}{\tau_{\text{max}}} \Bigg]$ \\
      %      \bottomrule
       % \end{tabular}};
    \node [shape=rectangle, align=center, xshift=0cm, yshift=-5.5cm, at=(table1)](table5){
        \begin{tabular}{lcc} 
        \toprule
            HP & Bounds & Role   \\ \midrule
            Dropout ($\rho$) & $\left[    
           0, \quad  \frac{1}{2\tau_{\text{max}}} \sum_{i=1}^{l} u_i 
            \right]$  & Decreed
            \\
            \bottomrule
        \end{tabular}};

       \draw[->] (table2)--(table5) node[midway, anchor=east, above, xshift=1.85cm, yshift=-0.25cm] {};

        \draw[->] (table3)--(table5) node[midway, anchor=west, above, xshift=-1cm] {};
    % -----------------------------------------------%
    
\end{tikzpicture}}
    \caption{Bounds and role for the variant.}
    \label{subfig:dropout_bounds}
\end{subfigure}
\hspace{0.1cm}
\begin{subfigure}[t]{0.35\textwidth}
    \centering
    \scalebox{0.65}{\begin{tikzpicture}

% Minimum size
\newcommand{\circleSize}{12mm}

    % --- meta ---%
    %\node [draw, anchor=west] (meta) at (0,10) {\begin{tabular}{c} Meta $(\meta)$\end{tabular} };   
    %
    \node[draw, shape=circle, minimum size=\circleSize] at (4,10) (o) {\Large $\overline{o}$}; 
    % ------- %

    % --- meta-dec ---%
    %\node [draw, anchor=west] (metadec) at (0,9) {\begin{tabular}{c} Meta-dec $(\metadec)$  \end{tabular} };
    %
    \node[draw, shape=circle, minimum size=\circleSize, xshift=0cm,  yshift=-2.25cm, at=(o)] (l) {\large $\overline{l}$}; 
    %
    
    % After nodes 
    \draw[dash pattern=on 2pt off 2pt, line width=1pt, ->] (o)--(l);
    % ------- %

    % --- decreed ---% 
    %\node [draw, anchor=west] (dec) at (0,8) {\begin{tabular}{c} Decreed $(\decreed)$ \end{tabular} };
    %
    %\node[draw, shape=circle, minimum size=\circleSize, xshift=-3.5cm,  yshift=-2.5cm, at=(l)] (a) {\Large $\alpha$};   
    %
    \node[draw, shape=circle, minimum size=\circleSize, xshift=-2cm,  yshift=-2.25cm, at=(l)] (u1) {\large $\overline{u}_1$};   
    \node[draw, shape=circle, minimum size=\circleSize, xshift=2cm,  yshift=-2.25cm, at=(l)] (u2) {\large $\overline{u}_2$};   
    %
    %\node[draw, shape=circle, minimum size=\circleSize, xshift=3.5cm,  yshift=-2.5cm, at=(l)] (b) {\Large $\beta$};
    
    % After nodes have appeared
    %\draw[->, line width=1pt] (o)--(a);
    \draw[->, line width=1pt] (l)--(u1);
    \draw[->, line width=1pt] (l)--(u2);
    %\draw[->, line width=1pt] (o)--(b);

    \draw[dash pattern=on 2pt off 2pt, line width=1pt, ->] (o)--(u1);
    \draw[dash pattern=on 2pt off 2pt, line width=1pt, ->] (o)--(u2);
    % ------- %

     % --- meta-dec ---%
    %\node [draw, anchor=west] (metadec) at (0,9) {\begin{tabular}{c} Meta-dec $(\metadec)$  \end{tabular} };
    %
    \node[draw, shape=circle, minimum size=\circleSize, xshift=0cm,  yshift=-6.5cm, at=(o)] (p) {\Large $\overline{\rho}$}; 

    \draw[dash pattern=on 2pt off 2pt, line width=1pt, ->] (u1)--(p);
    \draw[dash pattern=on 2pt off 2pt, line width=1pt, ->] (u2)--(p);
    % ------- %

    % Ancestors bracket
    \draw [decorate,decoration={brace,amplitude=7pt,mirror}, line width=0.8pt] (1.4, 10.75) -- (1.4, 4.75);
    \node[align=center] at (0.45, 7.8) {\large Anc.};

    % Grandparent label (no line or bracket)
    \node[align=left] at (6.5, 10) {\large Grandparent};

    % Parents bracket (mirrored)
    \draw [decorate,decoration={brace,amplitude=7pt}, line width=0.8pt] (6.6, 8.5) -- (6.6, 4.75);
    \node[align=left] at (7.6, 6.6) {\large Par.};

\end{tikzpicture}}
    \caption{Dropout and its ancestors.}
    \label{subfig:ancestors_dropout}
\end{subfigure}
\caption{\blr A variant of the MLP example with the dropout.}
\label{fig:ancestors_dropout}
\end{figure}

{\blr
In the MLP variant, the admissible values (the bounds) of the dropout $\rho$ are determined by the values of its parents, that is,  the number of layers $l$ and the number of units $u_i$ with $0 < i \leq l$.
The constant $\tau_{\text{max}} \coloneq \max \left\{ \sum_{i=1}^{l} u_i \ : \ o \in \{\texttt{ASGD}, \texttt{ADAM} \}, \ l \in L_o, \  u_i \in U_o \text{ for } 1 \leq i \leq l   \right\}$ implies that
%
%\begin{equation}
%    \tau_{\text{max}}:=\max \left\{ \sum_{i=1}^{l} u_i \ : \ o \in \{\texttt{ASGD}, \texttt{ADAM} \}, \ l \in L_o, \  u_i \in U_o \text{ for } 1 \leq i \leq l   \right\}.
%    \label{eq:dropout_tau}
%\end{equation}
%
the optimizer $o$ (grandparent) must also be considered to determine all the possible values $\rho$ can take.
%
%that the sets $L_o$ and $U_o$ are defined respectively in \eqref{eq:layers_set} and \eqref{eq:units_set}.
%
However, note that the bounds of the dropout $\rho$ do not have an explicit dependency with the values of the optimizer $o$ since, for a given MLP problem, the constant $\tau_{\text{max}}$ is fixed.
}
The ancestors can be defined recursively by starting at the parents, then passing by the parents of the parents, and so on, until the roots are reached.

\begin{mydef}[Ancestors]
    For $r \in R$ and $i \in I^r$, the ancestors of the variable $\overline{x}_i^r$, noted $\overline{\ancestors}_i^r$, is the subset of variables that are either parents or recursively ancestors of parents of $\overline{x}_i^r$, \textit{i.e.},
    \begin{equation}
        \overline{\ancestors}_i^r 
            := \overline{\parents}_i^r \ \cup \left(\bigcup_{\overline{x}_j^{r'}\in~\overline{\parents}_i^r} \overline{\ancestors}_j^{r'} 
            \right)
        \label{eq:ancestors}
    \end{equation}
where $\overline{\ancestors}_j^{r'}$ denotes the ancestors of the parent variable $\overline{x}_j^{r'}$.
The recursion in~\eqref{eq:ancestors} stops at the root nodes, \textit{i.e.}, with
$\overline{\parents}_j^{\meta}= \emptyset , \forall j \in I^{\meta}$.
    
    \label{def:ancestors} 
\end{mydef}

Now that several notions of graph theory have been adapted to this work, the definition of a hierarchical domain is formally established.
%--%
\begin{mydef}[Hierarchical domain]
A hierarchical domain is a domain with at least one variable with the decree property, \textit{i.e.}, at least one meta variable. 

\label{def:graph_structured_domain} 
\end{mydef}
%--%

Definition~\ref{def:graph_structured_domain} is one specific approach to formalize a hierarchical domain, yet there exist several equivalent statements.

%--%
\begin{theorem}[Hierarchical domain equivalences]
Let $\mathcal{X} \neq \emptyset$, with $G=(V,A)$ as its corresponding role graph. Then the following statements are equivalent:

\begin{enumerate}[leftmargin=*,labelindent=16pt]
    \itemsep0em 

    \item $\mathcal{X}$ is a hierarchical domain.

    \item The set of decree dependencies is non-empty, \textit{i.e.} $A \neq \emptyset$. 
     
    \item There is a least one point containing a variable with a least one parent.
      
\end{enumerate}
\label{thm:graph_structured_domain}
\end{theorem}
%
%
%\begin{proof}
%\begin{align*}
%    \mathcal{X} \text{ is a graph-structured domain} \textit{ (1.)} \Leftrightarrow \ &A \neq \emptyset, \text{ since there is a least one meta variable } x_i^{\meta}, \\ &\text{s.t. } (x_i^{\meta}, x_j^{r}) \in A, r \in \{\metadec, \decreed \} \text{ and } i,j \in \mathbb{N} \textit{ (2.)}\\
%    \Leftrightarrow \ &\exists x\in \mathcal{X}, \text{ s.t. } x_j^{r} \text{ is included with } (x_i^{\meta}, x_j^{r}) \in A \\ 
%    \Leftrightarrow \ &\exists x\in \mathcal{X}, \text{ s.t. }  \ancestors(x_j^{r}) \neq \emptyset, \text{ since } x_i^{\meta} \in \ancestors(x_j^{r}) \textit{ (3.)}  \\
%    \Leftrightarrow \ &\exists x\in \mathcal{X}, \text{ s.t. } \delta(x_j^{r})>0 \textit{ (4.)} 
%\end{align*}
 %
%\end{proof}
%
\begin{proof}
The theorem 
is a direct consequence of Definitions~\ref{def:role_graph}~and~\ref{def:ancestors}.
\end{proof}

%
%A consequence of being a graph-structured domain is that there is at least one pair of points in the domain $\mathcal{X}$ that do not share the same (included) variables or the same variable domains.
%
Theorem~\ref{thm:graph_structured_domain} emphasizes that the different bounds or inclusion-exclusion of variables within a hierarchical domain $\mathcal{X}$ is a consequence of interrelationships between variables, that is, its decree dependencies (Definition~\ref{def:decree_property_meta_variables}).
Note that a variable with missing entries can be modeled as a decreed variable whose inclusion is determined by an additional binary meta variable.

%
%The converse is not necessarily true.
%
%For example, the domain $\mathcal{X} = \{ x_1, (x_1,x_2) \}$, contains two points $x_1$ and $(x_1,x_2)$ that do not share the variable $x_2$, yet both variable $x_1$ and $x_2$ are not meta variables:~these domains are not treated in this work, as they lack structure to be formalized.
%as they lack any structure to be formalized and explicitly formulated with a general framework.
%

%--------------------------------------------------%
\subsection{Universal sets and restricted sets}
\label{sec:admissible_set}
%--------------------------------------------------%

%The next step is to develop the restricted sets that formalizes the notion of admissible values for the variables of an extended point.
%
As mentioned in Section~\ref{sec:role_graph}, meta-decreed and decreed variables are subjected to the values of their parents, since they must respect their decree dependencies for the given values of their parents.
%
%To model the inclusion or admissible values dependencies of a variable w.r.t. its parents, an universal set that contains all the possible values is conditioned with the values of its parents to generate the restricted set.   
%
In this section, the dependencies of a variable with respect to its parents are modeled through its restricted set, that is obtained by conditioning its universal set with the values of its parents. 
The universal set is defined next.

\begin{mydef}[Universal set]
For $r \in R$ and $i \in I^r$, the universal set $\overline{\mathcal{X}}_i^r$ of the variable $\overline{x}_i^r$ is the set that contains all possible values that the variable can take by considering  all possible values assigned to its ancestors $\overline{\ancestors}_i^r$.

\label{def:universal_set} 
\end{mydef}

{\blr In the MLP example, the universal set of the number of units $\overline{u}_1$} is $\overline{U} = U_{\texttt{ASGD}} \cup U_{\texttt{ADAM}} \cup \{\texttt{EXC} \}$.
The number of units $\overline{u}_1$ can be excluded, depending on the number of layers $\overline{l}$, hence its universal set must contain the special value $\texttt{EXC}$.

As discussed in the previous section, %in the most general case, 
the values of the ancestors must be considered, in addition to those of the parents, to determine a universal set.
%
%In the MLP example, recall from Equation~\eqref{eq:dropout_bounds} that the bounds of the dropout $\rho$ depends on its parents the number of hidden layers $l$ and the number of units $u_i$.
%
In the MLP example, the universal set of the dropout $\rho$ reduces to $\overline{P} = [0, 0.5]$, \textit{i.e.}, it is obtained when $\sum_{i=1}^l u_i = \tau_{\text{max}}$.
The bounds of the dropout $\left[0,  \frac{1}{2\tau_{\text{max}}} \sum_{i=1}^{l} u_i  \right]$ depends on its parents the number of layers $l$ and {\blr the number of units $u_i$, where $0 < i \leq l$}.
However, to determine its universal set $\overline{P}$, the constant $\tau_{\text{max}}= \max \left\{ \sum_{i=1}^{l} u_i \ : \ o \in \{\texttt{ASGD}, \texttt{ADAM} \}, \ l \in L_o, \  u_i \in U_o \text{ for } 1 \leq i \leq l   \right\}$ must be determined by considering the optimizer $o$ (grandparent).
Again, $\tau$ is a constant, hence the dropout $\rho$ has no decree dependency with the optimizer $o$.

%\begin{equation}
%    \overline{P} 
%   = [0, 0.5]
%    \label{eq:dropout_universal_set}
%\end{equation}
%
%where $\tau_{\text{max}}$ is the maximum sum of units that can be obtained with the hyperparameters of the MLP (defined after Equation~\eqref{eq:dropout_bounds} in Section~\ref{sec:working_example}), and $\tau_{\text{min}}$ is similarly the minimum sum of units, \textit{i.e.}, $\tau_{\text{min}}:=\min \left\{ \sum_{i=1}^{l} u_i \ : \ o \in \{\texttt{ASGD}, \texttt{ADAM} \}, \ l \in L_o, \  u_i \in U_o \text{ for } 1 \leq i \leq l   \right\}$.
%
%The parameters $\tau_{\text{min}}$ and $\tau_{\text{max}}$ are constants for the bounds of the dropout, and they are determined by considering the optimizer $o$.

%{\bl
%Consider again the simple example from the paragraph preceding Definition~\ref{def:ancestors}.
%with two decree relations,
% from the previous section, 
%The universal set of the decreed variable $\overline{x}_i^{\decreed}$ is determined by the smallest value that its grandparent $\overline{x}_k^{\meta} \in [\ell, u]$ can take: 
%    the universal set of $\overline{x}_i^{\decreed}$ is $[\ell, \infty[$.
%
%Note that $\overline{x}_i^{\decreed}=\ell$ is obtained when the two lower bounds of the decree dependencies are reached, \textit{i.e.}, 
%$\overline{x}_i^{\decreed} = \overline{x}_j^{\metadec} = \overline{x}_k^{\meta} = \ell$.
%
%This outlines that both its parent and grandparent, \textit{i.e.} its ancestors, are necessary to determine its universal set. 
%}
% ---------------------------------- %

The next step is to develop the restricted set of a variable. 
The restricted set of the variable $\overline{x}_i^r$
is the subset of the universal set $\overline{\mathcal{X}}_i^r $
such that $\overline{x}_i^r$ respects the decree dependencies for the given values of its parents $\overline{\parents}_i^r$, which are identified by arcs of the role graph $G$.
%identified by all arcs $(\overline{v}, \overline{x}_i^r) \in A$.
%
The formal definition of the restricted set is presented below.

\begin{mydef}[Restricted set]
For $r \in R$ and $i \in I^r$, {\blr and a given role graph $G=(V,A)$}, the restricted set of the variable $\overline{x}_i^r$ is the universal set $\overline{\mathcal{X}}_i^r$ conditioned by the values of its parents $\overline{\parents}_i^r$, expressed as 
% New espace quotient : http://www.normalesup.org/~sage/Enseignement/Cours/Quotient.pdf
\begin{equation*}
\restrictedset  := 
\left\{\overline{x}_i^r \in \overline{\mathcal{X}}_i^r 
    \ : \ \forall \ (\overline{v}, \overline{x}_i^r) \in A,\ 
     \overline{x}_i^r 
     \mbox{ respects the decree dependencies for the values of } \overline v\right\}.
\end{equation*}

% Old
%\begin{equation*}
%\overline{\mathcal{X}}_i^r  \mid \overline{\parents}_i^r  := 
%\left\{\overline{x}_i^r \in \overline{\mathcal{X}}_i^r 
%    \ : \ \forall \ (\overline{v}, \overline{x}_i^r) \in A,\ 
%     \overline{x}_i^r 
%     \mbox{ respects the decree dependencies for the values of } \overline v\right\}.
%\end{equation*}

\label{def:admissible_set} 
\end{mydef}

The notation $/ \phantom{}_{\overline{\parents}_i^r}$ for a restricted set $\restrictedset$ means ``\textit{given}'' the values of its parents.
The parents are placed as subscripts to outline the dependency.
% Restricted for meta and neutral variables
Meta and neutral variables are always included and have no parent, hence their restricted set is simply their universal set: $\overline{\mathcal{X}}_i^{\meta} / \phantom{}_{\overline{\parents}_i^{\meta}} = \overline{\mathcal{X}}_i^{\meta} \ \forall i \in I^{\meta}$ and $\overline{\mathcal{X}}_j^{\neutral} / \phantom{}_{\overline{\parents}_j^{\neutral}} = \overline{\mathcal{X}}_j^{\neutral} \ \forall j \in I^{\neutral}$.
% Restricted for meta-dec and decreed variables
The restricted set of a meta-decreed or decreed variable requires the values of its parents for
%both determining its inclusion, as well as
determining both its inclusion and its admissible values when it is included.
If the decree dependencies, given the values of parent variables, dictate that a child is excluded, then its restricted set is $\{ \texttt{EXC} \}$.
%
%In the MLP example, the restrictive set of $\overline{\alpha}_1$ is %the singleton 
%$\{ \texttt{EXC} \}$ when the optimizer $\overline{o} = \texttt{ADAM}$.
%
{\blr
In the MLP example used in this work, the restricted set of the number of units $\overline{u}_i$ is
\begin{equation*}
    \overline{U}_i / \phantom{}_{\overline{o}, \overline{l}}  = 
    \begin{cases}
        \begin{alignedat}{4}        
         &U_{\texttt{ASGD}} \  && \text{ if } \overline{o}=\texttt{ASGD} && \text{ and } && 1 \leq i \leq \overline{l},  \\
         &U_{\texttt{ADAM}} \  && \text{ if } \overline{o}=\texttt{ADAM}  && \text{ and } && 1 \leq i \leq \overline{l},  \\
         &\{ \texttt{EXC} \} \  && \text{ otherwise, } && 
        \end{alignedat}
    \end{cases}
\end{equation*}
where $i \in \{1,2\}$ and $\overline{l} \in \{0,1,2\}$.

%\begin{equation}
%    \overline{P} / \phantom{}_{\overline{l}, \overline{u}_{1:4}}  = 
%    \left[ 0, \  \dfrac{\sum_{i=1}^{\overline{l}} \overline{u}_i}{2\tau_{\text{max}}}
%    \right] \subseteq [0, 0.5].
%\label{eq:working_example_dropout_domain}
%\end{equation}
}

%--------------------------------------------------%
\subsection{Extended domain and transfer mapping}
\label{sec:extended_domain}
%--------------------------------------------------%

From the restricted sets, the extended domain $\overline{\mathcal{X}}$ is formally introduced.

\begin{mydef}[Extended domain]
The extended domain $\overline{\mathcal{X}}$ is a hierarchical domain constructed from a domain $\mathcal{X}$, its corresponding role graph $G=(V,A)$ and from the restricted sets of all included and excluded variables.
%
%An extended point $\overline{x} \in \overline{\mathcal{X}}$ belongs to an extended domain $\overline{\mathcal{X}}$ expressed as
The extended domain $\overline{\mathcal{X}}$ is expressed as
% Old version
%\begin{equation*}
%\overline{\mathcal{X}} := 
%\left\{ \ \overline{x} \ : \ \overline{x}_i^r \in \overline{\mathcal{X}}_i^r \mid \overline{\parents}_i^r, \ \forall r \in R \ \forall i \in I^r \ \right\}.
%\label{eq:domain_roles}
%\end{equation*}
% New version
\begin{equation*}
\overline{\mathcal{X}} := 
\left\{ \ \overline{x} \ : \ \overline{x}_i^r \in \overline{\mathcal{X}}_i^r / \phantom{}_{\overline{\parents}_i^r}, \ \forall r \in R, \ \forall i \in I^r \ \right\}.
\label{eq:domain_roles}
\end{equation*}

\label{def:extended_domain}
\end{mydef}

% Comment 
Definition~\ref{def:extended_domain} expresses that an extended point $\overline{x} \in \overline{\mathcal{X}}$ must respect all the inclusion-exclusion or admissible values dependencies, \textit{i.e.}, decree dependencies, between its variables. 
Recall that the main objective is to equip the domain $\mathcal{X}$ with a distance $\dist_p:\mathcal{X} \times \mathcal{X} \to \overline{\mathbb{R}}^+$ to facilitate optimization or machine learning tasks on domains that involve heterogeneous dataset.
% Introducing proposition
%Fortunately, an induced distance can be easily defined on the domain $\mathcal{X}$
%by composing the arguments of the graph-structured distance $\edist_p:\overline{\mathcal{X}} \times \overline{\mathcal{X}} \to \overline{\mathbb{R}}^+$ 
%with the mapping defined in the following theorem.
This will be done by introducing the meta distance $\edist_p:\overline{\mathcal{X}} \times \overline{\mathcal{X}} \to \overline{\mathbb{R}}^+$ on the extended domain $\overline{\mathcal{X}}$, and then by inducing a distance on the domain via the bijective mapping defined in the following theorem.

% Extension function now has become the transfer mapping
%\begin{mydef}[Transfer mapping]
%    The transfer mapping $T_G: \mathcal{X} \to \overline{\mathcal{X}}$ is a parametrized mapping w.r.t. the role graph $G$ that constructs an extended point $T_G(x) = \overline{x}  \in \overline{\mathcal{X}}$ from a point $x \in \mathcal{X}$ by adding its excluded variables, which are determined by the role graph $G$.
    %
    %\begin{align}
    %\begin{split}
    %    T_G \ : \ \mathcal{X} &\to \bar{\mathcal{X}}, \\
    %    x &\mapsto  \ \bar{x}.    
    %\end{split}
    %\end{align}
    
%\end{mydef}

%-- Function (wrapper) is a bijection ---%
\begin{theorem}[One-to-one correspondence]
The transfer mapping $T_G: \mathcal{X} \to \overline{\mathcal{X}}$, which assigns the extended point $T_G(x)  \in \overline{\mathcal{X}}$ to any point $x \in \mathcal{X}$ by adding its excluded variables determined by the role graph $G$, is a bijection.

%There is a bijection between the domain $\mathcal{X}$ and the extended domain $\overline{\mathcal{X}}$. 

\label{thm:one-to-one}
\end{theorem}
% ----------------- %

\begin{proof}
%Let $T_G: \mathcal{X} \to \overline{\mathcal{X}}$ be the {\em  transfer mapping} that assigns the extended point $T_G(x)  \in \overline{\mathcal{X}}$ to any point $x \in \mathcal{X}$ by adding its excluded variables determined by the role graph $G$.

% Injective 
Injectivity. Let $x,y \in \mathcal{X}$ be such that $x \neq y$.
By definition of $T_G$, there is a least one variable whose value differ between the two extended points $\overline{x}=T_G(x)$ and $\overline{y}=T_G(y)$, since 1) there is a least one included variable between $x$ and $y$ that does not share the same value; or 2) there is a least one variable that is strictly excluded for one point between $x$ and $y$.
Therefore, there is a least one variable that also does not share the same value in the extended points $\overline{x}$ and $\overline{y}$. 
This show that $T_G$ is injective since $x \neq y \Rightarrow T_G(x) \neq T_G(y)$.

% Surjective
Surjectivity. Let $\overline{x} \in \overline{\mathcal{X}}$ be an extended point.
Then, let $G'=(V', E')$ be the subgraph of the role graph $G$ obtained by removing the nodes and arcs corresponding to the excluded variables that take the value $\texttt{EXC}$ in  $\overline{x}$.
%
%{\rd The set of variables $V' \neq \emptyset$, because otherwise all variables in $\overline{x}$ would be meta-decreed or decreed with special value \texttt{EXC}, which is impossible under Assumption~\ref{hyp:DAG} that requires at least one root (meta variable).}
The set $V'$ is nonempty since either $\overline{x}$ has no meta variables and thus $V'=V$, or $\overline{x}$ has at least one meta variable $\overline{x}_i^{\meta} = x_i^{\meta} \in V'$.
Thus, $V'$ is the set of variables that are included in  $\overline{x}$, and $E'$ is the set of decree dependencies between the included variables of $\overline{x}$.
By construction, each variable in  $V'$ respects the decree dependencies between the other variables in the set $V'$.
Let $x$ be the point that contains only the included variables of the extended point $\overline{x}$, \textit{i.e.}, the variables in the set $V'$.
The point $x$ necessarily belongs to domain $\mathcal{X}$, since it only contains included variables that respect the decree dependencies between each others.
% otherwise
Indeed, otherwise, if the decree dependencies would not be respected, then $x$ would 1) contain a variable that should not be included; or 2) not contain a variable that should be included; or 3) contain an included variable that would take a value that is not allowed by the decree dependencies.    
Then, by definition of the mapping $T_G$, $T_G(x)=\overline{x}$, since $x \in \mathcal{X}$ contains all the included variables of the extended point $\overline{x} \in \overline{\mathcal{X}}$.
This shows that $T_G$ is thus surjective, since $\forall \overline{x} \in \overline{\mathcal{X}}, \ \exists x \in \mathcal{X}$, such that $T_G(x)=\overline{x}$. 

The transfer mapping $T_G$ is both injective and surjective, thus bijective.
\end{proof}

A consequence of Theorem~\ref{thm:one-to-one} is that if a distance $\edist_p:\overline{\mathcal{X}} \times \overline{\mathcal{X}} \to \overline{\mathbb{R}}^+$ is well-defined on the extended domain $\overline{\mathcal{X}}$, then a distance \hbox{$\dist_p:\mathcal{X}\times\mathcal{X}\to\overline{\mathbb{R}}^+$} can be induced on the domain $\mathcal{X}$ with the bijective transfer mapping $T_G:\mathcal{X} \to \overline{\mathcal{X}}$.
This is done in Section~\ref{sec:distance}.
%Theorem~\ref{thm:one-to-one} is used in Section~\ref{sec:distance}
%for inducing a distance on the domain $\mathcal{X}$ from the graph-structured distance defined on the extended domain $\overline{\mathcal{X}}$.

%--------------------------------------------------%
\section{Distance for hierarchical domains}
\label{sec:distance}
%--------------------------------------------------%

In this section, the meta distance $\edist_p:\overline{\mathcal{X}} \times \overline{\mathcal{X}} \to \overline{\mathbb{R}}^+$ is defined.
%
%Section~\ref{sec:meta_and_neutral_distance} details the distance for the meta and neutral components.  
%
Section~\ref{sec:included_excluded_distance} presents the included-excluded distance function that can compute distances for variables that can be included or excluded.
%
%The included-excluded distances of the MLP example are modeled in Section~\ref{sec:working_example_distance}.
%The included-excluded distances are the building blocks for constructing the distance for meta-decreed and decreed components, also presented in Section~\ref{sec:included_excluded_distance}. 
%The included-excluded distances are the building blocks for constructing the graph-structured distance $\edist:\overline{\mathcal{X}} \times \overline{\mathcal{X}} \to \overline{\mathbb{R}}^+$ in Section~\ref{sec:graph_structured_distance}.
%
Then, Section~\ref{sec:graph_structured_distance} presents the meta distance that is constructed with included-excluded distances, one per variable.

%--------------------------------------------------%
\subsection{Included-excluded distances}
\label{sec:included_excluded_distance}
%--------------------------------------------------%

%In this section, one-dimensional distance functions that compute distances between variables, which can be either included or excluded, are defined.
%
%Let $r \in \{\meta, \metadec, \decreed, \neutral \}$ and $i \in \{1,2,\ldots, n^r \}$. 
%
For two extended points $\overline{x}, \overline{y} \in \overline{\mathcal{X}}$, the distance regarding the $i$-th variable assigned to role $r$, respectively $\overline{x}_i^r$ and $\overline{y}_i^r$, is computed through three cases:
\begin{enumerate}
    \itemsep0em

    \item both $\overline{x}_i^r$ and $\overline{y}_i^r$ are excluded, hence the distance is set to zero;

    \item exactly one variable $\overline{x}_i^r$ or $\overline{y}_i^r$ is excluded, hence the distance is set to a parameter that models a distance between a variable that is included for one extended point, and excluded for the other extended point;

    %\item both $x_i^r$ and $y_i^r$ are included, hence an one-dimensional, well-established and variable-type appropriate distance function $d_i^r:\mathcal{V}_i^r \left( \ancestors \left(x_i^r \right) \right) \times \mathcal{V}_i^r \left( \ancestors \left(y_i^r \right) \right) \to \overline{\mathbb{R}}^+$ is used, \textit{e.g.}, if the variable is continuous then an one-dimensional Euclidean distance could be used.
   
    %\item  \makebox[\linewidth][s]{both $x_i^r$ and $y_i^r$ are included, hence an one-dimensional distance function}
    %
    %$d:\mathcal{V}_i^r \left( \ancestors \left(x_i^r \right) \right) \times \mathcal{V}_i^r \left( \ancestors \left(y_i^r \right) \right) \to \overline{\mathbb{R}}^+$ is used, \textit{e.g.}, the Euclidean distance.
    \item  both $\overline{x}_i^r$ and $\overline{y}_i^r$ are included, hence an one-dimensional distance function $d$ is used, \textit{e.g.}, the Euclidean distance~\cite{SoLePeZaKe2023}.
    %
    %$d:\mathcal{V}_i^r \left( \ancestors \left(x_i^r \right) \right) \times \mathcal{V}_i^r \left( \ancestors \left(y_i^r \right) \right) \to \overline{\mathbb{R}}^+$ is used, \textit{e.g.}, the Euclidean distance.
  
\end{enumerate}

% Old version
%Recall that a meta or neutral variable $\overline{x}_i^r \in \overline{\mathcal{X}}_i^r$ is always included, and it's restricted set is always the universal set $\overline{\mathcal{X}}_i^r$.
%
%Hence, for $r \in \{\meta, \neutral\}$, the distance between $\overline{x}_i^r$ and $\overline{y}_i^r$ is always computed in the third case with a distance function $d: \overline{\mathcal{X}}_i^r \times \overline{\mathcal{X}}_i^r \to \mathbb{R}^+$ that is compatible with the variable type.
% 
%For a meta-decreed or decreed variable, the restricted sets of
%$\overline{x}_i^r$ and $\overline{y}_i^r$ may differ. 
%
%Hence, for a meta-decreed or decreed variable, its distance must be defined on the universal set in order to allow comparisons with any pairs $\overline{x}_i^r, \overline{y}_i^r$ with different restricted sets. 
%
%Recall that the universal set of the number of units $\overline{u}_i$ is $\overline{U} = U_{\texttt{ASGD}} \cup U_{\texttt{ADAM}} \cup \{\texttt{EXC} \}$.
%
%Hence, to compare the number of units $\overline{u}_i$ from any two pairs of extended points, the distance for $\overline{u}_i$ must be defined on its universal set $\overline{U}$. 
%
%The following theorem formalizes the discussion above on the three cases and the universal set by introducing a novel distance based on distances proposed in~\cite{Saves2024, BaBuDiHwMaMoLaLeSa2023},

% Old version
Recall that a meta or neutral variable $\overline{x}_i^r \in \overline{\mathcal{X}}_i^r$ is always included, hence for $r \in \{\meta, \neutral\}$, the distance between $\overline{x}_i^r$ and $\overline{y}_i^r$ is always computed in the third case.
For a meta-decreed or decreed variable, the restricted sets of
$\overline{x}_i^r$ and $\overline{y}_i^r$ may differ, hence its corresponding included-excluded distance must be defined on its universal set $\overline{\mathcal{X}}_i^r$ in order to allow comparisons of any pairs $\overline{x}_i^r, \overline{y}_i^r$ with different restricted sets.
In the MLP example, the universal set of the number of units $\overline{u}_1$ is $\overline{U} = U_{\texttt{ASGD}} \cup U_{\texttt{ADAM}} \cup \{\texttt{EXC} \}$.
Hence, to compare the number of units $\overline{u}_i$ from any two pairs of extended points, the corresponding included-excluded distance of $\overline{u}_i$ must be defined on its universal set $\overline{U}$. 
The following theorem formalizes the discussion above on the three cases and the universal set by introducing a novel distance based on distances proposed in~\cite{Saves2024, BaBuDiHwMaMoLaLeSa2023}.
%----------------%

\begin{theorem}[Included-excluded distance]
Let $\overline{\mathcal{X}}_i^r$ be the universal set of the $i$-th variable assigned to the role $r \in R$, noted $\overline{x}_i^r$, and define $\overline{\mathcal{Y}}^r_i= \overline{\mathcal{X}}^r_i \setminus \{ \normalfont{\texttt{EXC}} \}$. 
Consider $d:\overline{\mathcal{Y}}^r_i \times\overline{\mathcal{Y}}_i^r \to \overline{\mathbb{R}}^+$,
a one-dimensional extended real-valued distance 
for the variable $\overline{x}_i^r$ when it is included, and 
$\theta_i^{r} \in \mathbb{\overline{R}}^+$
a parameter greater than or equal to 
$\sup \{ d \left(\mu, \nu \right): \mu, \nu \in \overline{\mathcal{Y}}_i^r \}/2$.
Then, for $u,v \in \overline{\mathcal{X}}_i^r$, the function $d_i^r : \overline{\mathcal{X}}_i^r \times \overline{\mathcal{X}}_i^r \to \overline{\mathbb{R}}^+$ defined by
\begin{equation} 
{d_i^r} \left( u, v \right) :=
    \left\{\begin{alignedat}{3}
        &d \left(u, v \right) \  &&\text{ if } \ u \neq \textup{\texttt{EXC}}  \neq   v \ &&\text{(both included}),\\
        &0 \ &&\text{ if } \ u = \textup{\texttt{EXC}} = v  \ &&\text{(both excluded)}, \\
        &\theta_i^r \ &&\text{ otherwise } &&\text{(one excluded)},    \\
    \end{alignedat}\right.
    \label{eq:included_excluded_distance_cases}
\end{equation}
is a one-dimensional extended real-valued distance function.
\label{thm:included_excluded_distance}
\end{theorem}

\begin{proof}
Let $r \in R$ and $i \in I^r$.
The identity of indiscernibles, nonnegativity and symmetry of $d_i^r$ are trivially proven since $\theta_i^r$ is strictly positive and since $d$ is a distance function. 
The rest of the proof consists of proving that $d_i^r$ satisfies the triangle inequality.
Let $u,v,z \in \overline{\mathcal{X}}_i^r$:
\\

%- Case 1 -%
\noindent Case 1 (both variables are excluded): if $u = \texttt{EXC} = v$, then 
\begin{equation*}
    d_i^r(u, v)=0~\leq~d_i^r(u,z)~+~d_i^r(z,v), \text{ by nonnegativity of } d_i^r.
\end{equation*}    \medbreak
%---------%

%- Case 2 -%
\noindent Case 2 (only one variable is excluded, by symmetry $v$): if $u \neq \texttt{EXC} = v$ and
%
% two sub-cases 
\begin{itemize}[leftmargin=1cm]
    \item if $z \neq \texttt{EXC}$, then 
    \begin{equation*}
        d_i^r (u, v) = \theta_i^r~\leq~d(u,z)~+~\theta_i^r~=~d_i^r(u,z)~+~d_i^r (z,v)  , \text{ by nonnegativity of } d.
    \end{equation*}   
    
     \item if $z = \texttt{EXC}$, then
     \begin{equation*}
         d_i^r(u,v) = \theta_i^r~\leq~\theta_i^r~+~0~=~d_i^r(u,z)~+~d_i^r (z,v).
     \end{equation*}  

\end{itemize}
\medbreak
%---------%

%- Case 3 -%
\noindent Case 3 (both variables are included): if $u \neq \texttt{EXC}$ and $v \neq \texttt{EXC}$  and
% two sub-cases 
\begin{itemize}[leftmargin=1cm]

    \item if $z \neq \texttt{EXC}$, then
    \begin{equation*}
        d_i^r(u,v)~=~ d(u,v)~\leq~d(u,z)~+~d(z,v)~=~d_i^r(u,z)~+~d_i^r(z,v).  
    \end{equation*}

    \item if $z = \texttt{EXC}$, then 
    \begin{equation*}
    d_i^r(u,v)~=~d(u,v)~\leq~\sup \{ d \left(\mu, \nu \right): \mu, \nu \in \overline{\mathcal{Y}}_i^r \}~\leq~2 \theta_i^r = d_i^r(u,z)~+~d_i^r(z,v).
    \end{equation*}
\end{itemize}
%---------%
%\phantom{A} 
\end{proof}

{\blr The following figure schematizes the included-excluded distance on a variant of the MLP example with only two variables.
Each graph represents an extended point.
}

\begin{figure}[htb!]
\begin{subfigure}[t]{0.31\textwidth}
    \centering
    \fbox{
    \scalebox{0.65}{\begin{tikzpicture}

    \node [draw, shape=ellipse, inner sep=2pt, align=center] (o1) at (0,10) 
    {\large
    \begin{tabular}{c}
    $\overline{o}$ \\
    $=\texttt{ASGD}$
    \end{tabular}
     };    
    \node [draw, shape=ellipse, inner sep=2pt, align=center, xshift=0cm, yshift=-2.25cm, at=(o1)] (a1) 
    {\large
    \begin{tabular}{c}
    $\overline{\alpha}$ \\
    $=0.9$
    \end{tabular}
     };   
    \draw[->] (o1)--(a1);

    \begin{scope}[shift={(3.5,0)}]
        \node [draw, shape=ellipse, inner sep=2pt, align=center] (o2) at (0,10) 
        {\large
        \begin{tabular}{c}
        $\overline{o}'$ \\
        $=\texttt{ASGD}$
        \end{tabular}
         };    
        \node [draw, shape=ellipse, inner sep=2pt, align=center, xshift=0cm, yshift=-2.25cm, at=(o2)] (a2) 
        {\large
        \begin{tabular}{c}
        $\overline{\alpha}'$ \\
        $=0.1$
        \end{tabular}
         };   
        \draw[->] (o2)--(a2);
    \end{scope}

    \node[align=center] at ($(a1)!0.5!(a2) - (0,1.75)$) {\Large $\Rightarrow d(\overline{\alpha}, \overline{\alpha}') = 0.8$};

\end{tikzpicture}}
    }
    \caption{Both included.}
    \label{subfig:inc_exc_distance_alpha1}
\end{subfigure}
\hspace{0.1cm}
\begin{subfigure}[t]{0.31\textwidth}
    \centering
    \fbox{
    \scalebox{0.65}{\begin{tikzpicture}

    \node [draw, shape=ellipse, inner sep=2pt, align=center] (o1) at (0,10) 
    {\large
    \begin{tabular}{c}
    $\overline{o}$ \\
    $=\texttt{ADAM}$
    \end{tabular}
     };    
    \node [draw, shape=ellipse, inner sep=2pt, align=center, xshift=0cm, yshift=-2.25cm, at=(o1)] (a1) 
    {\large
    \begin{tabular}{c}
    $\overline{\alpha}$ \\
    $=\texttt{EXC}$
    \end{tabular}
     };   
    \draw[->] (o1)--(a1);

    \begin{scope}[shift={(3.5,0)}]
        \node [draw, shape=ellipse, inner sep=2pt, align=center] (o2) at (0,10) 
        {\large
        \begin{tabular}{c}
        $\overline{o}'$ \\
        $=\texttt{ADAM}$
        \end{tabular}
         };    
        \node [draw, shape=ellipse, inner sep=2pt, align=center, xshift=0cm, yshift=-2.25cm, at=(o2)] (a2) 
        {\large
        \begin{tabular}{c}
        $\overline{\alpha}'$ \\
        $=\texttt{EXC}$
        \end{tabular}
         };   
        \draw[->] (o2)--(a2);
    \end{scope}

    \node[align=center] at ($(a1)!0.5!(a2) - (0,1.75)$) {\Large $\Rightarrow d(\overline{\alpha}, \overline{\alpha}') = 0$};

\end{tikzpicture}}
    }
    \caption{Both excluded.}
    \label{subfig:inc_exc_distance_alpha3}
\end{subfigure}
\hspace{0.1cm}
\begin{subfigure}[t]{0.31\textwidth}
    \centering
    \fbox{
    \scalebox{0.65}{\begin{tikzpicture}

    \node [draw, shape=ellipse, inner sep=2pt, align=center] (o1) at (0,10) 
    {\large
    \begin{tabular}{c}
    $\overline{o}$ \\
    $=\texttt{ADAM}$
    \end{tabular}
     };    
    \node [draw, shape=ellipse, inner sep=2pt, align=center, xshift=0cm, yshift=-2.25cm, at=(o1)] (a1) 
    {\large
    \begin{tabular}{c}
    $\overline{\alpha}$ \\
    $=\texttt{EXC}$
    \end{tabular}
     };   
    \draw[->] (o1)--(a1);

    \begin{scope}[shift={(3.5,0)}]
        \node [draw, shape=ellipse, inner sep=2pt, align=center] (o2) at (0,10) 
        {\large
        \begin{tabular}{c}
        $\overline{o}'$ \\
        $=\texttt{ASGD}$
        \end{tabular}
         };    
        \node [draw, shape=ellipse, inner sep=2pt, align=center, xshift=0cm, yshift=-2.25cm, at=(o2)] (a2) 
        {\large
        \begin{tabular}{c}
        $\overline{\alpha}'$ \\
        $=0.1$
        \end{tabular}
         };   
        \draw[->] (o2)--(a2);
    \end{scope}

    \node[align=center] at ($(a1)!0.5!(a2) - (0,1.75)$) {\Large $\Rightarrow d(\alpha, \alpha') = \theta_{\alpha} \geq 0.5$};

\end{tikzpicture}}
    }
    \caption{One excluded.}
    \label{subfig:inc_exc_distance_alpha2}
\end{subfigure}
\caption{\blr Included-excluded distance for $\alpha$, whose inclusion is controlled by the optimizer.}
\label{fig:inc_exc_distance}
\end{figure}

{\blr Figure~\ref{fig:inc_exc_distance} presents the three cases of the included-excluded distance for the update $\alpha$, whose inclusion is controlled by the optimizer.
In Figure~\ref{subfig:inc_exc_distance_alpha1}, the update is included in both extended points.
%then an Euclidean distance is used.
%
In Figure~\ref{subfig:inc_exc_distance_alpha3}, it is excluded in both extended points, hence the distance is zero.
Finally, in Figure~\ref{subfig:inc_exc_distance_alpha2}, the update is excluded in only one extended point.
In this case, the distance is set to a parameter $\theta_{\alpha}$ bounded below half of the largest possible distance achievable for two included updates, \textit{i.e.}, $\sup \{ d(\alpha, \alpha'): \alpha,\alpha' \in~]0,1[ \}/2=1/2$.
}

{\blr 
%Multiple comments on Theorem~\ref{thm:included_excluded_distance} that introduces the included-excluded distance are provided.
%
Next, multiple comments on the included-excluded distance are provided.}
% Comment on the parametrization
%First, the included-excluded distance $\overline{d}: \overline{\mathcal{X}}_i^r \left( \ancestors \left(\overline{x}_i^r \right) \right) \times \overline{\mathcal{X}}_i^r \left( \ancestors \left(\overline{y}_i^r \right) \right) \to \overline{\mathbb{R}}^+$ is said to be parametrized with the ancestors of $\overline{x}_i^r$ and $\overline{y}_i^r$, since these ancestors manage the computation of the distance with the three cases 1) both excluded, 2) one excluded or 3) both included. 
%
% with $r \in \{\metadec, \decreed \}$ and $i \in \{1,2\ldots, n^r\}$,
%
%
%For the sake of the presentation, this parametrization is not textually mentioned, but it is explicit in the domain of the included-excluded distance.
%\overline{d}: \overline{\mathcal{X}}_i^r \left( \ancestors \left(x_i^r \right) \right) \times \overline{\mathcal{X}}_i^r \left( \ancestors \left(y_i^r \right) \right) \to \overline{\mathbb{R}}^+$.
% Meta and neutral
%Second, the included-excluded distance is defined for all roles in order to have a more concise formulation of graph-structured distance in next Section~\ref{sec:graph_structured_distance}.
% 
First, the included-excluded distance is precisely useful for meta-decreed and decreed variables that can have different restricted sets, \textit{i.e.} that can be included or excluded, and/or have different admissible values.
As mentioned previously, for variables that are always included, such as meta and neutral variables, included-excluded distances are always computed with the both included case.
The included-excluded distance is nevertheless defined for these variables in order to obtain a concise formulation of the meta distance in Section~\ref{sec:graph_structured_distance}.
%

%
%For $r \in \{\meta, \neutral \}$ and $i  \in I^r$, the included-excluded distance is more precisely defined on its fixed bounds $\mathcal{X}_i^r$, such that $d_i^r \left(\overline{x}_i^{r}, \overline{y}_i^{r} \right): \mathcal{X}_i^r \times \mathcal{X}_i^r \to \mathbb{R}^+$, with
%$d_i^r \left(\overline{x}_i^{r}, \overline{y}_i^{r} \right) = d \left(\overline{x}_i^{r}, \overline{y}_i^{r} \right)$.
%

% Discuss that the distance regarding the type is general 
Second, the included-excluded distance is compatible with any variable type. 
Indeed, in~\eqref{eq:included_excluded_distance_cases}, the two cases both excluded and one excluded does not regard the variable type, and the case both included allows to use any distance function $d$.

%A consequence of Theorem~\ref{thm:included_excluded_distance} is that if the domain $\overline{\mathcal{V}}_i^r \left( \ancestors \left( x_i^r \right) \right)$ is always bounded, then the included-excluded distance $\overline{d}$ is a common distance that maps into $\mathbb{R}$, \textit{i.e.}, $\infty$ is not allowed. 

% Parameter
%the parameter $\theta_i^r$ provides flexibility for the distance of the one excluded case.
%in which a variable is included for an extended point, but excluded for the other one.
Third, the universal set $\overline{\mathcal{X}}_i^r$ allows to compare any pair of variables $\overline{x}_i^r, \overline{y}_i^r$.
Moreover, it is also necessary to establish a lower bound on the parameter $\theta_i^r$, which ensures the triangular inequality in the last case of the proof.
%$\theta_i^{r} \geq \sup \{ d \left(\overline{z}_i^r, \overline{w}_i^r \right): \overline{z}_i^r, \overline{w}_i^r \in \overline{\mathcal{Z}}_i^r \setminus \{ \texttt{EXC} \} \}/2$
The inequality $\theta_i^{r} \geq \sup \{ d \left(\mu, \nu \right): \mu, \nu \in \overline{\mathcal{X}}_i^r \setminus \{\texttt{EXC} \} \}/2$ implies that the distance $d_i^r(\overline{x}_i^r, \overline{y}_i^r) = \theta_i^r$ (one excluded case) must be at least half the largest distance between any pairs of included variables $\overline{x}_i^r,\overline{y}_i^r$, with possibly different restricted sets.
%
%%Yet, a pair of included meta-decreed or decreed variables $\overline{x}_i^r,\overline{y}_i^r$ do not necessarily share the same values-assignment for its ancestors.
%
%%Hence, they do not necessarily share the same ancestors-bounds $\mathcal{V}_i^r$ in Definition~\ref{def:admissible_set}.
%
%%The lower-bound must is determined with the complete bounds $\overline{S}_i^r$ in order to ensure that triangular inequality is respected in the last case of the proof.
%
%In the MLP example, recall that the universal set of the number of units $\overline{u}_i$ is $\overline{U} = U_{\texttt{ASGD}} \cup U_{\texttt{ADAM}} \cup \{ \texttt{EXC} \}$.
%
%Therefore, the parameter $\theta_i^r$ for $\overline{u}_i$ must be greater than $(\max(\overline{U}')-\min(\overline{U}'))/2$, where $\overline{U}' = U_{\texttt{ASGD}} \cup U_{\texttt{ADAM}}$.
%
Apart from its lower bound, the parameter $\theta_i^r$ is flexible.
%
%For a continuous or integer variable, this equivalently means that the distance for the one excluded case must be at least half the range of $V_i^r$, such that $\theta_i^r \geq M_i^r/2 = \left(\m \left(V_i^r \right) - \min \left(V_i^r \right)\right)/2$.
%
%Moreover, the range of a continuous or integer variable range can be infinite, implying that both $M_i^r$ and $\theta_i^r$ can be infinite: this is why $M_i^r \in \mathbb{R} \cup \{\infty\}$ in Definition~\ref{def:included_excluded_distance}.
%
%Remarkably, this lower bound ensures that triangular inequality is respected, but it also conceptually illustrates that there must be an arbitrarily large minimum distance between extended points that do not share at least one included variable.

% Extended distance as it can take the value infinity 
Fourth, the included-excluded distance is more formally an extended real-valued one~\cite{Be2013}, since it is allowed to take the infinite value.
The infinity value allows to consider meta-decreed or decreed variables with unbounded restricted sets.
For example, let $r \in \{\metadec, \decreed\}$ and $i \in I^r$, such that $\restrictedset = [0, \infty[ $ when it is included, and $\restrictedset = \{\texttt{EXC} \}$ when it is excluded. 
In this example, $\overline{\mathcal{X}}_i^r = [0, \infty[~\cup~\{ \texttt{EXC} \}$, therefore
 $\sup \{ d \left(\mu, \nu \right): \mu, \nu \in \overline{\mathcal{X}}_i^r \setminus \{ \texttt{EXC} \} \}=\infty$, hence the parameter $\theta_i^r$ must be set to infinity to guaranty the triangular inequality which is a behaviour one would generally avoid. 
If the restricted sets are always bounded, then included-excluded distance becomes a standard distance that maps into $\mathbb{R}$, instead of $\overline{\mathbb{R}}$.

% Comment on overload
%Finally, the lower-case notation $d_i^r$ and $d$ corresponds to one-dimensional distances.
%
%The superscript $r$ and subscript $i$ are omitted to lighten the notation, and since the arguments (variables) are already indexed in a distance $\overline{d}(\overline{x}_i^r, \overline{y}_i^r)$ or $d(x_i^r, y_i^r)$. 
%
%Section~\ref{sec:graph_structured_distance} will introduce the functions $\dist_p$ and $\edist_p$ for distances between points of the domain $\mathcal{X}$ and extended points of the extended domain $\overline{\mathcal{X}}$.
%
%The one-dimensional distances $\overline{d}$ and $d$ are conceptually overloaded, as it is done in computer science: the letters $d$ and $\overline{d}$ are used for all variables, but which variable has its unique signature, e.g. for the $i$-th variable of the role $r$ its signature is uniquely $\overline{d}: \overline{\mathcal{X}}_i^r \left( \ancestors \left(x_i^r \right) \right) \times \overline{\mathcal{X}}_i^r \left( \ancestors \left(y_i^r \right) \right) \to \overline{\mathbb{R}}^{+}$.
%
%This modeling choice is done to lighten the notation.
%The next section illustrates the distance functions $d$ and $\overline{d}$ on the MLP example.

%--------------------------------------------------%
\subsection{Meta distance and induced distance}
\label{sec:graph_structured_distance}
%--------------------------------------------------%

%
Now that the included-excluded distance has been detailed, the following theorem formally introduces the meta distance. 
\begin{theorem}[Meta distance] 
For any $p \geq 1$, the function
$\edist_p:\overline{\mathcal{X}} \times \overline{\mathcal{X}} \to \overline{\mathbb{R}}^+$ defined by
%
%\begin{equation}
%   \edist_p( \overline{x}, \overline{y}) := 
%    \left( \sum\limits_{r \in R} \sum\limits_{i=1}^{n^{r}}  \hspace{2pt} \overline{d} \left( \overline{x}_i^r, \overline{y}_i^r \right) ^p \right)^{\frac{1}{p}},
%\label{eq:graph_structured_distance} 
%\end{equation}
%
\begin{equation}
   \edist_p( \overline{x}, \overline{y}) := 
    \left( \sum\limits_{r \in R} \sum\limits_{i \in I^r}  \hspace{2pt} {d}_i^r \left( \overline{x}_i^r, \overline{y}_i^r \right) ^p \right)^{\sfrac{1}{p}},
\label{eq:graph_structured_distance} 
\end{equation}
is an extended real-valued distance function, where $R= \{ \meta, \metadec, \decreed,\neutral \}$, {\blr $i \in I^r = \{1,2,\ldots,$ $n^r\}$ and $n^r \in \mathbb{N}$ is the number of variables assigned to the role $r$}.
%$\omega_i^r > 0$ is the weight of the $i$-th variable assigned to the role $r$, and
%
% that is constructed with additions of one-dimensional distances of appropriate variable,
%\begin{itemize}[leftmargin=*,labelindent=18pt]
    
    %\item \makebox[\linewidth][s]{for $r \in \{\meta, \neutral\}$ and $i \in \{1,2,\ldots, n^r \}$, the included-excluded distance}
    %
    %$\overline{d}:\mathcal{X}_i^r \times \mathcal{X}_i^r\to \mathbb{R}^+$ is computed as $\overline{d} \left( x_i^r, y_i^r \right) = d \left( x_i^r, y_i^r \right) $,
%    \item for $r \in \{\meta, \neutral\}$ and $i \in I^r$, the included-excluded distance
%    $d_i^r:\mathcal{X}_i^r \times \mathcal{X}_i^r\to \mathbb{R}^+$ is computed on the fixed set $\mathcal{X}_i^r$ as $d_i^r \left( \overline{x}_i^r, \overline{y}_i^r \right) = d \left( x_i^r, y_i^r \right) $ (both included),
    
    %
%    \item for $r \in \{\metadec, \decreed\}$ and $j \in I^{r}$, the included-excluded distance 
    %
%    $d_i^r:\overline{\mathcal{S}}_j^r  \times \overline{\mathcal{S}}_j^r \to \overline{\mathbb{R}}^+$ is computed on the universal set $\overline{\mathcal{S}}_j^r$  in three cases as in Equation~\eqref{eq:included_excluded_distance_cases}.
    
    %\item for $r \in \{\metadec, \decreed\}$ and $j \in \{1,2,\ldots, n^r \}$, the inclusion-exclusion distance 
    %$\overline{d}:\mathcal{V}_j^r \left( \ancestors \left( x_j^r \right) \right)   \times \mathcal{V}_j^r \left( \ancestors \left( y_j^r \right) \right) \to \overline{\mathbb{R}}^+$ is computed in three case, as presented in Equation~\eqref{eq:included_excluded_distance_cases}.
%\end{itemize}

\label{thm:graph_structured_distance}
\end{theorem}
%

% New proof based on the norm
\begin{proof}
The identity of indiscernibles, nonnegativity and symmetry of $\edist_p$ are trivially proven since the operations of summation and exponentiation with $p\geq1$ on the included-excluded distances in~\eqref{eq:graph_structured_distance} conserve these properties.
The rest of the proof consists of showing the triangular inequality is respected by demonstrating that $\edist_p$ is equivalent to a $p$-norm, that respects the triangular inequality.
Let $K= \{1,2,\ldots, n^{\meta}+n^{\metadec}+n^{\decreed}+n^{\neutral}\}$ be a set of indices that reorders the indices $r \in \{\meta, \metadec, \decreed, \neutral \}$ and $i \in \{1,2,\ldots, n^r \}$: 
\begin{itemize}[leftmargin=*,labelindent=18pt]
    \itemsep0em
    
    \item $a_k :=  d_i^{\meta} \left( \overline{x}_i^{\meta}, \overline{y}_i^{\meta} \right)$, 
    \hspace{13pt} for $k=i$ with $i \in \{1,2,\ldots, n^{\meta} \}$, 

    \item $a_k := d_j^{\metadec} \left( \overline{x}_j^{\metadec}, \overline{y}_j^{\metadec} \right)$, for $k=n^{\meta}+j$ with $j \in \{1,2,\ldots, n^{\metadec} \}$,  

    \item $a_k := d_l^{\decreed} \left( \overline{x}_l^{\decreed}, \overline{y}_l^{\decreed} \right)$, for $k=n^{\meta}+n^{\metadec}+l$ with $l \in \{1,2,\ldots, n^{\decreed} \}$, 

    \item $a_k :=  d_v^{\neutral} \left( \overline{x}_v^{\neutral}, \overline{y}_v^{\neutral} \right)$, \hspace{-0.05cm}for $k=n^{\meta}+n^{\metadec}+n^{\decreed}+v$ with $v \in \{1,2,\ldots, n^{\neutral} \}$. 
\end{itemize}
\noindent Finally, let $a=\left(a_1,a_2, \ldots, a_{|K|} \right)$, then
% New version with norm
    \begin{equation*}
    \| a \|_p =  \left(\sum \limits_{k=1}^{|K|} \left|a_k \right|^p \right)^{\sfrac{1}{p}} =  \left( \sum \limits_{r \in R}  \sum \limits_{i \in I^r} d_i^r \left( \overline{x}_i^r, \overline{y}_i^r \right)^p \right)^{\sfrac{1}{p}} = \edist_p \left( \overline{x}, \overline{y} \right)   
    \end{equation*}
%  
%\phantom{A}
\end{proof}

For $p \rightarrow \infty$, the meta distance $\edist_{\infty}:\overline{\mathcal{X}} \times \overline{\mathcal{X}} \to \overline{\mathbb{R}}^+$ 
%is defined as 
converge towards the supremum distance, \textit{i.e.},
%For the special case that $p \rightarrow \infty$, the graph-structured distance $\edist_{\infty}:\overline{\mathcal{X}} \times \overline{\mathcal{X}} \to \overline{\mathbb{R}}^+$ is defined as a maximum, such that
%
\begin{equation}
    \edist_{\infty} \left( \overline{x}, \overline{y} \right) := \max \left \{~ d_i^r \left(\overline{x}_i^r, \overline{y}_i^r\right)  \ : \ r \in R, \ i \in I^r \right\},
\end{equation}
which is trivially a distance function by virtue of the $\max$ function.

Note that, in practice, variables often require scaling to improve the conditioning and eliminate biases related to variable scales.
In the context of the work, scaling categorical variables and excluded variables is ambiguous.
Fortunately, in our proposed distance, an included-excluded distance $d_i^r$ is defined with a flexible distance $d$ for its both included case.
Therefore, the distance $d$ can be defined using a scaling parameter, such that $d(a,b)=\omega_i^r d'(a,b)$, where $\omega_i^r>0$ is a weight parameter related to the variable $\overline{x}_i^r$ and $d'$ is a one-dimensional distance of appropriate variable type. 
The weight parameter $\omega_i^r$ can be used to automatically scale the lower bounds of the parameter $\theta_i^r$, since the lower bound of $\theta_i^r$ is defined with the one-dimensional distance $d$.
%
%This approach with scaling parameters can help removing biases from variable scales. It is used in Section~\ref{sec:numerical_exp}.}
In Section~\ref{sec:numerical_exp}, scaling parameters will be used to better adjust our proposed distance  to  the datasets; this helps to remove biases that are related to variable scales.
%
%at the cost of adjusting or pre-determining these weights.
%
%, weights are used as such, and they are treated as parameters that are adjusted to a validation dataset.}

Theorem~\ref{thm:one-to-one}, which establishes the bijection between the domain $\mathcal{X}$ and the extended domain $\overline{\mathcal{X}}$, implies that a distance $\dist_p:\mathcal{X}\times\mathcal{X}\to \overline{\mathbb{R}}^+$ can be induced from the meta distance $\edist_p:\overline{\mathcal{X}} \times \overline{\mathcal{X}} \to \overline{\mathbb{R}}^+$. The following corollary is a direct consequence of Theorems~\ref{thm:one-to-one}~and~\ref{thm:graph_structured_distance}.
%
%The following corollary formalizes such induced distance $\dist_p:\mathcal{X}\times\mathcal{X}\to \overline{\mathbb{R}}^+$.

\begin{corollary}[Induced distance]
For $p \geq 1$, the induced function $\dist_p:\mathcal{X}\times\mathcal{X}\to\overline{\mathbb{R}}^+$defined by
\begin{equation}
    \dist_p(x,y) := \edist_p \left( T_G(x), T_G(y) \right) = \edist_p \left( \overline{x}, \overline{y} \right),
    \label{eq:induced_distance}
\end{equation}
is an extended real-valued distance, where $\edist_p:\overline{\mathcal{X}} \times \overline{\mathcal{X}} \to \overline{\mathbb{R}}^+$ is a meta distance and $T_G:\mathcal{X} \to \overline{\mathcal{X}}$ is the bijective transfer mapping.  

\label{corollary:induced_distance}
\end{corollary}
%
%\begin{proof}
%    The corollary is a direct consequence of Theorem~\ref{thm:one-to-one}~and~\ref{thm:graph_structured_distance}.
%\end{proof}

%
Corollary~\ref{corollary:induced_distance} unpacks most of the contributions.
To arrive at Corollary~\ref{corollary:induced_distance}, it was necessary to: 1) define an extended point $\overline{x}$, restricted sets and the extended domain $\overline{\mathcal{X}}$ using notions from graph theory; 2)~define the transfer mapping $T_G:\mathcal{X} \to \overline{\mathcal{X}}$, and prove that it is bijective, 3) define the included-excluded distances on the universal set for tackling variables that can either be included or excluded, or with different admissible values, and 4) define the meta distance $\edist_p:\overline{\mathcal{X}} \times \overline{\mathcal{X}} \to \overline{\mathbb{R}}^+$ based on the contributions mentioned above.

%\begin{figure}[htb!]
%
%\centering
%  \scalebox{1}{\input{new_figs/recap_of_the_work}}
  %\includegraphics[width=0.75\linewidth]{figs/analysis_extendedpoll.png}
%
%\caption{An update of Figure~\ref{subfig:intro_big_picture_approach}, recapping the work.}
%\label{fig:recap}
%\end{figure}
%

%---------------------------------------------------%
\section{Computational experiments on mixed-variable and hierarchical problems}
%\section{Computational experiments on the working example}
\label{sec:numerical_exp}
%---------------------------------------------------%

{\blr 

This section compares three approaches on mixed-variable and hierarchical regression and classification problems with simple distance-based models.
%for instances of the MLP example.
%
%The models employed are the inverse distance weighting (IDW) and the $K$-nearest neighbors.
%
The first approach, named {\sf Sub}, divides a problem into subproblems, each assigned to a portion of the domain in which the included variables are fixed. 
In the experiments, {\sf Sub} constructs one model per subproblem, and each model uses an Euclidean distance defined on their corresponding subdomain. 

The second approach, called {\sf Meta}, constructs a single model with the induced distance from~\eqref{eq:induced_distance}, where $p=2$.
This approach directly tackles a hierarchical problem and aggregates data across the subproblems. 
%
%{\sf Meta} constructs a single model with an induced distance.
%The regression problems are viewed as least-squares optimization problems on parameters $\lambda$, more precisely, for a given model $\hat{f}:\mathcal{X} \to \mathbb{R}$ wapproach
%such that
%
%\begin{equation}
%\end{equation}
%

The third approach, named {\sf Hybrid}, employs the variable-size framework~\cite{PeBrBaTaGu2021} in which one layer of hierarchy is tackled and only shared included variable are compared. 
{\sf Hybrid} constructs a single model for the first two variants, and two models for the other variants.

The experiments focus on comparing approaches and their implementation of a distance.
For this reason, simple distance-based models are employed: Inverse Distance Weighting (IDW) models for regression, and $K$-nearest neighbors models for classification. 
Earth mover's (or Wasserstein) distances are not considered~\cite{GaDuKaSe2018, SoLePeZaKe2023}, since their computations require solving an optimal transport problem.
Distances are evaluated repeatedly in the experiments, and they are unsuitable in this setting.

%The rest of this section is organized as follows.
%
%The hyperparameters domain and its many datasets used in the experiments are detailed in Section~\ref{sec:experiments_description_pb}.
%
%Few classification experiments are carried out in Section~\ref{sec:classification}.
%
%Data profiles on regression problems are done in Section~\ref{sec:data_profiles}.
%
%Finally, experiments testing the aggregation of data on some regression problems are done in Section~\ref{sec:data_aggregation}.
%

%--------------------------------------------------%
\subsection{Description of the hyperparameters domain and its datasets}
\label{sec:experiments_description_pb}
%--------------------------------------------------%

%This section describes a working example inspired from~\cite{G-2022-11} that models the domain of the hyperparameters of an MLP.
%
%The performance of a deep model in function of its hyperparameters can be viewed as a mixed-variable function with meta variables:
%
%let $f:\mathcal{X} \to \mathbb{R}$ be a function that outputs a performance score $f(x) \in \mathbb{R}$ for a given set of hyperparameters $x \in \mathcal{X}$, where $\mathcal{X}$ is the domain of the hyperparameters.
%
%In practice, the performance score $f(x)$ is typically a score of accuracy on a untested dataset, which is expensive-to-evaluate since the training, validation and performance test is done for a given and fixed set of hyperparameters~\cite{hypernomad_paper}.
%

%The computational experiments consist of regression and classification problems on datasets in mixed-variable and hierarchical domains.
%
In the experiments, data points are vector of hyperparameters.
In order to restrict the number of variables, few important hyperparameters are intentionally discarded, such as the momentum, the activation function and the batch size. 
The hyperparameters domain (HPD) of
%presented in Figure~\ref{fig:working_example_table}, 
below, is used to generate datasets in hierarchical and mixed-variable domains.

%
%The figure is composed of several tables, which are separated into different cases of optimizer variable values.
%as well as two cases on the sum of the units in the hidden layers. 
%

%
%From this domain, five variants of increasing difficulty are created and presented in Figure~\ref{fig:MLP_problems}.

\begin{figure}[htb!]
\centering
    \scalebox{0.8}{\begin{tikzpicture}
\small 

    \node [shape=rectangle, align=center](table1) at (0,0) {
        \begin{tabular}{lccc} 
        \toprule
            HP & Variable & Bounds & Role \\ \midrule
            Learning rate & $r$ & $]0,1[$ & Neutral \\ 
           % Activation function & $a$ & \{\texttt{ReLU}, \texttt{Sig}, \texttt{Tanh}\} & Neutral \\
            \hdashline
            Optimizer & $o$ & $\{ \texttt{ASGD}, \texttt{ADAM} \}$  & Meta \\
            \bottomrule
        \end{tabular}};

    % ---------Split with optimizers --------------%
    \node [shape=rectangle, align=center, xshift=-4.85cm, yshift=-3.5cm, at=(table1)](table2){
        \begin{tabular}{llccc} 
        \toprule
            \multicolumn{2}{l}{HP} & Variable & Bounds & Role \\ \midrule
            \multicolumn{2}{l}{Decay} & $\alpha_1$ & $]0,1[$ & Decreed \\
            \multicolumn{2}{l}{Power update} & $\alpha_2$ & $]0,1[$ & Decreed \\
            \multicolumn{2}{l}{Average start} & $\alpha_3$   & $]0,1[$ & Decreed \\ 
            \multicolumn{2}{l}{\# of hidden layers} &  $l$ & $L_{\texttt{ASGD}}$ & Meta-dec. \\ \hdashline
             &  \# of units layer $i$ & $u_i$  & $U_{\texttt{ASGD}}$ & Decreed \\
            \bottomrule
        \end{tabular}};

         \node [shape=rectangle, align=center, xshift=4.85cm, yshift=-3.5cm, at=(table1)] (table3)  {
        \begin{tabular}{llccc} 
        \toprule
            \multicolumn{2}{l}{HP} & Variable & Bounds & Role \\ \midrule
            \multicolumn{2}{l}{Running average 1} & $\beta_1$ & $]0,1[$ & Decreed \\
            \multicolumn{2}{l}{Running average 2} & $\beta_2$ & $]0,1[$ & Decreed \\
            \multicolumn{2}{l}{Numerical stability} & $\beta_3$   & $]0,1[$ & Decreed \\ 
            \multicolumn{2}{l}{\# of hidden layers} &  $l$ & $L_{\texttt{ADAM}}$ & Meta-dec. \\ \hdashline
             &  \# of units layer $i$ & $u_i$  & $U_{\texttt{ADAM}}$ & Decreed \\
            \bottomrule
        \end{tabular}};
   
    \draw[->] (table1)--(table2) node[midway, anchor=west, above, xshift=-2cm, yshift=-0.15cm] {if $o=\texttt{ASGD}$};
    \draw[->] (table1)--(table3) node[midway, anchor=east, above, xshift=2cm, yshift=-0.15cm] {if $o=\texttt{ADAM}$};

    % ----------- center leaf ---------%
    \node [shape=rectangle, align=center, xshift=0cm, yshift=-7cm, at=(table1)](table5){
        \begin{tabular}{lccc} 
        \toprule
            HP & Variable & Bounds & Role \\ \midrule
            Dropout & $\rho$ &
            $\left[    
           0, \quad  \frac{1}{2\tau_{\text{max}}} \sum_{i=1}^{l} u_i
            \right]$ & Decreed \\
            \bottomrule
        \end{tabular}};
    
    \draw[->] (table2)--(table5) node[midway, anchor=east, above, xshift=1.85cm, yshift=-0.25cm] {};
    \draw[->] (table3)--(table5) node[midway, anchor=west, above, xshift=-1cm] {};
       
\end{tikzpicture}}
\caption{Hierarchical and mixed-variable domain of the HPD.}
\label{fig:working_example_table}
\end{figure}

In Figure~\ref{fig:working_example_table}, the bounds of the hyperparameters $\alpha_1$, $\alpha_2$, $\alpha_3$, $\beta_1$, $\beta_2$ and $\beta_3$ are normalized. 
The HPD has two levels of hierarchy, it contains meta and meta-decreed variables influencing inclusions or bounds.

To perform regression or classification, data points must be associated with corresponding images.
The image $f(x) \in [0,100]$ of a vector of hyperparameters $x \in \mathcal{X}$ consists of a performance test score between 0 and 100, as illustrated below.
%in Figure~\ref{fig:data_point_HPO}.  
%
\begin{figure}[htb!]
\centering
  \scalebox{0.9}{\begin{tikzpicture}

    \small
    %\begin{figure}[htb!]
    %\centering
    
    \node(L) [] at (-4.75,0) {
    \begin{tabular}{c}
    Vector of \\
    HPs \\
    $x$
    \end{tabular} 
    };

    \node(C) [draw, minimum width=3cm, minimum height=3cm ] at (0,0) {
    \begin{tabular}{l} 
     %\textbullet \ Data \\
     %\\
    %
    \textbullet \ Architecture  \\
     \\
    \textbullet \ Training   \\ \\
    \textbullet \ Validation  \\
    %\textbullet \ Validation and test 
    \end{tabular}
    };
    \node(R) [] at (4.75,0) {
    \begin{tabular}{c}
    Test score \\
    $f(x)$ \\
    
    \end{tabular} 
    };
    \draw [-{Stealth[length=2mm, width=1.5mm]}] (L)--(C);
    \draw [-{Stealth[length=2mm, width=1.5mm]}] (C)--(R);
    
    %\caption{Objective function of the HPO problem.}

\end{tikzpicture}}
\caption{\blr A data point $x \in \mathcal{X}$ and its image $f(x) \in [0,100]$ in the computational experiments.}
\label{fig:data_point_HPO}
\end{figure}
%
%The type of architecture is either a MLP deep model on the F-MNIST or a CNN deep model on the CIFAR-10.
%

% Comment on a point
Figure~\ref{fig:data_point_HPO} illustrates that computing the image $f(x)$ of a vector of hyperparameters $x$ involves training, validating and testing a deep model.
The architecture of this deep model is not treated as a variable, but rather a fixed parameter influencing the behavior of $f$.

In the experiments, there are 40 independent datasets.
A dataset is identified with the nomenclature HPD-\textbf{variant}-\textbf{size}-\textbf{architecture}, where 
%refers hyperparameters domain in Figure~\ref{fig:working_example_table} 
the three elements in bold are:
\begin{itemize}
    \item a \textbf{variant} of the HPD, in which some hyperparameters are fixed.
    The variants range from \#1 to \#5, representing increasing difficulties. They are detailed in~\ref{sec:appendixA_variants}.

    \item a \textbf{size}, amongst Very Small (VS), Small (S), Medium (M) and Large (L), representing the amount of data available.
    The sizes represent a second lever of difficulty, as smaller datasets are typically more difficult to tackle.
    The datasets are independent across sizes, \textit{i.e.}, there are no subsets of datasets.
    Their sizes are detailed in~\ref{sec:appendixB_sizes}.

    \item a type of \textbf{architecture}, either MLP or Convolutional Neural Network (CNN), is used to structure the hyperparameters and compute $f$ at a point $x \in \mathcal{X}$.
    The type of architecture introduces diversity into the data, as discussed previously.
    %since it is fixed and it influences the mapping of $x$ to its image $f(x)$.

\end{itemize}

For the CNN architecture, units represent channels.
The MLP architecture uses Fashion-MNIST~\cite{xiao2017fashion}, whereas CNN employs CIFAR10~\cite{krizhevsky2009learning}.
Fashion-MNIST and CIFAR10 are strictly used for generating data points, as in Figure~\ref{fig:data_point_HPO}: they must not be confused with the actual datasets used in the computational experiments.
Details on the generation of data are provided in~\ref{sec:appendix_data_generation}.

The datasets are randomly split into training (50\%), validation (25\%) and test (25\%) datasets.
% 
%In order to take into consideration the random splits, the datasets are instantiated by random seeds.
%
A dataset paired with a random seed is referred to as an \textit{instance}, and is denoted $p$.
Also, an approach paired with a model is called an \textit{approach-model}, and is denoted $s$.

% ----------------------------------------- %
\subsection{Classification experiments}
\label{sec:classification}
% ----------------------------------------- %

The first experiment evaluates the classification accuracy on test datasets using a $K$-nearest neighbors model.
Training datasets contain points consisting of available neighbors, and validation datasets are used to adjust the parameters of the approach-models.
The budget of evaluations to train an approach-model $s \in \{ \text{{\sf Sub}-KNN, {\sf Hybrid}-KNN, {\sf Meta}-KNN} \}$ on a given instance $p$ is 100$n_{p,s}$, where $n_{p,s}$ is the number of parameters of $s$ on $p$.

The datasets considered are the HPD-\textbf{X}-\textbf{VS}-\textbf{MLP} and HPD-\textbf{X}-\textbf{S}-\textbf{MLP}, \textit{i.e.}, the 
very small and small ones with the MLP architecture.
%are repurposed for classification.
The experiments are performed using 5 labels.
The images $f(x) \in [0,100]$ in the datasets are transformed into labels, by partitioning uniformly the interval $[0,100]$ into five subintervals of length $20$ and assigning labels.
For example, an image $f(x) \in {[}0, 20{[}$ is replaced by the label 0, and an image $f(y) \in [80, 100]$ by the label 4.
The MLP architecture is used for classification since it produces images that are more uniformly distributed over $[0,100]$ than the CNN architecture.
Table~\ref{tab:classfication} presents the accuracies obtained on test datasets, and these accuracies are computed as an average over 20 random seeds.

\begin{table}[htb!]
\small
\renewcommand{\arraystretch}{1.2}
\centering
% ------- %
\begin{subtable}[t]{0.48\textwidth}
\centering
\begin{tabular}{c|ccc}
\hline
\textbf{Variant} & \textbf{Sub} & \textbf{Hybrid} & \textbf{Meta} \\
\hline \hline
\#1 & 0.57 & 0.67 & \textbf{0.68} \\ \hline
\#2 & 0.60 & \textbf{0.67} & 0.62 \\ \hline
\#3 & 0.52 & \textbf{0.56} & 0.51 \\ \hline
\#4 & 0.54 & \textbf{0.62} & 0.57 \\ \hline 
\#5 & \textbf{0.58} & 0.57 & \textbf{0.58} \\ \hline
\end{tabular}
\caption{Very small size.}
\label{tab:classification_sub1}
\end{subtable}
% ------- %
\hfill
\begin{subtable}[t]{0.48\textwidth}
\centering
\begin{tabular}{c|ccc}
\hline
\textbf{Variant} & \textbf{Sub} & \textbf{Hybrid} & \textbf{Meta} \\
\hline \hline
\#1 & 0.67 & 0.72 & \textbf{0.73}\\ \hline
\#2 & 0.67 & \textbf{0.71} & 0.70 \\ \hline
\#3 & 0.59 & \textbf{0.62}& 0.60 \\ \hline
\#4 & 0.57 & \textbf{0.64} & 0.61 \\ \hline 
\#5 & 0.53 & \textbf{0.55} & 0.54 \\ \hline
\end{tabular}
\caption{Small size.}
\label{tab:classification_sub2}
\end{subtable}
% ------- %
\caption{\blr Average accuracies on test datasets with 20 random seeds.}
\label{tab:classfication}
\end{table}

While {\sf Meta} achieved the highest accuracy on variant~\#1, {\sf Hybrid} performed best overall.
%with modest improvements over {\sf Meta}.
%
{\sf Sub} obtained the worst performance across the experiments.
The standard deviations obtained are consistently similar across all variants, sizes and approaches, ranging from 0.05 to 0.1.

% ----------------------------------------- %
\subsection{Data profiles on regression problems}
\label{sec:data_profiles}
% ----------------------------------------- %

The main computational experiments are presented in this section.
%
%Regression is done with inverse distance weighting (IDW) models paired with the three approaches.
%referred as approach-models.
%
%Four approach-models are considered: {\sf Sub}-IDW, {\sf Meta}-IDW, {Hybrid}-IDW, {\sf Sub}-KNN, {\sf Meta}-KNN and {\sf Hybrid}-KNN.
% Parameters
Root Mean Squared Errors (RMSE) on validation and test datasets, respectively referred to RMSE validation and RMSE test, are computed and viewed as functions.
%
%For {\sf Sub}, a RMSE is computed by assigning data points to their respective subproblem and IDW model.
%
%The same applies for {\sf Hybrid} when it employs multiple models.

%
Before computing the RSME test, parameters are adjusted with respect to a corresponding RMSE validation.
The optimization is expressed as a least-squares problem focused on minimizing the RMSE validation.
The decision variables are the parameters of an approach-model $s \in \{ \text{{\sf Sub}-IDW, {\sf Hybrid}-IDW, {\sf Meta}-IDW} \}$.
%and its specific implementation of IDW models, referred as approach-models.
%
%The parameters to adjust depend on the approach employed.
%
%Regarding the model, IDW adds no parameter, whereas KNN adds the number of neighbors $K$ for {\sf Meta} and a number of neighbors per subproblem for {\sf Sub}.
%
%Both approaches have weight parameters as discussed in Section~\ref{sec:graph_structured_distance}.%
Each approach-model has weight parameters for the hyperparameters.
%\textit{i.e.}, the hyperparameters in Figure~\ref{fig:working_example_table}.
%
For {\sf Sub} and {\sf Hybrid}, some hyperparameters are repeated across subproblems, leading to distinct weight parameters per subproblem.
{\sf Meta} requires parameters for the excluded-included distances between hyperparameters that can take the special value \texttt{EXC}, and a parameter for the categorical distance of the optimizer (in the situations where the categorical distance is not fixed).

The derivatives of the objective functions are unknown.
The optimization of the RMSE validation is done with the open-source blackbox optimization software \nomad~\cite{nomad4paper} that is based on the Mesh Adaptive Direct Search algorithm~\cite{AuDe2006}.
RMSE tests are used to assess the generalization ability of the approach-models, and they are not subject to any optimization.

%The results are presented as data profiles, representing the proportion of $\tau$-solved instances by the approach-models.

%
Each dataset is instantiated by 5 seeds, for a total of 200 instances contained in the set $\mathcal{P}$.
An approach-model $s$ is said to have $\tau$-solved an instance $p \in \mathcal{P}$ when its RMSE is within a relative error of $\tau$ with respect to the best RMSE obtained by any $s$ on this instance $p$.
At evaluation $k\geq 0$, the convergence test of an approach-model $s$ on a instance $p$ is expressed as
\begin{equation}
\frac{\text{RMSE}_{p,s}(k) -  \text{RMSE}_{p}}{\text{RMSE}_{p}} \leq \tau,
\end{equation}
where $\text{RMSE}_{p,s}(k)>0$ is the RMSE of $s$ on instance $p$ at an evaluation $k\geq 0$, $\text{RMSE}_{p}>0$ is the best RMSE obtained on that instance $p$ and $\tau \in [0,1]$ is a given tolerance. 
%
%The fewer evaluation needed to meet the convergence criteria, the better it is.
%
Data profiles~\cite{MoWi2009} represent the proportion of $\tau$-solved instances by the approach-models:
\begin{equation}
    \text{data}_s\left( \kappa \right) \coloneq \frac{1}{\left| \mathcal{P} \right|} \left| \left\{ p \in \mathcal{P} \, : \, \frac{k_{p, s}}{n_{p,s} + 1}    \leq \kappa \right\} \right| \in [0,1]
\end{equation}
where $k_{p,s} \geq 0$ is the number of evaluations required for an approach-model $s$ to $\tau$-solved an instance $p \in \mathcal{P}$ and $n_{p,s}$ is the number of parameters of a solver $s$ on the instance $p$.
The horizontal axis $\kappa$ of a data profile represents groups of $(n_{p,s}+1)$ evaluations, which depends on the instance $p$ and the approach-model $s$, since the number of parameters $n_{p,s}$ (decision variables) depends on both of these.

In the experiments, the budget of evaluations for an approach-model $s$ on instance $p$ is 200$n_{p,s}$ evaluations.
Each optimization starts with Latin hypercube sampling (LHS) of 33\% of the total budget of evaluations.
See Appendix~A.3. in~\cite{AuHa2017} for more details.

Data profiles are presented in Figure~\ref{fig:data_profiles} for both the RMSE validation (objective function) and the RMSE test.  
%
%\todo{CA: Combien de temps ça a pris pour générer ces graphes ? \\ 
%
%\vspace{0.25cm}
%
%EHH: je n'ai pas les détails précis, mais ça variait bcp en fct de la variante et la taille. 
%
%J'ai mis le worst case.
%
%J'ai déduit avec les temps entre mes fichiers logs.
%}
%
The optimization problem is only performed on the RMSE validation.
%
%{\rd
Depending on the variant and size, solving an instance required between a few minutes and 75 minutes on an 11th generation Intel i7-11800H (2.30 GHz) CPU.
The maximum runtime of approximately 75 minutes was observed for all three approaches for instances with variant \#5 and the size set to large.
%}
%
RMSE tests are evaluated each time an improvement on the RMSE validation is found, in order to build the data profiles for the RMSE tests. 
%

% Data profiles
\begin{figure}[htb!]
\centering
\begin{subfigure}{\textwidth}
%\centering
\includegraphics[width=1\linewidth]{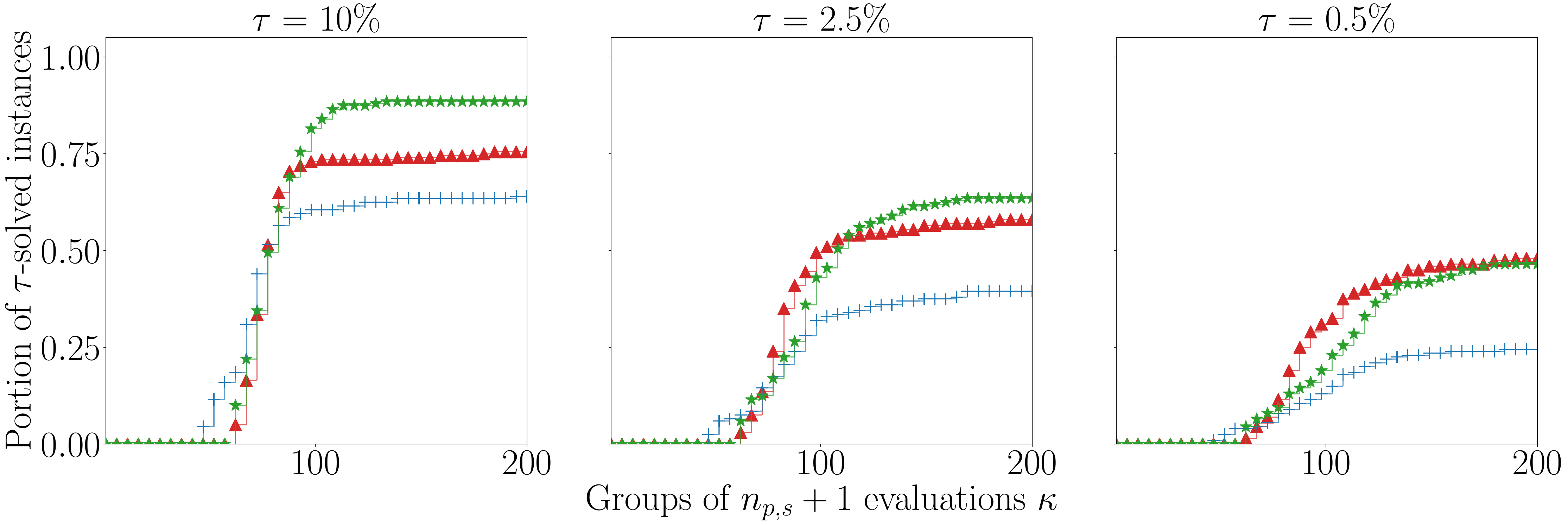}
\caption{Validation.}
\label{subfig:dataprofiles_validation}
\end{subfigure}
\begin{subfigure}{\textwidth}
%\centering
\includegraphics[width=1\linewidth]{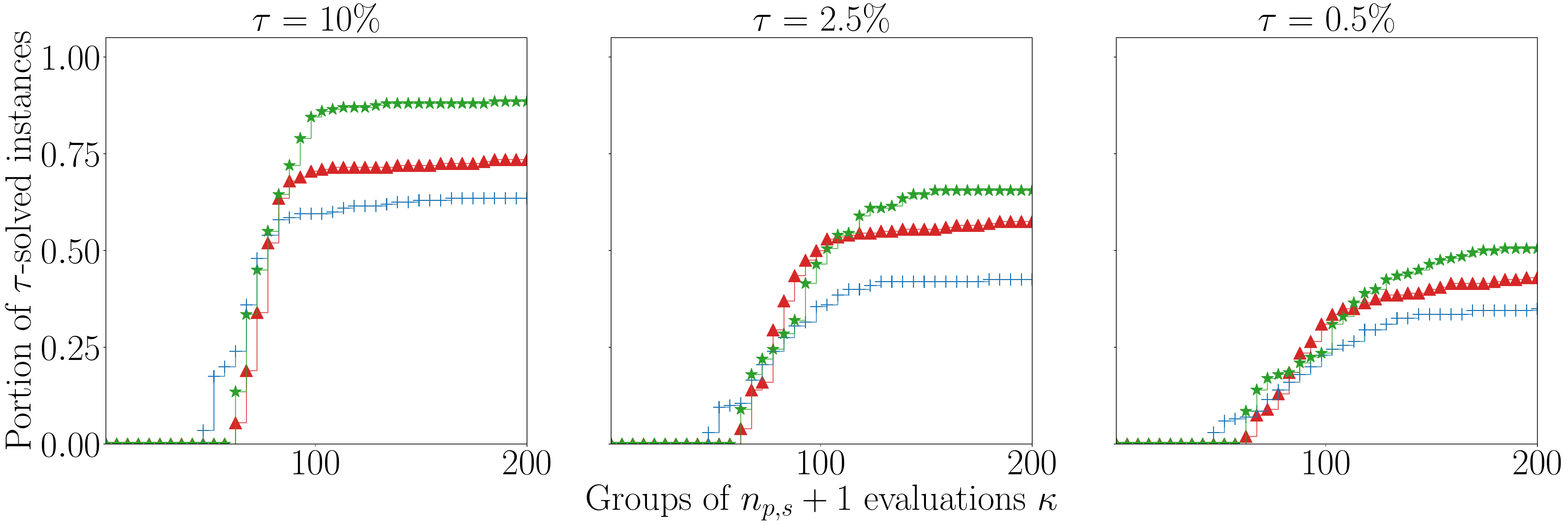}
\caption{Test.}
\label{subfig:dataprofiles_test}
\end{subfigure}
\begin{subfigure}{\textwidth}
\centering
\includegraphics[width=0.75\linewidth]{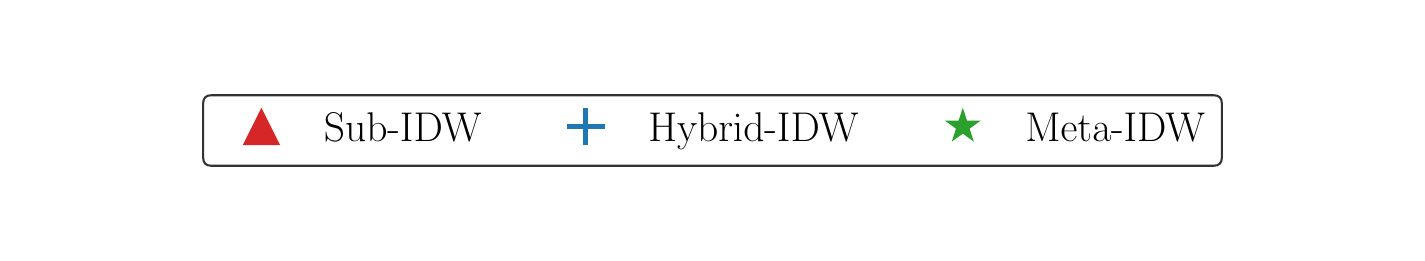}
\end{subfigure}
\caption{\blr Data profiles of the solvers on the regression problems with IDW models.}
\label{fig:data_profiles}
\end{figure}

% Comments on data profiles
The data profiles for the RMSE validation and test clearly indicates that {\sf Meta} outperforms the two other approaches.
%the other approach-models.
%
For example, Figure~\ref{subfig:dataprofiles_test} with $\tau=0.5\%$ shows that {\sf Meta}-KNN $\tau$-solves approximately 30\% of the instances with $100(n_{p,s}+1)$ evaluations, and after $200(n_{p,s}+1)$ evaluations it $\tau$-solves approximately 50\% of the instances, whereas {\sf Hybrid} only $\tau$-solves 30\%.
%whereas {\sf Sub}-KNN only $\tau$-solves 5\%.
%
%The {\sf Sub}-IDW and {\sf Meta}-IDW have similar profiles, but {\sf Sub}-IDW seems to perform slightly better.
%
The {\sf Hybrid} provides the worst performance on all profiles.
All data profiles exhibit a plateau during initial iterations, due to the LHS initialization.
%
%The LHS explores the parameter space, but does not seem to provide optimal parameters.

% ----------------------------------------- %
\subsection{Data aggregation with selected parameters experiments}
\label{sec:data_aggregation}
% ----------------------------------------- %

In this section, computational experiments are performed to study the aggregation of data.
The RMSE test is plotted against the number of points in a (partial) training dataset.
IDW models are constructed with a partial training dataset, which is increased iteratively with an additional random data point drawn from the training dataset without replacement.
The partial training dataset begins with a random data point from each subproblem, and points are added until it becomes the training dataset.
Training datasets represent interpolation points for IDW models.
%
%At early iterations, interpolation points are very limited, especially if they are separated into corresponding subproblems.
In early iterations, the numbers of points are limited, particularly when distributed across different subproblems.

%
%{\sf Meta} uses all interpolation, since it aggregates over subproblems.
%
%{\sf Sub} separates the interpolation points into their subproblems.
%
%{\sf Hybrid} aggregates for the first two variants, and partially separates with respect to the optimizer in the other variants.

%but only with a single layer of hierarchy.
%

The datasets are the HPD-\textbf{X}-M-\textbf{A}, \textit{i.e.}, all variants and both architecture are considered, while size is fixed at medium. 
The adjusted parameters obtained using the first random seed from Section~\ref{sec:data_profiles} are selected and fixed for all iterations. 
\begin{figure}[htb!]
\centering
\includegraphics[width=0.95\linewidth]{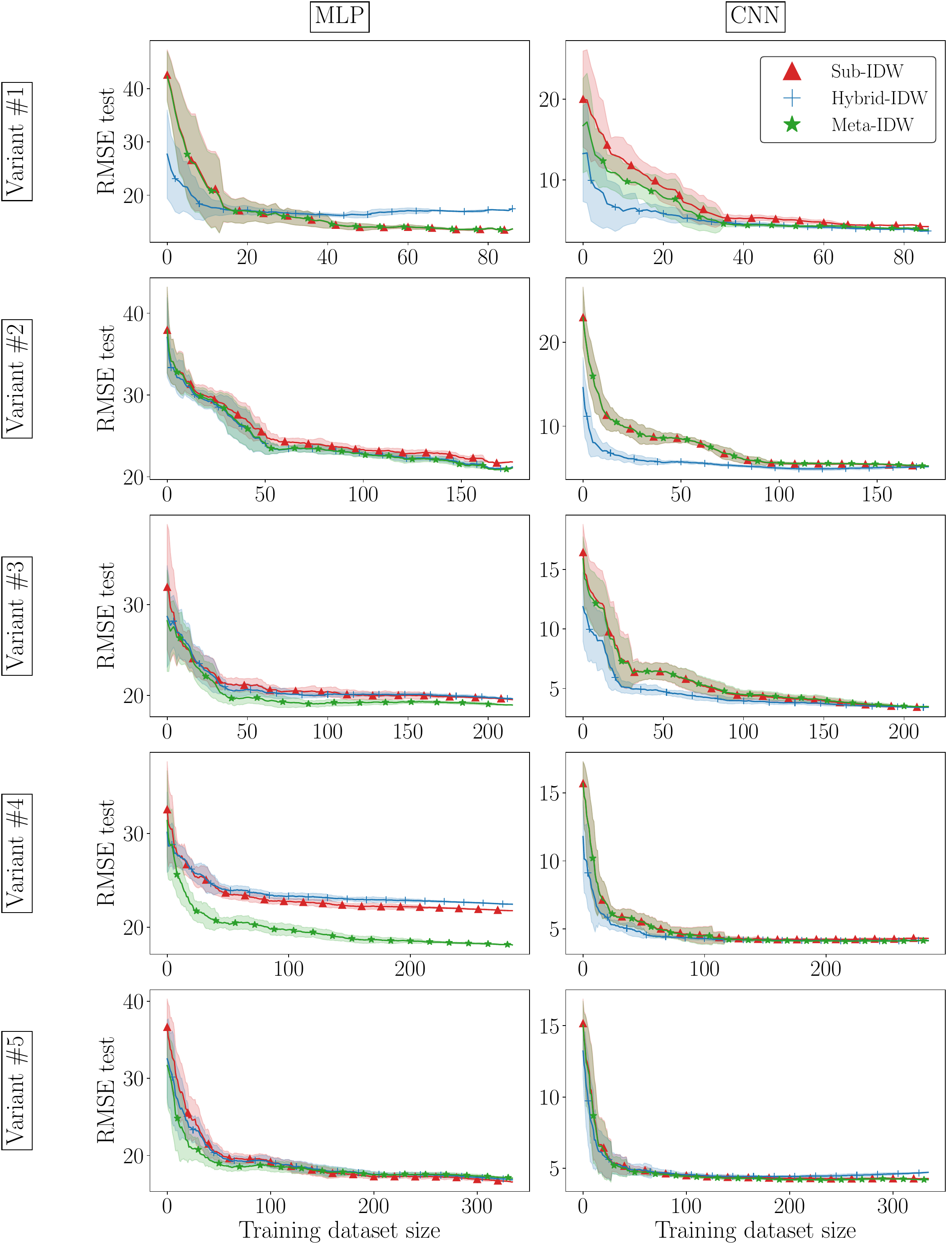}
\caption{\blr RMSE tests with iteratively increasing training data points and selected parameters.
Each subfigure is composed of 20 runs on different random seeds.
%
%The same vertical scale is used for comparability.
} 
\label{fig:numerical_graphs_training}
\end{figure}

In Figures~\ref{fig:numerical_graphs_training}, solid lines represent means, and shaded areas represent standard deviations. 
Each point in the graphs is averaged with 20 random seeds.
For the MLP architecture, {\sf Meta} performs best, whereas {\sf Sub} and {\sf Hybrid} have similar performances. 
For the CNN architecture, {\sf Hybrid} performs best at early iterations.
Towards the final iterations, all three approaches have similar performances, except for variant \#5, where {\sf Meta} has a slightly best RMSE test.
%
%The aggregation of data seems particularly promising for the datasets characterized by the MLP architecture model.

%As the level of difficulty increases through the instances, {\sf Meta} performs progressively better than {\sf Sub} for the IDW model, except for the last iterations in instance \#5.
%
%The aggregation of data seems particularly promising for the KNN model as it provides access to more neighbors across the subproblems.

}
% \blr end section

%--------------------------------------------------%
\section{Discussion}
\label{sec:conclusion}
%--------------------------------------------------%

{\blr

The present work focused on modeling mixed-variable and hierarchical domains in which two points do not share the same variables.
The research gap addressed is the lack of a distance function that is interpretable and is computed in constant-time.
To formalize the distance, a modeling framework is proposed, generalizing the following state-of-the-art frameworks: mixed-variable domains with (strictly) meta variables~\cite{G-2022-11}; tree-structured spaces~\cite{bergstra2011algorithms}; hierarchical spaces~\cite{HuOs2013}; variable-size design space~\cite{BaUrBrBa2023, PeBrBaTaGu2021}; and conditional search space (or neural architecture search) \cite{JiXuZh2022}.

The approach employed in this work can be summarized in Figure~\ref{fig:recap}, reprising Figure~\ref{subfig:intro_big_picture_approach} from the introduction.
\begin{figure}[htb!]
\centering
  \scalebox{1}{\begin{tikzpicture}

    % Yellow is green for simplicity
    %\definecolor{myyellow}{RGB}{255,255,0}
    \definecolor{myyellow}{RGB}{0,170,0}
    \definecolor{mygreen}{RGB}{0,170,0}
    \definecolor{myblue}{RGB}{0,160,255}
    \definecolor{myred}{RGB}{255,0,0}
    %\definecolor{mygray}{RGB}{160,160,160}
    \definecolor{mygray}{RGB}{255,255,255}
    \definecolor{mywhite}{RGB}{255,255,255}

    % Domain ellipse
    %\node[label={$(\mathcal{X}, \operatorname{dist})$}, ellipse, draw, minimum width=1.75cm, minimum height=2.75cm] (Domain) at (5,0) {};
    %\node[label={[align=center] $\left(\mathcal{X}, \dist_p \right)$, ellipse, draw, minimum width=1.75cm, minimum height=2.75cm] (Domain) at (5,0) {};
    \node[ellipse, draw, minimum width=1.75cm, minimum height=2.75cm, label={[align=center] above: $\left(\mathcal{X}, \dist_p \right)$}] (Domain) at (5,0) {};

    % ADDING SCATTERED BLOCKS FIGURE
    \begin{scope}[shift={(4.45,-0.35)}, scale=0.275] 
        % Yellow row (rotated)
        \begin{scope}[rotate=10, shift={(0.25,2.9)}]
            \fill[myyellow] (0,0) rectangle (4,1);
            \draw[thick] (0,0) rectangle (4,1);
            \foreach \x in {1,2,3} \draw[thick] (\x,0) -- (\x,1);
        \end{scope}

        % Green row (rotated)
        \begin{scope}[rotate=8, shift={(1,1.6)}]
            \fill[mygreen] (0,0) rectangle (4,1);
            \draw[thick] (0,0) rectangle (4,1);
            \foreach \x in {1,2,3} \draw[thick] (\x,0) -- (\x,1);
        \end{scope}

        % Blue row (not rotated)
        \begin{scope}[shift={(0,-0.2)}]
            \fill[myblue] (0,0) rectangle (3,1);
            \draw[thick] (0,0) rectangle (3,1);
            \foreach \x in {1,2} \draw[thick] (\x,0) -- (\x,1);
        \end{scope}

        % Red row (rotated)
        \begin{scope}[rotate=-10, shift={(2,-1.5)}]
            \fill[myred] (0,0) rectangle (2,1);
            \draw[thick] (0,0) rectangle (2,1);
            \draw[thick] (1,0) -- (1,1);
        \end{scope}
    \end{scope}

    % ---------------- RIGHT SIDE -------------------- %

    \begin{scope}[shift={(1, 0)}]

        % Extended domain ellipse
        %\node[label={$\left(\overline{\mathcal{X}}, \overline{\text{dist}} \right)$}, ellipse, draw, minimum width=1.75cm, minimum height=2.75cm] (Extended) at (10,0) {};
        \node[ellipse, draw, minimum width=1.75cm, minimum height=2.75cm, 
        label={[align=center]$\left( \overline{\mathcal{X}}, \edist_p \right)$}] 
        (Extended) at (10,0) {};

        % Grid inserted where the image was
        \begin{scope}[shift={(9.475, -0.55)}, scale=0.275]
            % Fill rectangles
            \fill[myyellow] (0,3) rectangle (4,4);
            \fill[mygreen] (0,2) rectangle (4,3);
            \fill[myblue] (0,1) rectangle (3,2);
            \fill[mywhite] (3,1) rectangle (4,2);
            \fill[myred] (0,0) rectangle (2,1);
            \fill[mygray] (2,0) rectangle (4,1);
            %\fill[myred] (3,0) rectangle (4,1);
    
            % Outline the entire rectangle
            \draw[thick] (0,0) rectangle (4,4);
            
            % Draw horizontal lines
            \draw[thick] (0,3) -- (4,3);
            \draw[thick] (0,2) -- (4,2);
            \draw[thick] (0,1) -- (4,1);
    
            % Draw equidistant vertical lines
            \draw[thick] (1,0) -- (1,4);
            \draw[thick] (2,0) -- (2,4);
            \draw[thick] (3,0) -- (3,4);
        \end{scope}
          
    \end{scope}

    % ------------------------ %
    % Arrows connecting points
    %\draw [-{Latex[length=3mm]}] (Point.north east) to [out=65, in=120] (Extended_pt.north west);
    %
    %\begin{scope}[yscale=-1, xscale=-1]
    %    \draw [{Latex[length=3mm]}-] (Point.north east) to [out=130, in=50] (Extended_pt.north west);
    %\end{scope}

    %\draw [{Latex[length=3mm]}-{Latex[length=3mm]}]  (5.5, 0.15) to [out=0, in=180] node[midway, above] {$T_G$} (8, 0.15);

    \draw [{Latex[length=3mm]}-{Latex[length=3mm]}] 
    (5.9, 0.155) to [out=0, in=180] 
    node[midway, above] {$T_G$} % Replace \(\star\) with the desired symbol
    (7.6+0.5+2, 0.15);

    % ---------------------- %

\end{tikzpicture}}
\caption{\blr A synthesis of the work.}
\label{fig:recap}
\end{figure}
A complicated heterogeneous dataset in a hierarchical domain $\mathcal{X}$ is mapped in the extended domain $\overline{\mathcal{X}}$, where the meta distance $\edist_p$ is used to compare any pair of points of the dataset.
A mapping $T_G$ based on a graph structure $G$ establishes a one-to-one correspondence between the original domain $\mathcal{X}$ and the extended domain $\overline{\mathcal{X}}$.
With this mapping, a distance $\dist_p$ is induced on the domain $\mathcal{X}$.

Theoretically, the meta distance should provide better generaliziability than simpler approaches partitioning domains, as it it allows to use more data simultaneously.
The computational experiments seems to confirm this.
Our approach using the meta distance performed globally better than an approach using subproblems and another one based on the variable-size framework~\cite{PeBrBaTaGu2021}.

This work is an important stepping stone towards efficient machine learning and optimization involving mixed-variable and hierarchical domains.
However, much more work remains. 
Currently, more complicated hierarchical interrelationships between variables, such as incompatibilities between values of different variables, are being developed in a subsequent work.
Moreover, further computational experiments on diverse problems, tasks and models are required.
For instance, Gaussian processes are commonly used for such mixed-variable and/or hierarchical domains~\cite{PeBrBaTaGu2021, BaDiMoLeSa2023, BaBuDiHwMaMoLaLeSa2023}, and these models will be studied in future work, with a focus on Bayesian optimization.

}

% ----------------------------------------%
\section*{Data availability statement} Scripts and data are publicly available at
\url{https://github.com/bbopt/graph_distance}. 
% ----------------------------------------%

% ----------------------------------------%
\section*{Declaration of interest statement.} On behalf of all authors, the corresponding author states that there is no conflict of interest. 
% ----------------------------------------%

% ----------------------------------------%
\section*{Acknowledgments} 
We express our gratitude to Amaury Diopus'kin for the \pytorch~implementation, produced during its internship in summer 2022, that was used for data generation.
% ----------------------------------------%

% ----------------------------------------------------%
%\begin{appendices}
\appendix

%------------ Appendix -------------%
\section{Variants of the HPD.}
\label{sec:appendixA_variants}
%-----------------------------------%

%\begin{figure}[htb!]
%
%\begin{subfigure}[t]{0.31\textwidth}
%    \centering
%    \fbox{
%    \scalebox{0.65}{\input{new_figs/MLP1}}
%    }
%    \caption{Variant \#1.}
%    \label{subfig:A}
%\end{subfigure}
%
%\hspace{0.1cm}
%
%\begin{subfigure}[t]{0.31\textwidth}
%    \centering
%    \fbox{
%    \scalebox{0.65}{\input{new_figs/MLP2}}
%    }
%    \caption{Variant \#2.}
%    \label{subfig:b}
%\end{subfigure}
%
%\hspace{0.1cm}
%
%\begin{subfigure}[t]{0.31\textwidth}
%    \centering
%    \fbox{
%    \scalebox{0.65}{\input{new_figs/MLP3}}
%    }
%    \caption{Variant \#3.}
%    \label{subfig:c}
%\end{subfigure}
%
%\vspace{0.3cm}

%\begin{subfigure}[t]{0.5\textwidth}
%    \centering
%    \fbox{
%    \scalebox{0.675}{\input{new_figs/MLP4}}
%    }
%    \caption{Variant \#4.}
%    \label{subfig:d}
%\end{subfigure}
%
%\begin{subfigure}[t]{0.5\textwidth}
%    \centering
%    \fbox{
%    \scalebox{0.675}{\input{new_figs/MLP5}}
%    }
%    \caption{Variant \#5.}
%    \label{subfig:e}
%\end{subfigure}
%
%\caption{Variants of the HPD, where each disconnected graph represent a subproblem.}
%\label{fig:MLP_problems}
%\end{figure}

% Table variants
Table~\ref{tab:problems_new} details the five variants and the parameters (decision variables) for the {\sf Sub} and {\sf Meta} approaches.
The variants are of increasing difficulty, notably in terms of number of subproblems and variables present.
Variant \#1 fixes the optimizer, hence the number of layers is meta instead of meta-decreed.
There are 3 subproblems, each assigned to a fixed number of layers.
Variant \#2 adds the hyperparameters $\alpha_{1:3}$.
Variant \#3 frees the optimizer $o \in \{ \texttt{ASGD}, \texttt{ADAM} \}$, and the hyperparameters $\beta_{1:3}$ are introduced via $o=\texttt{ADAM}$. 
There are 4 subproblems, each assigned to a pair $(o,l)$.
Variant \#4 adds a subproblem by allowing $l=3$ when $o=\texttt{ADAM}$.
The number of layers $l$ becomes a meta-decreed variable, since $L_{\texttt{ASGD}} \neq L_{\texttt{ADAM}}$. 
Variant \#5 adds the dropout $\rho$.

For {\sf Sub}, the variants are treated by their disconnected graphs separately.
In contrast, {\sf Meta}, considers each variant has a single problem, represented by a rectangular border. 
For {\sf Hybrid}, only one layer of hierarchy is considered. 
%and only shared included variables are compared.
%
Hence, in variant 1 and 2, all subproblems are considered since there is one layer of hierarchy.
For the other variants, {\sf Hybrid} has two subproblems, one per optimizer, since the layer of hierarchy tackle is the one related to the number of layers.

%Table~\ref{tab:problems_new} details the parameters (decision variables) of the variants for {\sf Sub} and {\sf Meta}.

% Subproblems budget
%For {\sf Sub}, the dataset size of a subproblem is 20 times the number of variables. 
%
%For {\sf Meta}, the dataset size is the sum of the subproblem dataset sizes.
% 
%For both approaches, 50\% of a dataset is allocated to a training set, and the validation and testing sets are each allocated a 25\%.
%
% Each data is generated randomly with uniform distribution on the domain of the variable.
%The data is generated with a uniform distribution on the restricted sets.
%
%The learning rate, as well as some hyperparameters $\alpha_i$ and $\beta_j$ are generated with a logarithmic scale.
%
%{\rd See Supplementary Material~X for more details on the data generation.}
%of Table~\ref{tab:problems}, such as the domain of variables that compose the parameter spaces. 
%
%{\gr In Table~\ref{tab:problems}, the instances are progressively more difficult.}
%

\begin{table}[htb!]
    \begin{subtable}{.5\linewidth}
      \centering
      \footnotesize
\begin{tabular}{|c|c|c|clll|c|}
\hline
Var.              & $o$                              & $l$ & \multicolumn{4}{c|}{Variables}            & \begin{tabular}[c]{@{}c@{}}\# of \\ variables\end{tabular} \\ \hline \hline
\multirow{3}{*}{\#1} & \multirow{3}{*}{$\texttt{ASGD}$} & 1   & $r$ & $u_{1}$   &                &        & 2                                                          \\
                     &                                  & 2   & $r$ & $u_{1:2}$ &                &        & 3                                                          \\
                     &                                  & 3   & $r$ & $u_{1:3}$ &                &        & 4                                                          \\ \hline \hline
\multirow{3}{*}{\#2} & \multirow{3}{*}{$\texttt{ASGD}$} & 1   & $r$ & $u_{1}$   & $\alpha_{1:3}$ &        & 5                                                          \\
                     &                                  & 2   & $r$ & $u_{1:2}$ & $\alpha_{1:3}$ &        & 6                                                          \\
                     &                                  & 3   & $r$ & $u_{1:3}$ & $\alpha_{1:3}$ &        & 7                                                          \\ \hline \hline
\multirow{4}{*}{\#3} & \multirow{2}{*}{$\texttt{ASGD}$} & 1   & $r$ & $u_{1}$   & $\alpha_{1:3}$ &        & 5                                                          \\
                     &                                  & 2   & $r$ & $u_{1:2}$ & $\alpha_{1:3}$ &        & 6                                                          \\ \cline{2-8} 
                     & \multirow{2}{*}{$\texttt{ADAM}$} & 1   & $r$ & $u_{1}$   & $\beta_{1:3}$  &        & 5                                                          \\
                     &                                  & 2   & $r$ & $u_{1:2}$ & $\beta_{1:3}$  &        & 6                                                          \\ \hline \hline
\multirow{5}{*}{\#4} & \multirow{2}{*}{$\texttt{ASGD}$} & 1   & $r$ & $u_{1}$   & $\alpha_{1:3}$ &        & 5                                                          \\
                     &                                  & 2   & $r$ & $u_{1:2}$ & $\alpha_{1:3}$ &        & 6                                                          \\ \cline{2-8} 
                     & \multirow{3}{*}{$\texttt{ADAM}$} & 1   & $r$ & $u_{1}$   & $\beta_{1:3}$  &        & 5                                                          \\
                     &                                  & 2   & $r$ & $u_{1:2}$ & $\beta_{1:3}$  &        & 6                                                          \\
                     &                                  & 3   & $r$ & $u_{1:3}$ & $\beta_{1:3}$  &        & 7                                                          \\ \hline \hline
\multirow{5}{*}{\#5} & \multirow{2}{*}{$\texttt{ASGD}$} & 1   & $r$ & $u_{1}$   & $\alpha_{1:3}$ & $\rho$ & 6                                                          \\
                     &                                  & 2   & $r$ & $u_{1:2}$ & $\alpha_{1:3}$ & $\rho$ & 7                                                          \\ \cline{2-8} 
                     & \multirow{3}{*}{$\texttt{ADAM}$} & 1   & $r$ & $u_{1}$   & $\beta_{1:3}$  & $\rho$ & 6                                                          \\
                     &                                  & 2   & $r$ & $u_{1:2}$ & $\beta_{1:3}$  & $\rho$ & 7                                                          \\
                     &                                  & 3   & $r$ & $u_{1:3}$ & $\beta_{1:3}$  & $\rho$ & 8                                                          \\ \hline
\end{tabular}
      \caption{Parameters for the {\sf Sub} approach.}
      \label{tab:problems_sub}
    \end{subtable}%
    \begin{subtable}{.5\linewidth}
      \centering
      \footnotesize
      \begin{tabular}{|c|clll|c|c|}
\hline 
Var.              & \multicolumn{4}{c|}{Variables}                                                          & \begin{tabular}[c]{@{}c@{}}\# of \\ variables\end{tabular} & \begin{tabular}[c]{@{}c@{}}\# of \\ $\theta_i^r$\end{tabular} \\ \hline \hline
\multirow{3}{*}{\#1} & \multicolumn{4}{c|}{\multirow{3}{*}{$l, r, u_{1:3}$}}                                   & \multirow{3}{*}{5}                                         & \multirow{3}{*}{2}                                            \\
                     & \multicolumn{4}{c|}{}                                                                   &                                                            &                                                               \\
                     & \multicolumn{4}{c|}{}                                                                   &                                                            &                                                               \\ \hline \hline
\multirow{3}{*}{\#2} & \multicolumn{4}{c|}{\multirow{3}{*}{$l, r, u_{1:3}, \alpha_{1:3}$}}                     & \multirow{3}{*}{8}                                         & \multirow{3}{*}{2}                                            \\
                     & \multicolumn{4}{c|}{}                                                                   &                                                            &                                                               \\
                     & \multicolumn{4}{c|}{}                                                                   &                                                            &                                                               \\ \hline \hline
\multirow{4}{*}{\#3} & \multicolumn{4}{c|}{\multirow{4}{*}{$o,l,r, u_{1:2}, \alpha_{1:3}, \beta_{1:3}$}}       & \multirow{4}{*}{11}                                        & \multirow{4}{*}{7}                                            \\
                     & \multicolumn{4}{c|}{}                                                                   &                                                            &                                                               \\
                     & \multicolumn{4}{c|}{}                                                                   &                                                            &                                                               \\
                     & \multicolumn{4}{c|}{}                                                                   &                                                            &                                                               \\ \hline \hline
\multirow{5}{*}{\#4} & \multicolumn{4}{c|}{\multirow{5}{*}{$o,l,r, u_{1:3}, \alpha_{1:3}, \beta_{1:3}$}}       & \multirow{5}{*}{12}                                        & \multirow{5}{*}{8}                                            \\
                     & \multicolumn{4}{c|}{}                                                                   &                                                            &                                                               \\
                     & \multicolumn{4}{c|}{}                                                                   &                                                            &                                                               \\
                     & \multicolumn{4}{c|}{}                                                                   &                                                            &                                                               \\
                     & \multicolumn{4}{c|}{}                                                                   &                                                            &                                                               \\ \hline \hline
\multirow{5}{*}{\#5} & \multicolumn{4}{c|}{\multirow{5}{*}{$o,l,r, u_{1:3}, \alpha_{1:3}, \beta_{1:3}, \rho$}} & \multirow{5}{*}{13}                                        & \multirow{5}{*}{8}                                            \\
                     & \multicolumn{4}{c|}{}                                                                   &                                                            &                                                               \\
                     & \multicolumn{4}{c|}{}                                                                   &                                                            &                                                               \\
                     & \multicolumn{4}{c|}{}                                                                   &                                                            &                                                               \\
                     & \multicolumn{4}{c|}{}                                                                   &                                                            &                                                               \\ \hline
\end{tabular}       
      \caption{Parameters for the {\sf Meta} approach.}
      \label{tab:problems_graph}
    \end{subtable}
        \caption{Variants of HPD with $u_{1:2}=(u_1,u_2)$, $u_{1:3}=(u_1,u_2, u_3)$, $\alpha_{1:3}=(\alpha_1,\alpha_2,\alpha_3)$ and $\beta_{1:3}=(\beta_1,\beta_2,\beta_3)$.
        %the activation function is fixed at $a=\texttt{ReLU}$ for all variants.
        }
        \label{tab:problems_new}
\end{table}

%------------ Appendix -------------%
%\clearpage \newpage
\section{Dataset sizes of the instances.}
\label{sec:appendixB_sizes}
%-----------------------------------%

Table~\ref{tab:instances_new} details the dataset sizes of all variants for {\sf Meta} and {\sf Sub}.
For variants \#1 and \#2, {\sf Hybrid} aggregates all data as {\sf Meta}.
For the other variants, {\sf Hybrid} aggregates over the hidden layers, but separates the data for each optimizer in Table~\ref{tab:instances_sub}.

\begin{table}[htb!]
    \begin{subtable}{.5\linewidth}
      \centering
      \footnotesize
        \begin{tabular}{|c|c|c|cccc|}
\hline
Variant                 & $o$                              & $l$ & VS & S   & M   & L   \\ \hline \hline
\multirow{3}{*}{\#1} & \multirow{3}{*}{$\texttt{ASGD}$} & 1   & 20 & 30  & 40  & 50  \\
                     &                                  & 2   & 30 & 45  & 60  & 75  \\
                     &                                  & 3   & 40 & 60  & 80  & 100 \\ \hline \hline
\multirow{3}{*}{\#2} & \multirow{3}{*}{$\texttt{ASGD}$} & 1   & 50 & 75  & 100 & 125 \\
                     &                                  & 2   & 60 & 90  & 120 & 150 \\
                     &                                  & 3   & 70 & 105 & 140 & 175 \\ \hline \hline
\multirow{4}{*}{\#3} & \multirow{2}{*}{$\texttt{ASGD}$} & 1   & 50 & 75  & 100 & 125 \\
                     &                                  & 2   & 60 & 90  & 120 & 150 \\ \cline{2-7} 
                     & \multirow{2}{*}{$\texttt{ADAM}$} & 1   & 50 & 75  & 100 & 125 \\
                     &                                  & 2   & 60 & 90  & 120 & 150 \\ \hline \hline
\multirow{5}{*}{\#4} & \multirow{2}{*}{$\texttt{ASGD}$} & 1   & 50 & 75  & 100 & 125 \\
                     &                                  & 2   & 60 & 90  & 120 & 150 \\ \cline{2-7} 
                     & \multirow{3}{*}{$\texttt{ADAM}$} & 1   & 50 & 75  & 100 & 125 \\
                     &                                  & 2   & 60 & 90  & 120 & 150 \\
                     &                                  & 3   & 70 & 105 & 140 & 175 \\ \hline \hline
\multirow{5}{*}{\#5} & \multirow{2}{*}{$\texttt{ASGD}$} & 1   & 60 & 90  & 120 & 150 \\
                     &                                  & 2   & 70 & 105 & 140 & 175 \\ \cline{2-7} 
                     & \multirow{3}{*}{$\texttt{ADAM}$} & 1   & 60 & 90  & 120 & 150 \\
                     &                                  & 2   & 70 & 105 & 140 & 175 \\
                     &                                  & 3   & 80 & 120 & 160 & 200 \\ \hline
\end{tabular}
      \caption{Dataset sizes for {\sf Sub} approach.}
      \label{tab:instances_sub}
    \end{subtable}%
    \begin{subtable}{.5\linewidth}
      \centering
      \footnotesize
      \begin{tabular}{|c|cccc|}
\hline
Variant                 & VS                   & S                    & M                    & L                    \\ \hline \hline
\multirow{3}{*}{\#1} & \multirow{3}{*}{90}  & \multirow{3}{*}{135} & \multirow{3}{*}{180} & \multirow{3}{*}{225} \\
                     &                      &                      &                      &                      \\
                     &                      &                      &                      &                      \\ \hline \hline
\multirow{3}{*}{\#2} & \multirow{3}{*}{180} & \multirow{3}{*}{270} & \multirow{3}{*}{360} & \multirow{3}{*}{450} \\
                     &                      &                      &                      &                      \\
                     &                      &                      &                      &                      \\ \hline \hline
\multirow{4}{*}{\#3} & \multirow{4}{*}{220} & \multirow{4}{*}{330} & \multirow{4}{*}{440} & \multirow{4}{*}{550} \\
                     &                      &                      &                      &                      \\
                     &                      &                      &                      &                      \\
                     &                      &                      &                      &                      \\ \hline \hline
\multirow{5}{*}{\#4} & \multirow{5}{*}{290} & \multirow{5}{*}{435} & \multirow{5}{*}{580} & \multirow{5}{*}{725} \\
                     &                      &                      &                      &                      \\
                     &                      &                      &                      &                      \\
                     &                      &                      &                      &                      \\
                     &                      &                      &                      &                      \\ \hline \hline
\multirow{5}{*}{\#5} & \multirow{5}{*}{340} & \multirow{5}{*}{510} & \multirow{5}{*}{680} & \multirow{5}{*}{850} \\
                     &                      &                      &                      &                      \\
                     &                      &                      &                      &                      \\
                     &                      &                      &                      &                      \\
                     &                      &                      &                      &                      \\ \hline
\end{tabular}
      \caption{Dataset sizes for {\sf Meta} approach.}
      \label{tab:instances_graph}
    \end{subtable}
        \caption{Dataset sizes of the instances, characterized by a variant-size, with sizes amongst Very Small (VS), Small (S), Medium (M) and Large (L).}
        \label{tab:instances_new}
\end{table}

% Comment on graph that aggregates the sizes
In Table~\ref{tab:instances_graph}, the values in size-related columns are the sums of the values in the rows of the corresponding columns in Table~\ref{tab:instances_sub}, \textit{i.e.}, {\sf Meta} aggregates the data of the subproblems of {\sf Sub}.

{\blr
%------------ Appendix -------------%
%\clearpage \newpage
\section{Data generation.}
\label{sec:appendix_data_generation}
%-----------------------------------%

The generation of a data couple $(x, f(x))$ is done by the following three steps:
1) a deep model, either MLP or CNN, is constructed with respect to its given hyperparameters $x \in \mathcal{X}$ with \pytorch;
2) the deep model is trained and validated with 25\% on either the Fashion-MNIST training dataset or CIFAR10 training dataset;
3) the performance score $f(x) \in [0,100]$, that represents the percentage of well-classified images, is computed on the deep model with another 25\% of either the corresponding Fashion-MNIST or CIFAR10 test datasets.
In order to reduce data generation time, only 25\% of the Fashion-MNIST and CIFAR10 datasets are used, and the number of epochs and batch size are respectively set to 25 and 128.
Fashion-MNIST and CIFAR10 are strictly used for generating data couples $(x, f(x))$ that form datasets.
Each dataset is heterogeneous and it is split into a training, validation and test datasets.
}

%----------- References ------------%
\clearpage \newpage
\bibliographystyle{plain}
\bibliography{bibliography.bib}

\begin{thebibliography}{10}

\bibitem{AhDe07}
A.~Ahmad and L.~Dey.
\newblock {A k-mean clustering algorithm for mixed numeric and categorical data}.
\newblock {\em Data \& Knowledge Engineering}, 63(2):503--527, 2007.

\bibitem{AlBuGrKoMe2013}
A.~Aleti, B.~Buhnova, L.~Grunske, A.~Koziolek, and I.~Meedeniya.
\newblock {Software Architecture Optimization Methods: A Systematic Literature Review}.
\newblock {\em IEEE Transactions on Software Engineering}, 39(5):658--683, 2013.

\bibitem{AlNeTr19}
N.~Ali, D.~Neagu, and P.~Trundle.
\newblock {Classification of Heterogeneous Data Based on Data Type Impact on Similarity}.
\newblock In {\em Advances in Computational Intelligence Systems}. Springer International Publishing, 2019.

\bibitem{AlNeTr19_2}
N.~Ali, D.~Neagu, and P.~Trundle.
\newblock {Evaluation of k-nearest neighbour classifier performance for heterogeneous data sets}.
\newblock {\em SN Applied Sciences}, 1:1--15, 2019.

\bibitem{AsGrMoGa16}
M.~Asadi, G.~Gr\"{o}ner, B.~Mohabbati, and D.~Ga\u{s}evi\'{c}.
\newblock {Goal-oriented modeling and verification of feature-oriented product lines}.
\newblock {\em Software \& Systems Modeling}, 15:257--279, 2016.

\bibitem{AuDe2006}
C.~Audet and J.E. {Dennis, Jr.}
\newblock {Mesh Adaptive Direct Search Algorithms for Constrained Optimization}.
\newblock {\em SIAM Journal on Optimization}, 17(1):188--217, 2006.

\bibitem{G-2022-11}
C.~Audet, E.~{Hall\'e-Hannan}, and S.~{Le~Digabel}.
\newblock {A General Mathematical Framework for Constrained Mixed-variable Blackbox Optimization Problems with Meta and Categorical Variables}.
\newblock {\em Operations Research Forum}, 4(12), 2023.

\bibitem{AuHa2017}
C.~Audet and W.~Hare.
\newblock {\em {Derivative-Free and Blackbox Optimization}}.
\newblock Springer Series in Operations Research and Financial Engineering. Springer, Cham, Switzerland, 2017.

\bibitem{nomad4paper}
C.~Audet, S.~{Le~Digabel}, V.~{Rochon~Montplaisir}, and C.~Tribes.
\newblock {Algorithm~1027: NOMAD version~4: Nonlinear optimization with the MADS algorithm}.
\newblock {\em {ACM} Transactions on Mathematical Software}, 48(3):35:1--35:22, 2022.

\bibitem{BaUrBrBa2023}
L.~Baraton, A.~Urbano, L.~Brevault, and M.~Balesdent.
\newblock {Comparative review of Multidisciplinary Design Analysis and Optimization architectures for the preliminary design of a liquid rocket engine}.
\newblock In {\em Aerospace Europe Conference 2023}. EUCASS, 2023.

\bibitem{Ba05}
D.~Batory.
\newblock {Feature models, grammars, and propositional formulas}.
\newblock In {\em International Conference on Software Product Lines}. Springer, 2005.

\bibitem{Be2013}
G.~Beer.
\newblock {The Structure of Extended Real-valued Metric Spaces}.
\newblock {\em Set-Valued and Variational Analysis}, 21:591--602, 2013.

\bibitem{BeSeRu10}
D.~Benavides, S.~Segura, and A.~Ruiz-Cort{\'e}s.
\newblock {Automated analysis of feature models 20 years later: A literature review}.
\newblock {\em Information Systems}, 35(6):615--636, 2010.

\bibitem{BeFaPeSkSu2005}
M.~Bender, M.~Farach-Colton, G.~Pemmasani, S.~Skiena, and P.~Sumazin.
\newblock {Lowest common ancestors in trees and directed acyclic graphs}.
\newblock {\em Journal of Algorithms}, 57(2):75--94, 2005.

\bibitem{bergstra2011algorithms}
J.~Bergstra, R.~Bardenet, Y.~Bengio, and B.~K{\'e}gl.
\newblock {Algorithms for hyper-parameter optimization}.
\newblock In {\em Advances in neural information processing systems}, 2011.

\bibitem{BuCiDeNaLa2021}
J.H. Bussemaker, P.D. Ciampa, T.~{De~Smedt}, B.~Nagel, and G.~{La~Rocca}.
\newblock {System Architecture Optimization: An Open Source Multidisciplinary Aircraft Jet Engine Architecting Problem}.
\newblock In {\em AIAA AVIATION 2021 Forum}. American Institute of Aeronautics and Astronautics, 2021.

\bibitem{ChBoNaWhPa2012}
J.~Choo, S.~Bohn, G.C. Nakamura, A.M. White, and H.~Park.
\newblock {Heterogeneous Data Fusion via Space Alignment Using Nonmetric Multidimensional Scaling}.
\newblock In {\em Proceedings of the 2012 SIAM International Conference on Data Mining}, pages 177--188. SIAM, 2012.

\bibitem{FeHu19}
M.~Feurer and F.~Hutter.
\newblock {Hyperparameter optimization}.
\newblock {\em Automated machine learning: Methods, systems, challenges}, pages 3--33, 2019.

\bibitem{GaDuKaSe2018}
A.~Gardner, C.A. Duncan, J.~Kanno, and R.R. Selmic.
\newblock {On the Definiteness of Earth Mover's Distance and Its Relation to Set Intersection}.
\newblock {\em IEEE Transactions on Cybernetics}, 48(11):3184--3196, 2018.

\bibitem{GaHe2020}
E.C. Garrido-Merch\'an and D.~Hern\'andez-Lobato.
\newblock {Dealing with categorical and integer-valued variables in Bayesian Optimization with Gaussian processes}.
\newblock {\em Neurocomputing}, 380:20--35, 2020.

\bibitem{HoIsKu2015}
K.~Horio, S.~Ishikawa, and R.~Kubota.
\newblock {Effective Hierarchical Optimization using Integration of Solution Spaces and its Application to multiple Vehicle Routing Problem}.
\newblock In {\em 2015 International Symposium on Intelligent Signal Processing and Communication Systems}. IEEE, 2015.

\bibitem{HuHuKeTs16}
L.-Y. Hu, M.-W. Huang, S.-W. Ke, and C.-F. Tsai.
\newblock {The distance function effect on k-nearest neighbor classification for medical datasets}.
\newblock {\em SpringerPlus}, 5:1--9, 2016.

\bibitem{HuOs2013}
F.~Hutter and M.A. Osborne.
\newblock {A Kernel for Hierarchical Parameter Spaces}.
\newblock Technical Report 1310.5738, ArXiv, 2013.

\bibitem{JiLi05}
R.~Jin and H.~Liu.
\newblock {A Novel Approach to Model Generation for Heterogeneous Data Classification}.
\newblock In {\em Proceedings of the 19th International Joint Conference on Artificial Intelligence}. Morgan Kaufmann Publishers Inc., 2005.

\bibitem{JiXuZh2022}
K.~Jing, J.~Xu, and Z.~Zhang.
\newblock {A neural architecture generator for efficient search space}.
\newblock {\em Neurocomputing}, 486:189--199, 2022.

\bibitem{KiHo17}
K.~Kim and {J.-s. Hong}.
\newblock {A hybrid decision tree algorithm for mixed numeric and categorical data in regression analysis}.
\newblock {\em Pattern Recognition Letters}, 98:39--45, 2017.

\bibitem{krizhevsky2009learning}
A.~Krizhevsky and G.~Hinton.
\newblock {Learning multiple layers of features from tiny images}.
\newblock Technical report, Citeseer, 2009.

\bibitem{hypernomad_paper}
D.~Lakhmiri, S.~{Le~Digabel}, and C.~Tribes.
\newblock {HyperNOMAD: Hyperparameter Optimization of Deep Neural Networks Using Mesh Adaptive Direct Search}.
\newblock {\em {ACM} Transactions on Mathematical Software}, 47(3), 2021.

\bibitem{LuPi04a}
S.~Lucidi and V.~Piccialli.
\newblock {A Derivative-Based Algorithm for a Particular Class of Mixed Variable Optimization Problems}.
\newblock {\em Optimization Methods and Software}, 17(3--4):317--387, 2004.

\bibitem{LuPiSc05a}
S.~Lucidi, V.~Piccialli, and M.~Sciandrone.
\newblock {An Algorithm Model for Mixed Variable Programming}.
\newblock {\em SIAM Journal on Optimization}, 15(4):1057--1084, 2005.

\bibitem{MoWi2009}
J.J. Mor\'e and S.M. Wild.
\newblock {Benchmarking Derivative-Free Optimization Algorithms}.
\newblock {\em SIAM Journal on Optimization}, 20(1):172--191, 2009.

\bibitem{MScGhazaleh}
G.~Noroozi.
\newblock {Data Heterogeneity and Its Implications for Fairness}.
\newblock Master's thesis, Western University, 2023.
\newblock Available at \url{https://ir.lib.uwo.ca/etd/9623/}.

\bibitem{PeBrBaTaGu2019}
J.~Pelamatti, L.~Brevault, M.~Balesdent, E.-G. Talbi, and Y.~Guerin.
\newblock {Efficient global optimization of constrained mixed variable problems}.
\newblock {\em Journal of Global Optimization}, 73(3):583--613, 2019.

\bibitem{PeBrBaTaGu2021}
J.~Pelamatti, L.~Brevault, M.~Balesdent, E.-G. Talbi, and Y.~Guerin.
\newblock {Bayesian optimization of variable-size design space problems}.
\newblock {\em Optimization and Engineering}, 22:387--447, 2021.

\bibitem{PeCaRe10}
C.L. Pereira, G.D.C. Cavalcanti, and T.I. Ren.
\newblock {A New Heterogeneous Dissimilarity Measure for Data Classification}.
\newblock In {\em 22nd IEEE International Conference on Tools with Artificial Intelligence}, pages 373--374. IEEE, 2010.

\bibitem{QiWuJe2008}
P.Z.G. Qiand, H.~Wu, and C.F.J. Wu.
\newblock {Gaussian process models for computer experiments with qualitative and quantitative factors}.
\newblock {\em Technometrics}, 50(3):383--396, 2008.

\bibitem{RaLiShAmTa2018}
D.~Ramachandram, M.~Lisicki, T.~J. Shields, M.~R. Amer, and G.~W. Taylor.
\newblock {Bayesian optimization on graph-structured search spaces: Optimizing deep multimodal fusion architectures}.
\newblock {\em Neurocomputing}, 298:80--89, 2018.

\bibitem{RaWi06}
C.E. Rasmussen and C.K.I. Williams.
\newblock {\em Gaussian Processes for Machine Learning}.
\newblock The {MIT} Press, 2006.

\bibitem{ReHaHoBu2024}
A.~Remadi, K.~E. Hage, Y.~Hobeika, and F.~Bugiotti.
\newblock {To prompt or not to prompt: Navigating the use of large language models for integrating and modeling heterogeneous data}.
\newblock {\em Data \& Knowledge Engineering}, 152:102313, 2024.

\bibitem{SaFu1983}
A.~Sanfeliu and K.-S. Fu.
\newblock {A distance measure between attributed relational graphs for pattern recognition}.
\newblock {\em IEEE Transactions on Systems, Man, and Cybernetics}, SMC-13(3):353--362, 1983.

\bibitem{Saves2024}
P.~Saves.
\newblock {\em {High-dimensional multidisciplinary design optimization for aircraft eco-design}}.
\newblock PhD thesis, ONERA and ISAE-SUPAERO, 2024.
\newblock Available at \url{https://theses.hal.science/ONERA-MIP/tel-04439128v1}.

\bibitem{BaDiMoLeSa2023}
P.~Saves, Y.~Diouane, N.~Bartoli, T.~Lefebvre, and J.~Morlier.
\newblock {A mixed-categorical correlation kernel for Gaussian process}.
\newblock {\em Neurocomputing}, 550:126472, 2023.

\bibitem{BaBuDiHwMaMoLaLeSa2023}
P.~Saves, R.~Lafage, N.~Bartoli, Y.~Diouane, J.H. Bussemaker, T.~Lefebvre, J.T. Hwang, J.~Morlier, and J.R.R.A Martins.
\newblock {SMT 2.0: A Surrogate Modeling Toolbox with a focus on Hierarchical and Mixed Variables Gaussian Processes}.
\newblock {\em Advances in Engineering Software}, 188:103571, 2024.

\bibitem{SoLu15}
Y.-Y. Song and Y.~Lu.
\newblock {Decision tree methods: applications for classification and prediction}.
\newblock {\em Shanghai Arch Psychiatry}, 27(2):130--135, 2015.

\bibitem{SoLePeZaKe2023}
B.~Sow, R.~{Le~Riche}, J.~Pelamatti, S.~Zannane, and M.~Keller.
\newblock {Learning functions defined over sets of vectors with kernel methods}.
\newblock In {\em 5th International Conference on Uncertainty Quantification in Computational Science and Engineering (UNCECOMP)}. European Community on Computational Methods in Applied Sciences (ECCOMAS), 2023.

\bibitem{Da09}
D.J. Wild.
\newblock {Mining large heterogeneous data sets in drug discovery}.
\newblock {\em Expert Opinion on Drug Discovery}, 4(10):995--1004, 2009.

\bibitem{WuChZhXiDe19}
J.~Wu, X.-Y. Chen, H.~Zhang, L.-D. Xiong, H.L., and S.-H. Deng.
\newblock {Hyperparameter Optimization for Machine Learning Models Based on Bayesian Optimization}.
\newblock {\em Journal of Electronic Science and Technology}, 17(1):26--40, 2019.

\bibitem{xiao2017fashion}
H.~Xiao, K.~Rasul, and R.~Vollgraf.
\newblock {Fashion-MNIST: a Novel Image Dataset for Benchmarking Machine Learning Algorithms}.
\newblock Technical Report 1708.07747, arXiv, 2017.

\bibitem{YaSh20}
L.~Yang and A.~Shami.
\newblock {On hyperparameter optimization of machine learning algorithms: Theory and practice}.
\newblock {\em Neurocomputing}, 415:295--316, 2020.

\bibitem{Zh2020}
J.~Zhang.
\newblock {Graph Neural Distance Metric Learning with Graph-Bert}.
\newblock Technical Report 2002.03427, Arxiv, 2020.

\bibitem{ZhXiXiZh21}
L.~Zhang, Y.~Xie, L.~Xidao, and X.~Zhang.
\newblock {Multi-source heterogeneous data fusion}.
\newblock In {\em International Conference on Artificial Intelligence and Big Data (ICAIBD)}. IEEE, 2018.

\bibitem{ZhTaChAp2020}
Y.~Zhang, S.~Tao, W.~Chen, and D.W. Apley.
\newblock {A Latent Variable Approach to Gaussian Process Modeling with Qualitative and Quantitative Factors}.
\newblock {\em Technometrics}, 62(3):291--302, 2020.

\end{thebibliography}
\pdfbookmark[1]{References}{sec-refs}
%----------------------------------%

% -------- Table of correspondances for references ---- %
% 1.	G-2022-11 : G-2022-11
% 2.	Mixed_Paul : BaDiMoLeSa2023
% 3.	Jones98 : JoScWe1998
% 4.	HS : JaRe1999
% 5.	HS_Jacobi : BrMeRa2007
% 6.	Hutter : HuOs2013
% 7.	Effectiveness : BaBuCiNaLe2021
% 8.	Horn_hier : HoScStZa2019
% 9.	DACE_hier : HuJoMe2009
% 10.	Pelamattihier : PeBrBaTaGu2021
% 11.	Zaefferer : HoZa2018
% 12.   RaWi2006 : RaWi2006
% 13.   saves2023smt : BaBuDiHwMaMoLaLeSa2023
% 14.   gardner2017definiteness : GaDuKaSe2018
% 15	gomez2018automatic : AdAgAsDuGoHeHiSaShWe2018
% 16.	condearenzana : BaCoLeLoMo2021
% 17	tran:hal-03170761 : DaMoSiTr2021
% 18	hebbal2021bayesian : BaBrHeMeTa2021
% 19	wildberger2007rational : Wi2007
% 20	muller_so-mi_2013 : MuShPi2013
% ------------------------------------ %

\end{document}